\documentclass{article}
\usepackage[latin9]{inputenc}
\usepackage{array}
\usepackage{float}
\usepackage{amsmath}
\usepackage{amsthm}
\usepackage{amssymb}
\usepackage{graphicx}
\usepackage[authoryear]{natbib}
\usepackage{lastpage}

\newcommand\numberthis{\addtocounter{equation}{1}\tag{\theequation}}

\usepackage{xcolor}

\newcommand{\red}[1]{\textcolor{red}{ #1 }}
\makeatletter

\providecommand{\tabularnewline}{\\}
\floatstyle{ruled}
\newfloat{algorithm}{tbp}{loa}
\providecommand{\algorithmname}{Algorithm}
\floatname{algorithm}{\protect\algorithmname}
  \providecommand{\remarkname}{Remark}

\theoremstyle{definition}
\newtheorem{defn}{\protect\definitionname}
\theoremstyle{plain}
\newtheorem{thm}{\protect\theoremname}
\theoremstyle{plain}
\newtheorem{lem}{\protect\lemmaname}
\theoremstyle{plain}
\newtheorem{cor}{\protect\corollaryname}
\theoremstyle{plain}
\newtheorem{prop}{\protect\propositionname}
\newtheorem{rem}{\protect\remarkname}
\theoremstyle{definition}

\@ifundefined{date}{}{\date{}}
\usepackage{jmlr2e}
\title{No Weighted-Regret Learning \\
in Adversarial Bandits with Delays}
\editor{Csaba Szepesvari}
\DeclareMathOperator{\Alg}{Alg}
\usepackage{bbm}

\ShortHeadings{No Weighted-Regret Learning in Adversarial Bandits with Delays}{Bistritz,Zhou,Chen,Bambos,Blanchet}

\author{\name Ilai Bistritz\thanks{This is an extended version of \cite{Bistritz2019}. For details, see Subsection 1.1 on previous work. This research was supported by the Koret Foundation grant for Smart Cities and Digital Living. }
\email bistritz@stanford.edu\\
       \addr Department of Electrical Engineering\\
       Stanford University\\
       Stanford, CA 94305, USA
       \AND
       \name Zhengyuan Zhou\email zzhou@stern.nyu.edu \\
       \addr Stern School of Business\\
       New York University\\
      New York, NY 10003, USA
       \AND
	\name Xi Chen\thanks{Xi Chen would like to thank the support from NSF via IIS-1845444.} \email xc13@stern.nyu.edu \\
       \addr Stern School of Business\\
       New York University\\
       New York, NY 10003, USA
       \AND
	\name Nicholas Bambos \email bambos@stanford.edu\\
       \addr Department of Electrical Engineering\\
       Stanford University\\
       Stanford, CA 94305, USA
       \AND
	\name Jose Blanchet \email jose.blanchet@stanford.edu\\
       \addr Department of Management Science \& Engineering\\
       Stanford University\\
       Stanford, CA 94305, USA
}

\makeatother

\providecommand{\corollaryname}{Corollary}
\providecommand{\definitionname}{Definition}
\providecommand{\lemmaname}{Lemma}
\providecommand{\propositionname}{Proposition}
\providecommand{\theoremname}{Theorem}

\jmlrheading{23}{2022}{1-\pageref{LastPage}}{4/20; Revised
1/22}{3/22}{20-411}{Ilai Bistritz, Zhengyuan Zhou, Xi Chen, Nicholas Bambos, Jose Blanchet}

\begin{document}
\maketitle
\begin{abstract}
Consider a scenario where a player chooses an action in each round $t$ out of $T$ rounds and observes the incurred cost after a delay of $d_{t}$ rounds. The cost functions and the delay sequence are chosen by an adversary. We show that in a non-cooperative game, 
the expected weighted ergodic distribution of play converges to the set of coarse correlated equilibria if players use algorithms that have ``no weighted-regret'' in the above scenario, even if they have linear regret due to too large delays.
For a two-player zero-sum game, we show that no weighted-regret is sufficient for the weighted ergodic average of play to converge to the set of Nash equilibria. We prove that the FKM algorithm with $n$ dimensions achieves an expected regret of $O\left(nT^{\frac{3}{4}}+\sqrt{n}T^{\frac{1}{3}}D^{\frac{1}{3}}\right)$ and the EXP3 algorithm with $K$ arms achieves an expected regret of $O\left(\sqrt{\log  K\left(KT+D\right)}\right)$
even when $D=\sum_{t=1}^{T}d_{t}$ and $T$ are unknown. These bounds use a novel doubling trick that, under mild assumptions, provably retains the regret bound for when $D$ and $T$ are known. Using these bounds, we show that FKM and EXP3 have no weighted-regret even for $d_{t}=O\left(t\log t\right)$. Therefore, algorithms with no weighted-regret can be used to approximate a CCE of a finite or convex unknown game that can only be simulated with bandit feedback, even if the simulation involves significant delays.

\end{abstract}

\section{Introduction}

Consider an agent that makes sequential decisions, and each decision incurs some cost. The agent's goal is to minimize this cost over time. The question of \textbf{what} the agent learns about the cost functions naturally influences the best performance the agent can guarantee. With full information, after acting at round $t$, the agent receives the cost function of round $t$ as feedback. With bandit feedback, as we consider here, the agent only receives the cost of her decision. Another fundamental question is \textbf{when} the agent receives the feedback. In most practical learning environments, an agent does not get to learn the cost of her action immediately. For example, it takes a while to observe the effect of a decision on a treatment plan or before observing the market's response to an advertisement. With delayed feedback, decisions must be made before all the feedback from the past choices is received. 

Practical environments are non-stationary since they typically consist of other learning agents, and the learning of one agent affects that of the others. Moreover, the costs are naturally correlated over time. Hence, guarantees for stochastic environments are not strong enough for multi-agent environments. Instead, we consider cost sequences that are chosen by an adversary that knows the agent's algorithm. Proving regret bounds against an adversary certifies the robustness of a learning algorithm, regardless of whether an actual malicious adversary exists or not. Following the same reasoning, proving regret bounds with adversarial delays certifies the robustness of an algorithm to non-stationary delays. 

An algorithm is said to have "no-regret" \citep{bowling2005convergence} if it 
has a sublinear regret in $T$. It is well known that when $N$ agents in a non-cooperative game each use an algorithm that has no-regret against an adaptive adversary, the ergodic distribution of play converges to the set of coarse correlated equilibria (CCE) \citep{hannan1957approximation,hart2013simple}.
For a two-player zero-sum game, the ergodic average of play converges to the set of Nash equilibria (NE)
\citep{cai2011minmax}. The emergence of a CCE or a NE in a game between no-regret learners establishes their role as predictors for the outcome of the game. From a practical point of view, the convergence of the expected ergodic distribution to the set of CCE or of the ergodic average to
the set of NE makes no-regret algorithms an appealing way to approximate a CCE or a NE when the reward functions are unknown so only simulating the game is possible (see \citet{hellerstein2019solving}). When simulating
an unknown game, bandit feedback is a more realistic assumption than full information (or gradient feedback). Approximating the equilibrium can help to predict the outcome of the interaction between deployed agents even if they use other algorithms than those used for the approximation. If the equilibrium is globally efficient, cooperative agents may agree to play it after using no-regret algorithms to distributedly approximate it first.

The convergence to the set of CCE is maintained if the algorithm still enjoys the no-regret property even with delayed feedback.  However, for large enough delays (e.g. $d_{t}=O\left(t\log t\right)$), the regret of any algorithm becomes linear in the horizon $T$ so the
no-regret property no longer holds. Our first main contribution in this paper is to show that even with delays that cause a linear regret, the expected weighted ergodic distribution may still converge to the set of CCE, and the weighted ergodic average may still converge (in $L^{1}$) to the set of NE for a two-player zero-sum game.

Many practical multi-agent interactions (i.e., games) are complicated to model. Instead, we can simulate the game based on data or in an experiment. Since the agents' performance is only measured in hindsight, such a simulation typically involves delays. If our simulated agents each run an online learning algorithm independently (e.g., FKM or EXP3), we can approximate a CCE of the game (NE for a two-player zero-sum game) by computing the weighted ergodic distribution (average). Our results imply that by properly tuning the weights, this method can approximate an equilibrium even when the standard regret is linear.

Our game-theoretic results motivate analyzing the weighted-regret as opposed to the classical regret when delays are involved. Hence, we study the weighted-regret of some widely-applied algorithms, for both a discrete action set $1,...,K$ (i.e., arms in multi-armed bandits) and a convex and compact action set $\mathcal{K\subset\mathbb{R}}^{n}$. For bandit convex optimization with a convex compact action set, the most widely
used adversarial bandits learning algorithm is FKM \citep{flaxman2005online}.
With no delays, the expected regret of FKM is $O\left(nT^{\frac{3}{4}}\right)$
where $n$ is the dimension of $\mathcal{K}$. For the discrete case,
the most popular adversarial bandits learning algorithm is EXP3 \citep{auer1995gambling,auer2002nonstochastic,bubeck2012regret,neu2010online}.
With no delays, the expected regret of EXP3 is $O\left(\sqrt{TK\log  K}\right)$.

Our second main contribution is to bound the expected weighted-regret of FKM against an adaptive or oblivious adversary. As a special case, we show that with an arbitrary and possibly unbounded sequence of delays $d_{t}$, FKM achieves an expected
regret of $O\left(nT^{\frac{3}{4}}+\sqrt{n}T^{\frac{1}{3}}\left(\sum_{t\notin\mathcal{M}}d_{t}\right)^{\frac{1}{3}}+\left|\mathcal{M}\right|\right)$ against an oblivious adversary,
where $\mathcal{M}=\left\{ t\,|\,t+d_{t}>T,\,t\in\left[1,T\right]\right\}$ is the set of rounds that their feedback is not received
before round $T$. Our third main contribution is to bound the expected weighted-regret of EXP3 against an adaptive or oblivious adversary. As a special case, we show that with an arbitrary and possibly unbounded sequence of delays $d_{t}$, EXP3 achieves an expected
regret of $O\left(\sqrt{\left(KT+\sum_{t\notin\mathcal{M}}d_{t}\right)\log  K}+\left|\mathcal{M}\right|\right)$ against an oblivious adversary. Our weighted-regret bounds reveal for which delay sequences FKM and EXP3 enjoy the no weighted-regret property.

Like the horizon $T$, the sum of delays $D=\sum_{t=1}^{T}\min\left\{ d_{t},T-t+1\right\} $
might be unknown to the decision-maker, which may need them to tune the algorithm. While the
standard doubling trick \citep{cesa1997use} can deal with
an unknown $T$, it does not help with an unknown $D$. Our fourth
main contribution is a general novel two-dimensional doubling trick
where epochs are indexed by a ``delay index'' as well as a ``time
index''. The delay index doubles every time the number of missing
samples so far doubles, and the time index doubles with the rounds
as usual. We show that under mild conditions, this novel doubling trick can be applied to any online learning algorithm with delayed feedback, beyond the case
of adversarial bandits. We apply this result to
achieve an expected regret of $O\left(nT^{\frac{3}{4}}+\sqrt{n}T^{\frac{1}{3}}D^{\frac{1}{3}}\right)$
for FKM and of $O\left(\sqrt{\left(KT+D\right)\log  K}\right)$ for
EXP3. 

\subsection{Previous Work \label{subsec:Previous-Work}}

In recent years, learning with delayed feedback has attracted
considerable attention, ranging from multi-armed bandits (\citet{mandel2015queue}) to Markov decision processes (\citet{neu2010online}) and even distributed optimization (\citet{agarwal2011distributed}).

Most literature on learning with delayed feedback deals with multi-armed bandits, i.e., with a discrete set
of actions. Fixed delays were considered in \citet{weinberger2002delayed}
and \citet{zinkevich2009slow}. Stochastic rewards and stochastic
i.i.d. delays have been considered in \citet{pmlr-v80-pike-burke18a}. Stochastic i.i.d. delays with random missing samples have been considered in \citet{vernade2017stochastic}. Bandits
with adversarial rewards but still stochastic i.i.d. delays were considered
in \citet{joulani2013online}. \citet{cesa2019delay} considered an
interesting case of communicating agents that cooperate to solve
a common adversarial bandit problem, where the messages between agents
may arrive after a bounded delay with a known bound $d$. Recently,
advancements were made for the case of stochastic delays, studying arm-dependent delays (\citet{Manegueu2020}), linear bandits (\citet{Vernade2020}) contextual bandits (\citet{zhou2019learning}), and reward-dependent delays (\cite{lancewicki2021stochastic}).
In contrast, in our scenario, the adversary chooses the delay sequence. 

In \citet{quanrud2015online}, the case of adversarial delays with full information feedback has been considered, where the
feedback is the costs of all arms (or the gradient of the cost function). Our goal is to study bandit feedback instead, motivated by the multi-agent scenario.  

In \citet{cesa2018nonstochastic}, a different adversarial bandits with delayed feedback scenario has been studied, where all the feedback that is received at the same round is summed up and cannot be distinguished, and delays are bounded by $d$. For both the multi-armed and convex cases, \citet{cesa2018nonstochastic} designed a wrapper algorithm and proved a regret bound for their delayed feedback scenario as a function of the regret of the algorithm being wrapped for the no-delay scenario. For EXP3, the resulting regret bound is $O\left(\sqrt{dTK\log  K}\right)$. Compared to their scenario, we consider time-stamped feedback with delays that can be unbounded.

Multi-agent learning and convergence to NE under delays have been considered in \citet{zhou2017countering,heliou2020gradient} for variationally stable games and monotone games (\cite{rosen1965existence}). We study general non-cooperative games and convergence to the set of CCE under delays by generalizing
the framework of no-regret learning to no weighted-regret learning.
Our analysis applies to both the multi-armed bandit and bandit convex optimization cases. 

While their focus is on monotone games, \cite{heliou2020gradient} also prove a regret bound for an FKM-type algorithm for a single agent against an adversary. It is assumed in \cite{heliou2020gradient} that the delay sequence is of the form $d_{t}=O(t^{\alpha})$ for $\alpha<1$, so not even a small subset of the samples can be delayed by $O(t)$ rounds. Their algorithm puts the received samples in a queue and only uses one sample per round regardless of how many samples were received this round. This is different than our FKM version which uses all samples upon reception. The expected regret bound of this queuing version of FKM proved in \cite{heliou2020gradient} is $O(T^{\frac{3}{4}}+T^{\frac{2}{3}}T^{\frac{\alpha}{3}})$, regardless of the sum of delays, so it is much looser than our bound. In particular, if the sequence of delays is mostly zeros, but once in $L$ rounds equals to $t^{\alpha}$, the sum of delays is arbitrarily smaller, depending on $L$.

This paper extends a preliminary conference version \citep{Bistritz2019}
that only analyzed EXP3 under delays. In this journal
version, we also analyze FKM for bandit convex optimization under delays. Additionally, we improve the doubling trick of \citet{Bistritz2019}
and show that it can be applied to any online learning algorithm with
delayed feedback. Last but not least, we generalize the game-theoretic
results of \citet{Bistritz2019} to non-cooperative games and the CCE.

While preparing this journal version, we became aware that the concurrent
work of \citet{thune2019nonstochastic} published in the same conference
as \citep{Bistritz2019} provides a similar analysis for the single-agent
EXP3 case with a constant step-size $\eta_{t}=\eta$. Taking a different
approach to deal with unknown $D=\sum_{t}d_{t}$ and $T$, \citet{thune2019nonstochastic}
assume that the delays are available at action time. In this work,
we instead provide a novel doubling trick that does not require this
assumption and achieves the same $O\left(\sqrt{\left(TK+D\right)\log  K}\right)$
that was achieved in \citet{thune2019nonstochastic}. This improves
the doubling trick that was proposed in \citet{Bistritz2019} that
achieved an expected regret of $O\left(\sqrt{\left(TK^{2}+D\right)\log  K}\right)$.
Replacing EXP3 with a novel follow-the-regulated-leader algorithm, \citet{zimmert2019optimal}
improved the expected regret to the optimal $O\left(\sqrt{TK+D\log  K}\right)$
even when $D$ is unknown without using a doubling trick.

While this paper was under review, \citet{gyorgy2020adapting} have
achieved $O\left(\sqrt{\left(TK+D\right)\log  K}\right)$ regret for
EXP3 by adaptively tuning the step-size. As opposed to our EXP3 bound,
their regret bound has been shown to also hold
with high probability. Furthermore, assuming
a bound on the maximal delay (or that the delay is available at action
time) \citet{gyorgy2020adapting} propose a data-adaptive version
of EXP3 which yields a regret that depends on the cumulative cost. 

The works \citet{thune2019nonstochastic,zimmert2019optimal,gyorgy2020adapting,Bistritz2019}
all study multi-armed bandits, while this paper also studies
convex bandit optimization, using the FKM algorithm under
delayed feedback. Moreover, \citet{thune2019nonstochastic,zimmert2019optimal,gyorgy2020adapting}
only studied the single-agent problem while we are mainly motivated by the multi-agent problem, proving convergence of the expected weighted ergodic distribution to the set of CCE under delayed feedback. Our emphasis on the multi-agent case leads
to two technical differences even in our single-agent results. First, we prove regret bounds also against an adaptive adversary that can choose the cost function in response to the players' past actions.
This distinction between an oblivious and adaptive adversary is necessary
to show convergence to the set of CCE. Additionally, our single-agent
results are formulated using the ``weighted-regret'', which weights
the costs of different turns according to a given weight sequence.
Even for the EXP3 analysis, this formulation leads to several subtleties
that did not arise in \citet{thune2019nonstochastic,zimmert2019optimal,gyorgy2020adapting}
(e.g., in Lemma \ref{Counting Delays Lemma} and Lemma \ref{Multiplicative Lemma}).
Last but not least, this paper provides a novel doubling
trick that can deal with an unknown sum of delays in addition to the unknown horizon. Our novel doubling trick can be applied to any online learning algorithm under delayed feedback, beyond the case of adversarial bandits. For example, the preliminary version of this doubling trick was employed in \cite{lancewicki2020learning} that proposed novel learning algorithms for delayed feedback in adversarial Markov decision processes.

\subsection{Outline}

Section \ref{sec:Problem-Formulation} formalizes the general problem
of learning with delayed bandit feedback and highlights our main results.
Section \ref{sec:Games} discusses the outcome of the interaction between multiple learners that are each subjected to a possibly different delay sequence. We extend the well-known connection between no-regret learning and CCE to learning with delayed feedback. Surprisingly, even algorithms that have linear regret under delays can still lead to the set of CCE. Section \ref{sec:GeneralDoublingTrickDelays} presents our general doubling trick that can be applied to online learning algorithms with delayed feedback, not necessarily in adversarial or bandit feedback environments. Section \ref{sec:Bandit-Convex-Optimization} and Section \ref{sec:EXP3-in-Adversarial}
consider the FKM algorithm for adversarial bandit convex optimization
and the EXP3 algorithm for adversarial multi-armed bandits, respectively.
Section \ref{sec:Bandit-Convex-Optimization} and Section \ref{sec:EXP3-in-Adversarial}
each starts by proving expected weighted-regret bounds under delayed bandit
feedback for the algorithm in consideration, both against
an oblivious and an adaptive adversary. Next, we apply the
result on our doubling trick for both FKM and EXP3 to obtain expected
regret bounds for the case where $T$ and $D$ are unknown. Then,
we show that FKM and EXP3 have no weighted-regret even
with respect to delay sequences for which they both have linear regret in
$T$. This allows us to apply our game-theoretic results for both FKM and EXP3, showing that they can approximate a CCE or a NE (in a two-player zero-sum game) of a  simulated game where only delayed bandit feedback is available. Finally, Section \ref{sec:Conclusions} concludes
the paper. Long proofs are postponed to the appendix. 

\section{Problem Formulation \label{sec:Problem-Formulation}}

Consider a player that in each round $t$ from $1$ to $T$ picks an action $\boldsymbol{a}_{t}\in\mathcal{K}$ from a set $\mathcal{K}$.
The cost at round $t$ from playing $\boldsymbol{a}_{t}$  is $l_{t}\left(\boldsymbol{a}_{t}\right)\in\left[0,1\right]$.
We consider two types of adversaries:
\begin{enumerate}
\item \textbf{Oblivious Adversary}: chooses the cost functions $l_{1},...,l_{T}$
before the game starts. 
\item \textbf{Adaptive Adversary}: chooses the cost function $l_{t}$ after
observing $\left\{ a_{1},\ldots a_{t-1}\right\} $, for each $t$.
\end{enumerate}
With full information and no delays, the player gets to know the function
$l_{t}$ immediately after playing $\boldsymbol{a}_{t}$. In the bandit
delayed feedback scenario, the player only gets to know the value
of $l_{t}\left(\boldsymbol{a}_{t}\right)$ at the beginning of round
$t+d_{t}$. (i.e., after a delay of $d_{t}\geq1$ rounds). The adversary
(oblivious or adaptive) chooses the delay sequence $\left\{ d_{t}\right\} $
before the game starts. 

We assume that the cost feedback includes the timestamp of the action that incurred this cost. This is indeed the case in many applications, such as when robots take physical actions, a recommendation is made to a customer or a treatment is given to a patient. If the delays are
bounded by $d$, \citet{cesa2018nonstochastic} have shown that EXP3
can still be implemented even with no timestamps, with expected regret $O(\sqrt{dTK\log {K}})$
instead of $O\left(\sqrt{\left(d+K\right)T\log {K}}\right)$. For our unbounded delays case, it is not clear if FKM or EXP3 can be implemented
without timestamps. 

The set of costs (feedback samples) received \textbf{and} used at
round $t$ is denoted $\mathcal{S}_{t}$, so $s\in\mathcal{S}_{t}$
means that the cost of $\boldsymbol{a}_{s}$ from round $s$ is received
and used at round $t$. Since the game lasts for $T$ rounds, all
costs for which $t+d_{t}>T$ are never received. Of course, the value
of $d_{t}$ does not matter as long as $t+d_{t}>T$, and these are
just samples that the adversary chose to prevent the player from receiving.
We name these costs the missing samples and denote their set by $\mathcal{M}.$ 

While the rounds of the game are indexed by $t$, it will be useful to our analysis to index a finer time scale that counts the steps of the algorithm for every such $t$. We define $s_{-},s_{+}$ as the steps a moment before and after the algorithm uses the feedback from round $s$, respectively. These steps are taking
place in round $t$ if $s\in\mathcal{S}_{t}$. 

The player wants to have a learning algorithm that uses past observations to make good decisions over time. The performance of the player's algorithm is measured using the regret. The expected
regret is the total expected cost over a horizon of $T$ rounds, compared to the total expected cost that results from playing the best fixed action in hindsight in all rounds:
\begin{defn}
\label{RegretDefinition}Let $\boldsymbol{a}^{*}=\underset{\boldsymbol{a}\in\mathcal{K}}{\arg\min}\sum_{t=1}^{T}l_{t}\left(\boldsymbol{a}\right)$.
The expected regret is defined as
\begin{equation}
\mathbb{E}\left\{ R\left(T\right)\right\} \triangleq\sum_{t=1}^{T}\mathbb{E}\left\{ l_{t}\left(\boldsymbol{a}_{t}\right)-l_{t}\left(\boldsymbol{a}^{*}\right)\right\} \label{eq:1}
\end{equation}
where $\mathbb{E}$ is the expectation over the (possibly) random
actions $\boldsymbol{a}_{1},...,\boldsymbol{a}_{T}$ of the player. 
\end{defn}
We analyze two widely applied algorithms for the two central special
cases of the scenario above: 
\begin{enumerate}
\item \textbf{Bandit Convex Optimization} - $\mathcal{\mathcal{K\subset\mathbb{R}}}^{n}$
is a compact and convex set and $l_{t}:\mathcal{K}\rightarrow\left[0,1\right]$
is convex and Lipschitz continuous with parameter $L$. With no delays,
the FKM algorithm, also known as ``gradient descent without the gradient''
\citep{flaxman2005online}, achieves an expected regret of $O\left(nT^{\frac{3}{4}}\right)$
for this problem.
\item \textbf{Multi-Armed Bandit} - $\mathcal{K}=\left\{ 1,...,K\right\} $,
$l_{t}:\left\{ 1,...,K\right\} \rightarrow\left[0,1\right]$. With
no delays, the EXP3 algorithm \citep{auer2002nonstochastic} achieves
an expected regret of $O\left(\sqrt{TK\log  K}\right)$ for this problem. 
\end{enumerate}

\subsection{Results and Contribution}

Our main results for the single-agent case are summarized and compared to the literature in Table 1. They are based on the regret bounds proven in Theorem \ref{AdaptiveFKMBound} for FKM and Theorem \ref{thm:EXP3-Doubling}
for EXP3 for an unknown $T$ and unknown $D=\sum_{t=1}^{T}\min\left\{ d_{t},T-t+1\right\}$.

The regret bound $O\left(nT^{\frac{3}{4}}+\sqrt{n}T^{\frac{1}{3}}D^{\frac{1}{3}}\right)$ reveals a
remarkable robustness of FKM to delayed feedback. For the sequence $d_{t}=t^{\frac{1}{4}}$, the expected regret maintains the same $O\left(nT^{\frac{3}{4}}\right)$ as in the no-delay case. Even for $d_{t}=t^{\frac{4}{5}}$, the expected regret is
$O\left(nT^{\frac{14}{15}}\right)$, so FKM still has no-regret.

Similarly, the regret bound $O\left(\sqrt{\left(TK+D\right)\log  K}\right)$ reveals a significant robustness of EXP3 to delayed feedback. This
follows since the $T$ term is factored by $K$ while the delay
term $D$ is not. Consider bounded delays of the form $d_{t}=K$.
Then, the order of magnitude of the regret as a function of $T$ and
$K$ is $O\left(\sqrt{TK\log  K}\right)$, exactly as that of EXP3 without
delays. For comparison,
consider the full information case where at each round the costs of
all arms are received. Assume that the player uses the exponential
weights algorithm, which is the equivalent of EXP3 for the full information
case. For the same delay sequence $d_{t}=K$, exponential weights
achieves a regret bound of $O\left(\sqrt{TK\log  K}\right)$,
$\sqrt{K}$ times worse than the $O\left(\sqrt{T\log  K}\right)$ it achieves with no delays. 

Both bandit feedback and delays are obstacles that hurt the performance of the learning of the agent, as reflected in the expected regret. Surprisingly, even when the adversary has control over both of these obstacles, the degradation in the regret is mild.  Intuitively, with bandit feedback,
the effect of delay is much weaker than with full information since less information is delayed. This is an encouraging finding
since practical systems typically have both bandit feedback and delays. 

\begin{table}[t]
\resizebox{\textwidth}{!}{%
\begin{tabular}{|>{\centering}p{1.6cm}|>{\centering}p{3cm}|>{\centering}p{3.5cm}|>{\centering}p{3cm}|>{\centering}p{3.3cm}|>{\centering}p{2.95cm}|}
\hline 
 & \multicolumn{2}{c|}{Convex Optimization} & \multicolumn{3}{c|}{$K$ Arms}\tabularnewline
\hline 
 & OGD

(Gradient Feedback) & FKM

(Bandit Feedback)  & Exponential Weights 

(Full Information)  & EXP3

(Bandit Feedback) & FTRL

(Bandit Feedback)\tabularnewline
\hline 
No-delay & $O\left(\sqrt{T}\right)$\\

\citet{zinkevich2003online} & $O\left(nT^{3/4}\right)$\\

\citet{flaxman2005online} & $O\left(\sqrt{T\log  K}\right)$ & $O\left(\sqrt{TK\log  K}\right)$\\

\citet{auer1995gambling} &  $O\left(\sqrt{TK}\right)$\tabularnewline
\hline 
Adversarial Delays & $O\left(\sqrt{D}\right)$\\
\citet{quanrud2015online} & $O\left(nT^{\frac{3}{4}}+\sqrt{n}T^{\frac{1}{3}}D^{\frac{1}{3}}\right)$

\red{
Theorem \ref{AdaptiveFKMBound}} & $O\left(\sqrt{D\log  K}\right)$\\
\citet{quanrud2015online} & $O\left(\sqrt{\left(TK+D\right)\log K}\right)$\\
\red{Theorem \ref{thm:EXP3-Doubling}}

and \citet{thune2019nonstochastic,gyorgy2020adapting} & $O\left(\sqrt{TK+D\log K}\right)$

\citet{zimmert2019optimal}\tabularnewline
\cline{1-2} \cline{3-6} 
\end{tabular}}

\caption{Expected regret for adversarial bandits (assuming all feedback is received
before $T$, for the ease of comparison of results). For shorthand,
we use $D=\sum_{t=1}^{T}d_{t}$.}
\end{table}

As a benchmark, we provide a simple lower bound of $\Omega\left(\sqrt{D}\right)$ on the expected regret of any algorithm, even for a given $D=\sum_{t=1}^{T}d_{t}$, for multi-armed bandits or convex bandit optimization. With no delays, i.e., $D=T$, the bound coincides with the existing lower bounds that it invokes.

For multi-armed bandits a tight lower bound of $\Omega\left(\sqrt{TK}+\sqrt{D\log K}\right)$ was shown in \cite{zimmert2019optimal} based on the bound in \cite{cesa2019delay}.

For the bandit convex optimization problem, FKM with no delays does not meet the lower bound of $\Omega(\sqrt T)$. However, with delays, the logarithmic gap between our FKM upper bound $O(T^{\frac{3}{4}}+T^{\frac{1}{3}}D^{\frac{1}{3}})$ and the lower bound $\Omega(\sqrt D)$ shrinks. For $D=T^{\frac{5}{4}}$ FKM guarantees $O(T^{\frac{6}{8}})$ instead of $\Omega(T^{\frac{5}{8}})$ and for  $D=T^{\frac{3}{2}}$ FKM guarantees $O(T^{\frac{10}{12}})$ instead of  $\Omega(T^{\frac{9}{12}})$.

\begin{prop}
Consider multi-armed bandits or convex bandit optimization, as defined above. Then for any algorithm and for any integer
$D\geq T$, there exists a sequence of delays $\left\{ d_{t}\right\}$
such that $D-\sqrt{2D}\leq\sum_{t=1}^{T}d_{t}\leq D$ and the expected
regret with an oblivious adversary is $\Omega\left(\sqrt{D}\right)$.
\end{prop}
\begin{proof}
Divide the time horizon into $T_{0}$
and $T_{1}=T-T_{0}$. Set $d_{t}=1$ for $1\leq t\leq T_{0}$ so
there are no delays in this period. Set $d_{t}=T-t+1$ for $T_{0}<t\leq T$
so that the feedback in this period is never received. Then
$\sum_{t=1}^{T}d_{t}=T_{0}+\sum_{t=T_{0}+1}^{T}\left(T-t+1\right)=T_{0}+\frac{\left(T-T_{0}\right)\left(T-T_{0}+1\right)}{2}$ so we can choose $T_{0}=\Theta(T-\sqrt{D})$ such that
$D-\sqrt{2D}\leq T_{0}+\frac{\left(T-T_{0}\right)\left(T-T_{0}+1\right)}{2}\leq D$. The lower bound for either multi-armed bandits \cite[Theorem 15.2]{lattimore2020bandit}
or convex bandit optimization \cite[Theorem 3.2]{hazan2019introduction} is $\Omega\left(\sqrt{T_{0}}\right)$
for some cost functions $l_{1},\ldots,l_{T_{0}}$. The state of the algorithm at
round $T_{0}$ is the initial condition for another adversarial bandit
problem where no feedback is received for an horizon of $T_{1}$ rounds.
Hence, for any initial condition, there exists a cost function $g$ such
that if $l_{T_{0}+1}=...=l_{T}=g$, then the $T_1$ last rounds incur a regret of $\Theta\left(T_{1}\right)=\Theta\left(\sqrt{D-T }\right)$.
\end{proof}

Our results for non-cooperative games with convex cost functions under delays are summarized in Table \ref{tab:Zero-Sum}. They are based on the sufficient conditions
for no weighted-regret for FKM (Lemma \ref{lem:NoErgodicFKM}) and
EXP3 (Lemma \ref{lem:NoErgodicEXP3}). Surprisingly, the delays do
not have to be bounded for the convergence to the set of CCE to hold
(or to the set of NE for a two-player zero-sum game), and they can even increase
as fast as $d_{t}=O\left(t\log t\right)$. Moreover, the feedback
of the players does not need to be synchronized, and they may be subjected
to different delay sequences. If $\frac{d_{t}}{t}\rightarrow0$ as
$t\rightarrow\infty$ the convergence to the set of CCE follows from
the sublinear regret of FKM and EXP3. This is no longer the case for
$d_{t}=\Theta\left(t\right)$ or $d_{t}=\Theta\left(t\log t\right)$, where the regret of FKM, EXP3, or
any other algorithm is $\Theta\left(T\right)$, so the learning against
the adversary fails. Our results show that against other agents the
situation is more optimistic, as the weighted ergodic average can still converge
to the set of CCE (see Proposition \ref{StillConvergesFKM} and Proposition
\ref{StillConvergesEXP3}). To achieve that, agents need to use a time-varying step-size $\eta_{t}$, as can be seen in
Table \ref{tab:Zero-Sum}. In fact, one can go up to $d_{t}=\Theta\left(t\log t\log\left(\log t\right)\right)$
and continue iteratively in this manner, as long as $\sum_{t=1}^{\infty}\frac{1}{d_{t}}=\infty$.
For larger delays, it is not possible to converge to the set of CCE
or NE using our approach. 

The implication of our results is for approximating a CCE (NE) of a (two-player zero-sum) game that we can only simulate "in the lab" based on data or an experiment. In such a scenario, we can only evaluate the agents' performance (i.e., cost) based on the effect of their actions, which means delayed bandit feedback for the agents. We show that algorithms with no weighted-regret can approximate a CCE or a NE even with large delays that yield linear (trivial) regret.

\begin{table}[t]
\resizebox{\textwidth}{!}{%
\begin{tabular}{|>{\centering}p{4cm}|>{\centering}p{2.5cm}|>{\centering}p{2.5cm}|>{\centering}p{2.7cm}|>{\centering}p{3.6cm}|}
\hline 
 & $d_{t}\leq t^{\frac{1}{4}}$ & $d_{t}\leq t^{\frac{3}{4}}$ & $d_{t}\leq t$ & $d_{t}\leq t\log t$\tabularnewline
\hline 
Parameters for no weighted-regret: FKM & $\eta_{t}=\frac{1}{t^{\frac{5}{8}}\log\left(t+1\right)}$

$\delta=T^{-\frac{1}{8}}$ & $\eta_{t}=\frac{1}{t^{\frac{7}{8}}\log\left(t+1\right)}$

$\delta=T^{-\frac{1}{24}}$ & $\eta_{t}=\frac{1}{t\log\left(t+1\right)}$

$\delta=\left(\log\log T\right){}^{-\frac{1}{3}}$ & $\eta_{t}=\frac{1}{t\log\left(t+1\right)\log\log\left(t+1\right)}$

$\delta=\left(\log\log\log T\right)$$^{-\frac{1}{3}}$\tabularnewline
\hline 
Distance from CCE: FKM & $O\left(\frac{\log T}{T^{\frac{1}{8}}}\right)$ & $O\left(\frac{\log T}{T^{\frac{1}{24}}}\right)$ & $O\left(\frac{1}{\left(\log\log T\right)^{\frac{1}{3}}}\right)$ & $O\left(\frac{1}{\left(\log\log\log T\right)^{\frac{1}{3}}}\right)$\tabularnewline
\hline 
FKM Regret & $O\left(nT^{\frac{3}{4}}\right)$ & $O\left(nT^{\frac{11}{12}}\right)$ & $O\left(T\right)$ & $O\left(T\right)$\tabularnewline
\hline 
Parameters for no weighted-regret: EXP3 & $\eta_{t}=\frac{1}{t^{\frac{5}{8}}\log\left(t+1\right)}$ & $\eta_{t}=\frac{1}{t^{\frac{7}{8}}\log\left(t+1\right)}$ & $\eta_{t}=\frac{1}{t\log\left(t+1\right)}$ & $\eta_{t}=\frac{1}{t\log\left(t+1\right)\log\log\left(t+1\right)}$\tabularnewline
\hline 
 Distance from CCE: EXP3 & $O\left(\frac{\log T}{T^{\frac{3}{8}}}\right)$ & $O\left(\frac{\log T}{T^{\frac{1}{8}}}\right)$ & $O\left(\frac{1}{\log\log T}\right)$ & $O\left(\frac{1}{\log\log\log T}\right)$\tabularnewline
\hline 
EXP3 Regret & $O\left(T^{\frac{5}{8}}\sqrt{\log  K}\right)$ & $O\left(T^{\frac{7}{8}}\sqrt{\log  K}\right)$ & $O\left(T\right)$ & $O\left(T\right)$\tabularnewline
\hline 
\end{tabular}}

\caption{\label{tab:Zero-Sum}Conditions for no weighted-regret for different
delay sequences, along with the corresponding single agent expected
regret bounds. For shorthand, we use $D=\sum_{t=1}^{T}d_{t}$.}
\end{table}

\section{Non-cooperative Games with Delayed Bandit Feedback\label{sec:Games}}

One of the main reasons why adversarial regret bounds are needed is
that practical environments consist of multiple interacting agents,
leading to non-stationary reward processes. In this section, we study a non-cooperative game where each player only receives delayed bandit feedback, given some arbitrary sequence of delays that can be different for different players. 

It is well known that without delays, players that use an online learning algorithm with sublinear regret (i.e., no-regret) against an adaptive adversary will converge to the set of
CCE in the empirical distribution sense \citep{hannan1957approximation,hart2013simple},
and to the set of NE for a two-player zero-sum game \citep{blackwell1956analog}. With
large enough delays, the regret becomes linear in $T$ so there is no guarantee that the dynamics converge to the set of CCE or NE in any sense. Surprisingly, we show that a CCE (or a NE for a two-player zero-sum game) can still emerge even with linear regret that results from too large delays. Our weighted-regret
bounds for FKM and EPX3 provide sufficient conditions under which
CCE or a NE can be approximated this way for a convex or finite game,
respectively. In this sense, the weighting in the weighted ergodic distribution "filters out" part of the delay noise in the approximation of the CCE/NE.  
In this section, we formulate our results for general continuous games and then explain how finite games can be viewed as a special case, using mixed actions.

Our key observation is that with delayed feedback, it is not the regret that matters for the game dynamics but rather what we call the weighted-regret. The weighted-regret weighs the costs in different
rounds according to a given non-increasing sequence $\eta_{t}$ so it coincides
with the regret when $\eta_{t}=\eta,\forall t$. We define ``no weighted-regret'' to replace the traditional no-regret property: 
\begin{defn}
Let $\left\{ l_{t}:\mathcal{K}\rightarrow\left[0,1\right]\right\}_{t}$ be a sequence of cost functions, chosen by an adaptive adversary. Let $\boldsymbol{a}^{*}=\underset{\boldsymbol{a}\in\mathcal{K}}{\arg\min}\sum_{t=1}^{T}\eta_{t}l_{t}\left(\boldsymbol{a}\right)$. Let $\left\{ d_{t}\right\} $ be a delay sequence such that the cost
from round $t$ is received at round $t+d_{t}$. We say that an algorithm that produces the random sequence of (single-agent) actions $\left\{ \boldsymbol{a}_{t}\right\}$
has no weighted-regret with respect to $\left\{ d_{t}\right\}$ and the non-increasing weight sequence $\left\{ \eta_{t}\right\}$ if
\begin{equation}
\underset{T\rightarrow\infty}{\lim}\mathbb{E}\left\{ \frac{\sum_{t=1}^{T}\eta_{t}\left(l_{t}\left(\boldsymbol{a}_{t}\right)-l_{t}\left(\boldsymbol{a}^{*}\right)\right)}{\sum_{t=1}^{T}\eta_{t}}\right\} =0\label{eq:2}
\end{equation}
where the expectation is with respect to the random $\left\{ \boldsymbol{a}_{t}\right\}$ generated by the algorithm.  
\end{defn}

Having no weighted-regret is only non-trivial when $\sum_{t=1}^{\infty}\eta_{t}=\infty$. When $\sum_{t=1}^{\infty}\eta_{t}<\infty$,
the feedback from the last $\frac{T}{2}$ rounds can be discarded without affecting \eqref{eq:2}. When $\sum_{t=1}^{\infty}\eta_{t}=\infty$, there is no round $t$ after which we can discard all feedback and still maintain \eqref{eq:2}. 

We define no weighted-regret with respect to an adaptive adversary since, in a non-cooperative game with cost functions $\{u_n\}$, the "adversarial" cost function of player $n$ is $l_t(\boldsymbol{a}_{n})=u_{n}(\boldsymbol{a}_{n},\boldsymbol{a}_{-n,t})$, so it is determined by the actions of the other players. In turn, the actions of the other players depend on the past actions of player $n$. Hence, the cost function of player $n$ depends on her past actions, as with an adaptive adversary. In general, the equilibrium does not consist of absolute strategies such as "min-max" in zero-sum games. Hence, regret guarantees against an adaptive adversary allow proving convergence to the set of CCE for general non-cooperative games.

When taking the limit $T\rightarrow\infty$, it is important to emphasize
that we do not change the infinite sequence of delays $\left\{ d_{t}\right\}$,
but only reveal more elements in this sequence. In other words, we
are looking at the same game but over a longer time horizon. Therefore,
while $d_{t}=\frac{T}{2}$ makes sense for a constant $T$, it is
misleading when taking $T\rightarrow\infty$ since it represents a delay that occurred at time $t$ but changes with the limit, so the
limit is no longer of the ``same game''. 

\subsection{Coarse Correlated Equilibrium for $N$-player Games}

The CCE is a well-established equilibrium
concept for learning in games \citep{hannan1957approximation,ashlagi2008value,hart2013simple}. Our convergence argument utilizes the notion of an $\varepsilon$-CCE:
\begin{defn}\label{def:CCE}
Consider a non-cooperative game where the action set of all players is some compact set $\mathcal{A}$ and the reward function of each player  $u_{n}:\mathcal{A}^{N}\rightarrow\left[0,1\right]$
is continuous. 
Recall that $\boldsymbol{a}_{-n}$ is the action profile of all players except player $n$. Let $\mathcal{P}\left(\mathcal{A}^{N}\right)$ be the set of all Borel probability
measures over $\mathcal{A}^{N}$, equipped with the weak-{*} topology
(see \citet{simmons1963introduction}). The set of all $\varepsilon$-CCE points is
the set of distributions over $\mathcal{A}^{N}$ such that:
\begin{equation}
\mathcal{C_{\varepsilon}}=\left\{ \rho\in\mathcal{P}\left(\mathcal{A}^{N}\right)\,\middle|\,\mathbb{E}^{\boldsymbol{a}^{*}\sim\rho}\left\{ u_{n}\left(a_{n}^{*},\boldsymbol{a}_{-n}^{*}\right)\right\} \geq\underset{a_{n}\in\mathcal{A}}{\max}\,\mathbb{E}^{\boldsymbol{a}^{*}\sim\rho}\left\{ u_{n}\left(a_{n},\boldsymbol{a}_{-n}^{*}\right)\right\} -\varepsilon\,,\forall n\right\} \label{eq:3}
\end{equation}
and the set of CCE points is $\text{\ensuremath{\mathcal{C}_{0}}}$
with $\varepsilon=0$.
\end{defn}

The $\varepsilon-$CCE is a distribution over the action profiles such that no player can improve her expected reward by more than $\varepsilon$ by playing any pure action if other players keep playing according to this distribution. The CCE can be interpreted as a coordinator that uses a random signal to instruct the players what to play such that they all want to follow this instruction given that the others do. This equilibrium is called "correlated" since the actions of the players are statistically dependent, potentially through the coordinator's signal they all observe. The history of the game, and even as little as the bandit feedback each player received in the past, can implement such a coordinator. 

The CCE should not be confused with the more refined correlated equilibrium (CE). The CCE coincides with the CE if only constant departure functions are considered in the definition of the CE \cite[Page 11]{hart2013simple}, instead of all measurable mappings. Hence, by definition, every CE is a CCE. We focus on the CCE since it relates to the regret (Definition \ref{RegretDefinition}) directly, while the CE is related to the internal regret which is less common for online learning algorithms (\cite{blum2007external}).

In the game of Definition \ref{def:CCE}
a CCE always exists since Theorem \citep[Theorem 3]{hart2013simple} shows that a correlated equilibrium exists. Hence,  $\mathcal{C}_{0}$ is non-empty. 

The entity that converges to the set of CCE $\mathcal{C}_{0}$ in our non-cooperative game scenario is the expected weighted ergodic distribution of the actions
$\boldsymbol{a}_{t}$. For the special case of $\eta_{t}=\eta$
for all $t$ for some $\eta>0$, the weighted ergodic distribution of $\boldsymbol{a}_{t}$
is simply its ergodic (i.e., empirical) distribution.
\begin{defn}
For a weight sequence $\left\{ \eta_{t}\right\}$ and horizon
$T$, the weighted ergodic distribution of a sequence of actions $\left\{ \boldsymbol{a}_{t}\right\}$ is defined as:
\begin{equation}
\rho_{T}\triangleq\frac{\sum_{t=1}^{T}\eta_{t}\delta_{\boldsymbol{a}_{t}}}{\sum_{t=1}^{T}\eta_{t}}\label{eq:4}
\end{equation}
\end{defn}
where $\delta_{\boldsymbol{a}_{t}}$ is Dirac's measure, so  $\delta_{\boldsymbol{a}_{t}}\left(A\right)=1$ if $\boldsymbol{a}_{t}\in A$ and $\delta_{\boldsymbol{a}_{t}}\left(A\right)=0$ otherwise.
Note that $\mathbb{E}\left\{ \delta_{\boldsymbol{a}_{t}}\right\} =\mathbb{E}\left\{ p_{\boldsymbol{a}_{t}}\right\}$ where $p_{\boldsymbol{a}_{t}}$ is the distribution of $\boldsymbol{a}_{t}$ given the information the algorithm has at round $t$. Hence, we can use $\frac{\sum_{t=1}^{T}\eta_{t}p_{\boldsymbol{a}_{t}}}{\sum_{t=1}^{T}\eta_{t}}$ instead to estimate a CCE, which exploits more information. 

The following theorem establishes the convergence of $\mathbb{E}\left\{ \rho_{T}\right\} $ to the set of CCE $\mathcal{C}_{0}.$ 
\begin{thm}
\label{CorrelatedEquilibrium}Let $N$ players play a non-cooperative game where the action set of all players $\mathcal{A}$ is compact and the reward function of each player  $u_{n}:\mathcal{A}^{N}\rightarrow\left[0,1\right]$ is continuous. Let $\left\{ \eta_{t}\right\}$
be the non-increasing weight sequence. If each player $n$ runs an algorithm that has no weighted-regret with respect to its delay sequence $\left\{ d_{t}^{n}\right\} _{t}$
then $\mathbb{E}\left\{ \rho_{T}\right\} $ converges to the set of CCE $\mathcal{C}_{0}$ as $T\rightarrow\infty$. 
\end{thm}
\begin{proof}
See Appendix.
\end{proof}
The expectation $\mathbb{E}\left\{ \rho_{T}\right\} $ is with respect to the random actions. By definition, the set of $\varepsilon$-CCE
is convex, so the average of multiple $\varepsilon$-CCE is
also an $\varepsilon$-CCE. Hence, the implication of Theorem \ref{CorrelatedEquilibrium}
is that to approximate a CCE using $T$ samples, one can run $\sqrt{T}$ independent simulations of the game (possibly not identically distributed) and then average the resulting $\left\{ \rho_{\sqrt{T}}^{\left(i\right)}\right\} _{i=1}^{\sqrt{T}}$. From the strong law of large numbers,
this estimation converges as $T\rightarrow\infty$ with probability 1 to $\mathcal{C}_{0}$
since $\rho_{\sqrt{T}}^{\left(i\right)}$ is bounded
and $\left\{ \rho_{\sqrt{T}}^{\left(i\right)}\right\} _{i=1}^{\sqrt{T}}$
are independent. 

\begin{rem}[Finite Games]\label{rem:Finite Games}
One useful special case of Theorem \ref{CorrelatedEquilibrium} is that of finite games (i.e.,  multi-armed bandits). To see that, choose the action set to be the $K$-dimensional simplex, i.e.,  $\mathcal{A}=\Delta^{K}$. Let $U_{n}:\{1,\ldots,K \}^{N}\rightarrow[0,1]$ be the reward function of player $n$ in the finite game. Then we can define the reward function of the continuous game $u_{n}:\mathcal{A}^{N}\rightarrow\left[0,1\right]$ for every $\boldsymbol{x}\in \Delta^{K}$ as follows:

\begin{equation}
    u_{n}\left(\boldsymbol{x}\right)\triangleq\mathbb{E}^{\boldsymbol{a}\sim\boldsymbol{x}}\left\{ U_{n}\left(\boldsymbol{a}\right)\right\}
\end{equation}

where the expectation averages $\boldsymbol{a} \in \{1,\ldots,K \}^{N}$ according to the distribution that $\boldsymbol{x}$ defines on ${1,\ldots,K}$. This is indeed a special case of our formulation above since the simplex $\Delta^{K}$ is compact, and $u_{n}\left(\boldsymbol{x}\right)$ is linear and therefore continuous. Moreover, we have

\begin{equation}\label{eq:6a}
    \mathbb{E}^{\boldsymbol{x}^{*}\sim\rho}\left\{ u_{n}\left(\boldsymbol{x}^{*}\right)\right\} =\mathbb{E}^{\boldsymbol{x}^{*}\sim\rho}\left\{ \mathbb{E}^{\boldsymbol{a}^{*}\sim\boldsymbol{x}^{*}}\left\{ U_{n}\left(\boldsymbol{a}^{*}\right)\right\} \right\} =\mathbb{E}^{\boldsymbol{a}^{*}\sim\phi\left(\rho,\boldsymbol{x}^{*}\right)}\left\{ U_{n}\left(\boldsymbol{a}^{*}\right)\right\} 
\end{equation}

where $\phi\left(\rho,\boldsymbol{x}\right)$ is the distribution over $\{1,\ldots,K \}$ that is induced by randomizing $\boldsymbol{x}$ according to $\rho$ and then randomizing $\boldsymbol{a}$ according to $\boldsymbol{x}$ (i.e., the compound distribution with $\boldsymbol{x}$ as the random parameter).
Since the maximum of $u_{n}$ is attained at the corners of the simplex, we also have

\begin{equation}\label{eq:7a}
    \underset{\boldsymbol{x}_{n}\in\mathcal{A}}{\max}\,\mathbb{E}^{\boldsymbol{x}^{*}\sim\rho}\left\{ u_{n}\left(\boldsymbol{x}_{n},\boldsymbol{x}_{-n}^{*}\right)\right\} =\underset{a_{n}\in\left\{ 1,\ldots,K\right\} }{\max}\,\mathbb{E}^{\boldsymbol{a}^{*}\sim\phi\left(\rho,\boldsymbol{x}^{*}\right)}\left\{ U_{n}\left(a_{n},\boldsymbol{a}_{-n}^{*}\right)\right\}. 
\end{equation}

Due to \eqref{eq:6a} and \eqref{eq:7a}, Definition \ref{def:CCE} coincides with the definition of a CCE for a finite game, which results from Definition \ref{def:CCE} with $\mathcal{A}=\{1,\ldots,K\}$ and any arbitrary reward functions $\{u_n:\{1,\ldots,K\}^{N}\rightarrow[0,1]\}_n$. Then $\mathcal{P}\left(\mathcal{A}^{N}\right)$ takes the form of the $K$-dimensional simplex.

It is important to notice that FKM cannot be applied in a finite game with bandit feedback. While FKM can certainly be applied to linear cost functions, it would require the player to obtain $u_n({\boldsymbol{x}})$ as feedback, which is the expected cost incurred to a player that only picks one discrete arm each turn at random, according to a distribution $\boldsymbol{x}_{n}$. For this reason, we also analyze the weighted-regret of the EXP3 algorithm, which can work with discrete arm choices and bandit feedback. Remarkably, requiring less feedback, EXP3 also achieves better expected regret for this linear case. If $l_{t}^{\left(i\right)}$ is the cost of playing arm $i$ at round $t$, then the expected cost of playing a random arm $a_{t}$ according to the distribution $\boldsymbol{x}_{t}\in\Delta^{K}$  is 

\begin{equation}
    \mathbb{E}\left\{ l_{t}^{(a_{t})}\right\} =\mathbb{E}\left\{ \mathbb{E}\left\{ \sum_{i=1}^{K}x_{t}^{\left(i\right)}l_{t}^{\left(i\right)}\mid\boldsymbol{x}_{t}\right\} \right\} =\mathbb{E}\left\{ \sum_{i=1}^{K}x_{t}^{\left(i\right)}l_{t}^{\left(i\right)}\right\}. 
\end{equation}

Hence, EXP3 also leads to faster convergence to the set of CCE in finite games.
\end{rem}

The general result of Theorem \ref{CorrelatedEquilibrium} implies
stronger results for special classes of games where the set of CCE
has an interesting structure. For example, in strictly monotone games the unique CCE places probability one on the unique pure NE (\cite{ui2008correlated}). Another example is a polymatrix game, which is a finite action set game where each player plays a separate two-player zero-sum
game against each of her neighbors on a given graph. For polymatrix games, for which a two-player zero-sum game is a special case, it was shown in
\citet{cai2011minmax} that the marginal distributions of the CCE
are a NE. However, this result holds only for multi-armed bandits since it assumes discrete action sets. Our next section
establishes that the weighted ergodic average of two no weighted-regret algorithms in a two-player zero-sum game converges to the set of NE for multi-armed bandits \textbf{and} bandit convex optimization.

\subsection{Nash Equilibrium for Two-Player Convex-Concave Zero-Sum Games}

In this subsection, we consider two-player zero-sum games where the action set $\mathcal{A}\subset \mathbb{R}^{n}$ of both players is convex and compact. The cost function $u:\mathcal{A}\times\mathcal{A}\rightarrow [0,1]$ is assumed to be continuous, convex in the first argument and concave in the second. When the row player plays $\boldsymbol{y}$ and the
column player plays $\boldsymbol{z}$, the first pays a cost of $u\left(\boldsymbol{y},\boldsymbol{z}\right)$
and the second gains a reward of $u\left(\boldsymbol{y},\boldsymbol{z}\right)$.

For two-player zero-sum games, we show that algorithms with no weighted-regret lead to the set of Nash equilibria (NE). A NE is an action profile
such that no player wants to switch an action given that the other
players keep their actions. For our convergence argument, we define the
set of all approximate (pure) NE of a two-player zero-sum game:
\begin{defn}\label{def:NE}
Define a two-player zero-sum game where the action set of both players is $\mathcal{A}\subset \mathbb{R}^{n}$ and the cost function is $u:\mathcal{A}\times\mathcal{A}\rightarrow [0,1]$. The set of all $\varepsilon$-NE points of this game is defined as:
\begin{equation}
\mathcal{N_{\varepsilon}}=\left\{ \left(\boldsymbol{y}^{*},\boldsymbol{z}^{*}\right)\in\mathcal{A}\times\mathcal{A}\,\mid\,u\left(\boldsymbol{\boldsymbol{y}^{*},\boldsymbol{z}^{*}}\right)\leq\underset{\boldsymbol{y}\in\mathcal{A}}{\min}\,u\left(\boldsymbol{\boldsymbol{y},\boldsymbol{z}^{*}}\right)+\varepsilon,u\left(\boldsymbol{\boldsymbol{y}^{*},\boldsymbol{z}^{*}}\right)\geq\underset{\boldsymbol{z}\in\mathcal{A}}{\max}\,u\left(\boldsymbol{\boldsymbol{\boldsymbol{y}^{*}},\boldsymbol{z}}\right)-\varepsilon\right\} \label{eq:6}
\end{equation}
and the set of NE points is $\mathcal{N}_{0}$ with $\varepsilon=0$.
\end{defn}
The NE is a more exclusive solution concept than the CCE, so our result is stronger for this special case of a two-player zero-sum game. This holds since if a NE exists, it is always a CCE, which follows immediately from Definition \ref{def:CCE} and  Definition \ref{def:NE} by substituting the distribution that gives the NE action profile with probability 1 as the distribution $\rho$. 
For a two-player zero-sum game with a convex and continuous cost function and compact action sets, a pure
NE always exists \citep{nikaido1955note, debreu1952social} so $\mathcal{N}_0$ is non-empty. 

It was shown in \citet{bailey2018multiplicative} that for the
no-delay case, the last iterate $\left(\boldsymbol{y}_{t},\boldsymbol{z}_{t}\right)$ does not converge in general to a NE and even moves away from it. Instead, it is the ergodic average action that converges to the set of NE $\mathcal{N}_0$.
With delayed feedback, the entity that converges to $\mathcal{N}_0$
in our two-player zero-sum game scenario is the weighted ergodic average of
the actions $\left\{ \boldsymbol{a}_{t}\right\} _{t}$. For the
special case of $\eta_{t}=\eta$ for all $t$ for some $\eta>0$, the weighted
ergodic average of $\boldsymbol{a}_{t}$ is just its ergodic average.
\begin{defn}
For a weight sequence $\left\{ \eta_{t}\right\} $ and horizon
$T$, the weighted ergodic average of a sequence of actions $\left\{ \boldsymbol{a}_{t}\right\}$
is defined as
\begin{equation}
\boldsymbol{\bar{a}}_{T}\triangleq\frac{\sum_{t=1}^{T}\eta_{t}\boldsymbol{a}_{t}}{\sum_{t=1}^{T}\eta_{t}}.\label{eq:7}
\end{equation}
\end{defn}
Then, the following theorem establishes the convergence to the set
of NE $\mathcal{N}_{0}$. Following Remark \ref{rem:Finite Games}, we can apply Theorem \ref{ZeroSumTheorem} to a finite game with cost table $U: \{1,\ldots,K \}\times\{1,\ldots,K \}\rightarrow [0,1]$ such that $\boldsymbol{y},\boldsymbol{z}$ are mixed actions (distributions over $\{1,\ldots,K \})$, and  $u\left(\boldsymbol{y},\boldsymbol{z}\right)\triangleq\sum_{i=1}^{K}\sum_{j=1}^{K}y^{\left(i\right)}z^{\left(j\right)}U\left(i,j\right)$. 

Since the players' payoffs ("value of the game") are the same in all NE of a two-player zero-sum game \cite[Page 182]{cesa2006prediction}, algorithms with no weighted-regret can be used to approximate this outcome, regardless of the NE that the weighted ergodic average approximates.
\begin{thm}
\label{ZeroSumTheorem} Let two players play a zero-sum game with a convex and compact action set $\mathcal{A}\subset \mathbb{R}^{n}$ and a cost function $u\left(\boldsymbol{y},\boldsymbol{\boldsymbol{z}}\right):\mathcal{\mathcal{A}}\times\mathcal{\mathcal{A}}\rightarrow\left[0,1\right]$. Assume that $u\left(\boldsymbol{y},\boldsymbol{\boldsymbol{z}}\right)$ is convex in $\boldsymbol{y}$ and concave in $\boldsymbol{z}$ and
is continuous. Let $\boldsymbol{y}_{t}$ and $\boldsymbol{z}_{t}$
be the actions of the row and column players at round $t$, and let $\boldsymbol{\bar{y}}_{T}$ and $\boldsymbol{\bar{\boldsymbol{z}}}_{T}$ be their weighted ergodic averages.
Let $\left\{ \eta_{t}\right\}$ be the non-increasing weight sequence. Let $\left\{ d_{t}^{r}\right\} $
and $\left\{ d_{t}^{c}\right\} $ be the delay sequence of the row
player and the column player. If both players use a no weighted-regret algorithm with respect to $\left\{ d_{t}^{r}\right\} ,\left\{ d_{t}^{c}\right\} $
then, as $T\rightarrow\infty$:
\begin{enumerate}
\item $\left(\boldsymbol{\bar{y}}_{T},\boldsymbol{\bar{\boldsymbol{z}}}_{T}\right)$
converges in $L^{1}$ to the set of Nash equilibria $\mathcal{N}_{0}$
of the two-player zero-sum game.
\item $U\left(\boldsymbol{\bar{y}}_{T},\boldsymbol{\bar{\boldsymbol{z}}}_{T}\right)$
converges in $L^{1}$ to the value of the game $\underset{\boldsymbol{y}}{\min}\underset{\boldsymbol{z}}{\max}\,U\left(\boldsymbol{y},\boldsymbol{z}\right)=\underset{\boldsymbol{\boldsymbol{z}}}{\max}\underset{\boldsymbol{y}}{\min}\,U\left(\boldsymbol{y},\boldsymbol{\boldsymbol{z}}\right)$.
\end{enumerate}
\end{thm}
\begin{proof}
See Appendix.
\end{proof}
\section{Doubling Trick for Online Learning with Delays \label{sec:GeneralDoublingTrickDelays}}

Online learning under delayed feedback introduces another key parameter, which is the sum of delays $D=\sum_{t=1}^{T}d_{t}$. The sum of delays appears in the expected regret bound of many online algorithms and is required to tune their step-size and other parameters. If $D$ or a tight upper bound for it is unknown, then an adaptive algorithm is needed. 

With no delays, the standard doubling trick (see \citet{cesa1997use})
can be used if $T$ is unknown. However, the same doubling trick does not work with delayed feedback. We now present a novel doubling trick for the delayed feedback case, where $T$ and $D$ are unknown. 

Our enhanced doubling trick is two-dimensional, as each epoch is indexed by a delay index $w$ as well as a time index $h$. Compared to that, our previous doubling trick in \citet{Bistritz2019} only tracked a delay index $w$. As a result, the new doubling trick leads to tighter regret bounds. For example, it yields $O\left(\sqrt{\left(TK+D\right)\log  K}\right)$ instead of $O\left(\sqrt{\left(TK^{2}+D\right)\log  K}\right)$ for EXP3. More importantly, the doubling trick in \citet{Bistritz2019} required to analyze the regret of the algorithm that employed it. In comparison, the new doubling trick is plug and play since it provably retains the regret bound for when $D$ and $T$ are known for a wide class of online learning algorithms.

Our doubling trick assumes that a regret bound of the form $O\left(k_{1}D^{d}+k_{2}T^{c}+k_{3}T^{a}D^{b}\right)$ is available if $T,D$ are known, for some constants $k_{1},k_{2},k_{3}\geq0$ and $0\leq a,b,c,d\leq1$. We divide the time horizon into super-epochs, indexed by $\nu$, each consists of epochs as explained below. A super-epoch is a set $\mathcal{E}_{\nu}$ of consecutive rounds that all use the same algorithm parameters
$\mathcal{P}_{\nu}$ (e.g., a step-size $\eta_{\nu}$). Let $\nu_{t}$ be the index of the super-epoch that contains round $t$. Let $m_{t}$ be the number of missing feedback samples at round $t$. A missing feedback sample at round $t$ is a sample from $\tau\leq t$ such that $\nu_{t}=\nu_{\tau}$ (i.e., belongs to the same super-epoch) that was not received before or at round $t$. 

An epoch is the set of consecutive rounds where the sum of delays is
within a given interval and the time index is within another given
interval. To enable that, we employ a delay index counter $w$ and a time index counter $h$. We increase $w$ every time when $\sum_{\tau=1}^{t}m_{\tau}$, that
tracks $D$, doubles. We increase $h$ every time when
the number of rounds $t$ doubles. We then define the $\left(h,w\right)$
epoch as 
\begin{equation}
\mathcal{T}_{h,w}=\left\{ t\,|2^{w-1}\leq\sum_{\tau=1}^{t}m_{\tau}<2^{w},\,2^{h-1}\leq t<2^{h}\right\}. \label{eq:9}
\end{equation}

The $\left(h,w\right)$ epochs are then partitioned into super-epochs as follows (also illustrated in Fig. 1):

\begin{itemize}
    \item For each $h$, the super-epoch $\mathcal{E}=\mathcal{S}_{h}$ is the set of $\left(h,w\right)$ such that $2k_{2}2^{ch}\geq k_{1}2^{dw}+k_{3}2^{ah}2^{bw}$.

    \item For each $w$, the super-epoch $\mathcal{E}=\mathcal{S}_{w}$ is the set of $\left(h,w\right)$ such that $2k_{1}2^{dw}\geq k_{2}2^{ch}+k_{3}2^{ah}2^{bw}$.
    
    \item All other epochs $\left(h,w\right)\notin\mathcal{S}_{h}\cup\mathcal{S}_{w}$ are each a separate super-epoch $\mathcal{E}=\left\{ \left(h,w\right)\right\}$.
\end{itemize}

During the $\left(h,w\right)\in\mathcal{E}_{\nu}$ epoch, the algorithm equipped with our doubling trick uses the parameters $\mathcal{P}_{\nu}$ (e.g., $\mathcal{P}_{\nu}=\left\{ \eta_{\nu},\delta_{\nu}\right\}$ for
FKM). Different epochs $\left(h_{1},w_{1}\right)$
and $\left(h_{2},w_{2}\right)$ in the same super-epoch use the same
set of parameters. As shown in Fig. \ref{fig:2D-Doubling-Trick}, this can only be the case if $h_{1}=h_{2}$ or $w_{1}=w_{2}$. Feedback samples from previous super-epochs are discarded once received, and are no longer counted in $\sum_{\tau=1}^{t}m_{\tau}$ after their super-epoch has ended. The resulting algorithm is detailed in Algorithm \ref{alg:DoublingTrickWrapper}.

Fig. \ref{fig:2D-Doubling-Trick} illustrates the partition of epochs into super epochs used for our doubling trick. We can see that the $\left(h,w\right)$ epoch space
is split into three regions with three different types of super-epochs.
In the upper one (in blue) $D^{d}$ dominates the regret, in the middle one (in pink) $T^{a}D^{b}$ dominates the regret and in the lower one (in orange) $T^{c}$ dominates the regret. Each blue, pink, or orange box is a super-epoch on which we apply our regret bound separately. The grey arrows represent the actual path that the epoch indices $\{\left(h,w\right)\}$ went through. 

\begin{figure}[tbh]
~~~~~~~~~~~~~~~~~~~~~~~~~~~~~~~~~~~~~~~~~~~~\includegraphics[width=5cm,height=5cm,keepaspectratio]{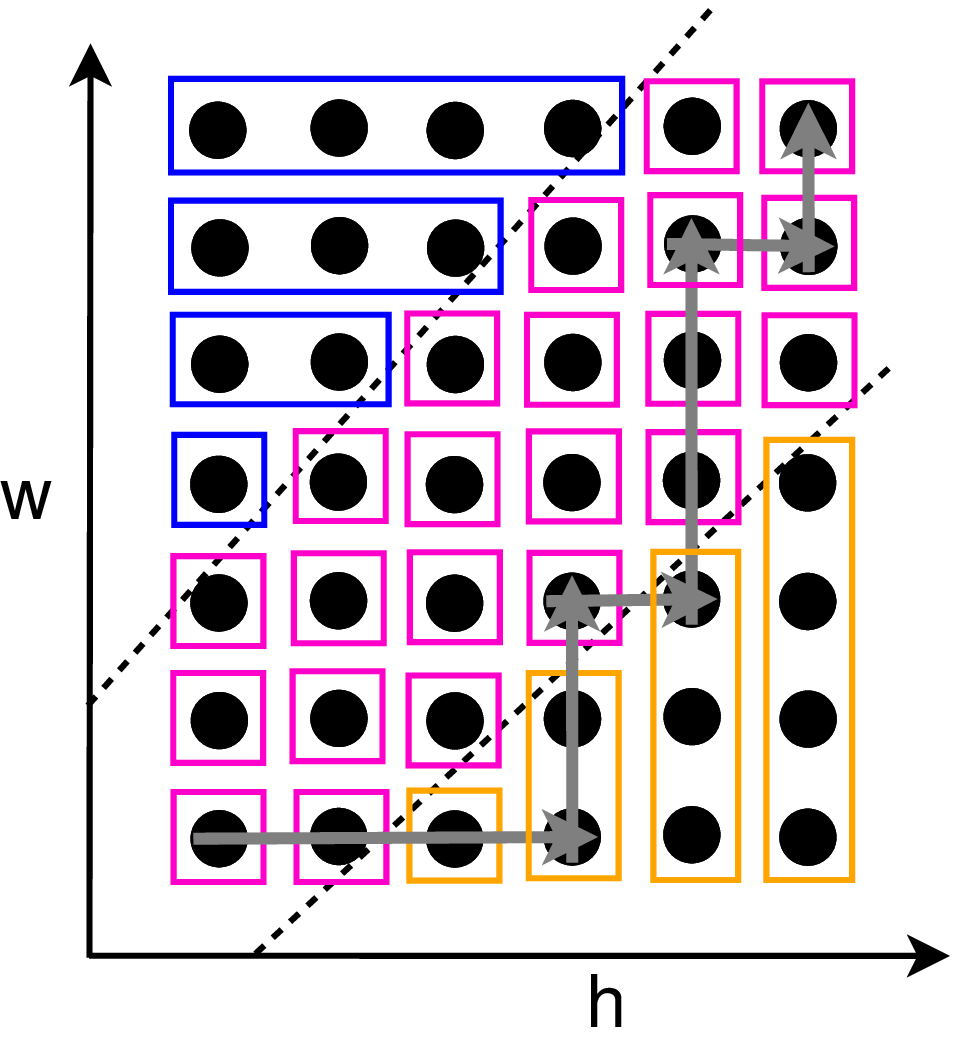}

\caption{\label{fig:2D-Doubling-Trick}2D Doubling Trick. Each box is a different super-epoch, and each dot is a different epoch. }

\end{figure}

\begin{algorithm}[t]
\caption{\label{alg:DoublingTrickWrapper}2D Doubling Trick Wrapper for Unknown $T$
and $D$}

\textbf{Initialization:} Set $h=1,w=1,\nu=1$. Choose an algorithm $\Alg\left(\mathcal{P}_{\nu}\right)$, where $\mathcal{P}_{\nu}$ is the set of the algorithm parameters. Initialize $\boldsymbol{p}_{1},\mathcal{P}_{1}$. Let $R\left(T,D,\mathcal{P^{*}}\left(T,D\right)\right)=k_{1}D^{d}+k_{2}T^{c}+k_{3}T^{a}D^{b}$ be an expected regret bound for when $D,T$ are known, that applies to $\Alg\left(\mathcal{P^{*}}\left(T,D\right)\right)$. Divide the $(h,w)$ space into super-epochs as detailed below \eqref{eq:9}.

\textbf{For $t=1,...,T$ do}
\begin{enumerate}
\item Pick an action $\boldsymbol{a}_{t}\in\mathcal{K}$ at random according to the distribution $\boldsymbol{p}_{t}$.
\item Let $\mathcal{\tilde{C}}_{t}$ be the set of delayed costs $l_{s}\left(\boldsymbol{a}_{s}\right)$
received at round $t$ such that $\nu_{s}=\nu_{t}$ (i.e., originated in the current super-epoch). Calculate the number of missing samples at round $t$ by computing the difference between the number of rounds since the beginning of the super-epoch and the number of samples from this super-epoch received so far:
\begin{equation}
m_{t}=\left(t-\underset{\tau\in\left\{ t'\,|\,\nu_{t}=\nu_{t'}\right\} }{\min}\tau+1\right)-\sum_{\tau\in\left\{ t'\,|\,\nu_{t}=\nu_{t'}\right\} }\left|\mathcal{\tilde{C}}_{\tau}\right|.\label{eq:10}
\end{equation}
\item Update the delay index: if $\sum_{\tau=1}^{t}m_{\tau}\geq2^{w}$, then update $w\leftarrow w+1$. 
\item Update the time index: if $t\geq2^{h}$, then update $h\leftarrow h+1$.
\item Start a new super-epoch: if the new $h,w$ indices are outside the current super-epoch, i.e., $\left(h,w\right)\notin\mathcal{E}_{\nu}$,  then start a new super-epoch with parameters $\mathcal{P}_{\nu+1}\leftarrow\mathcal{P^{*}}\left(2^{h_{\nu+1}},2^{w_{\nu+1}}\right)$, where $h_{\nu},w_{\nu}$ are the maximal $h,w$ indices in super-epoch $\mathcal{E}_{\nu}$:
\begin{enumerate}
    \item Initialize the algorithm with these parameters, i.e.,  $\Alg\left(\mathcal{P}_{\nu+1}\right)$.
    \item Increase the super-epoch index $\nu\leftarrow\nu+1$.
\end{enumerate}
\item Using only the samples in $\mathcal{\tilde{C}}_{t}$, update the distribution
$\boldsymbol{p}_{t}$ according to $\Alg\left(\mathcal{P}_{\nu}\right)$.
\end{enumerate}
\textbf{End}\\
\end{algorithm}
The next Lemma proves that our doubling trick tracks $D$, similar in spirit to how the standard doubling trick tracks $T$ up to a factor of 2. It also bounds the largest delay and time indices possible. 
\begin{lem}
\label{lem:doublingStuff} Let $H,W$ be the last time and delay indices, respectively. Let $\mathcal{S}_{w},\mathcal{S}_{h}$ be super-epochs, as defined below \eqref{eq:9}. Let $\mathcal{T}_{w}$ be the set of all rounds in $\mathcal{S}_{w}$ and $\tau_{w}\triangleq\underset{t\in\mathcal{T}_{w}}{\max}\,t$. Let $W_{h}$ be the maximal $w$ such that $\left(h,w\right)\in\mathcal{S}_{h}$. Let $\mathcal{T}_{h}$ be the set of all rounds in $\mathcal{S}_{h}$ and $\tau_{h}\triangleq\underset{t\in\mathcal{T}_{h}}{\max}\,t$. Algorithm \ref{alg:DoublingTrickWrapper} maintains:
\begin{enumerate}
\item For every $w$ and $h$, $\sum_{t\in\mathcal{T}_{w}}\min\left\{ d_{t},\tau_{w}-t+1\right\} \leq2^{w-1}$
and $\sum_{t\in\mathcal{T}_{h}}\min\left\{ d_{t},\tau_{h}-t+1\right\} \leq2^{W_{h}}$.
\item $W\leq\log_{2}\sum_{t=1}^{T}\min\left\{ d_{t},T-t+1\right\} +1$ and
$H\leq\log_{2}\left(T+2\right)-1$.
\end{enumerate}
\end{lem}
\begin{proof}
Let $\mathcal{M}_{\mathcal{S}_{w}}$ be the set of feedback samples
for costs in $\mathcal{S}_{w}$ that are not received within $\mathcal{S}_{w}$.
Every round $t\in\mathcal{T}_{w}$ such that $t\notin\mathcal{M}_{\mathcal{S}_{w}}$
contributes exactly $d_{t}$ to $\sum_{t\in\mathcal{T}_{_{w}}}m_{t}$,
since the $t$-th feedback is missing for $d_{t}$ rounds in $\mathcal{T}_{w}$.
Every round $t\in\mathcal{M}_{\mathcal{S}_{w}}$ contributes $\tau_{w}-t+1\leq d_{t}$
to $\sum_{t\in\mathcal{T}_{_{w}}}m_{t}$ before it stops being counted.
Therefore 
\begin{equation}
\sum_{t\in\mathcal{T}_{w}}\min\left\{ d_{t},\tau_{w}-t+1\right\} \leq\sum_{t\in\mathcal{T}_{w}}m_{t}\underset{\left(a\right)}{\leq}2^{w-1}\label{eq:11}
\end{equation}
where (a) follows since if $\sum_{t\in\mathcal{T}_{w}}m_{t}>2^{w-1}$
then $\sum_{\tau=1}^{t}m_{\tau}\geq2^{w-1}+2^{w-1}=2^{w}$ and the delay index $w$ should have already increased to $w+1$. Applying the same argument
on $\mathcal{T}_{_{h}}$, we obtain that
\begin{equation}
\sum_{t\in\mathcal{T}_{h}}\min\left\{ d_{t},\tau_{h}-t+1\right\} \leq\sum_{t\in\mathcal{T}_{h}}m_{t}\underset{\left(a\right)}{\leq}2^{W_{h}}\label{eq:12}
\end{equation}
where (a) follows since if $\sum_{t\in\mathcal{T}_{h}}m_{t}>2^{W_{h}}$
then $\sum_{\tau=1}^{t}m_{\tau}\geq2^{W_{h}}$ so $w$ must have increased
to $W_{h}+1$. 

For the second part of the lemma, $2^{H+1}-2=\sum_{h=1}^{H}2^{h}\leq T$ so $H\leq\log_{2}\left(T+2\right)-1$, and 
\begin{equation}
\sum_{t=1}^{T}\min\left\{ d_{t},T-t+1\right\} \underset{\left(a\right)}{\geq}\sum_{t=1}^{T}m_{t}\geq2^{W-1}\label{eq:20}
\end{equation}
where (a) uses that every sample is counted in $\sum_{t=1}^{T}m_{t}$
for at most $d_{t}$ or $T-t+1$ rounds. 
\end{proof}
Now we can prove our main result of this section. The first assumption
is merely the regret bound that one can obtain if $T$ and $D$ are
known. The second assumption says that holding the parameters $\mathcal{P}$ fixed, the regret bound is non-decreasing
with $T$ and $D$. As we see in Section
\ref{sec:Bandit-Convex-Optimization} and Section \ref{sec:EXP3-in-Adversarial},
these basic assumptions hold for FKM and EXP3. 

\begin{thm}
\label{2DDoublingTrickTheorem} Let $\left\{ d_{t}\right\} $ be a delay sequence such that the cost
from round $t$ is received at round $t+d_{t}$. Let $T$ be the time horizon and define $D=\sum_{t=1}^{T}\min\left\{ d_{t},T-t+1\right\}$.
Let $R\left(T,D,\mathcal{P}\left(T,D\right)\right)$ be an upper bound on the expected regret (Definition \ref{RegretDefinition}) of an online learning algorithm that uses the parameters $\mathcal{P}\left(T,D\right)$. Assume:
\begin{enumerate}
\item There exists a sequence $\mathcal{P^{*}}\left(T,D\right)$ and constants $k_{1},k_{2},k_{3}\geq0$ and $0\leq a,b,c,d\leq1$
such that for all $T,D$:
\begin{equation}
R\left(T,D,\mathcal{P}^{*}\left(T,D\right)\right)\leq k_{1}D^{d}+k_{2}T^{c}+k_{3}T^{a}D^{b}.\label{eq:21}
\end{equation}
\item For a fixed $\mathcal{P}^{*}$, $R\left(T,D,\mathcal{P}^{*}\right)$ is non-decreasing with $T$ and $D$. 
\end{enumerate}
Then if Algorithm \ref{alg:DoublingTrickWrapper} is used to track $\mathcal{P^{*}}\left(T,D\right)$, it achieves a total expected regret of 
\begin{equation}
R\left(T,D\right)=O\left(k_{1}D^{d}+k_{2}T^{c}+k_{3}T^{a}D^{b}\right).\label{eq:23}
\end{equation}
\end{thm}
\begin{proof}
We apply the regret bound for each of the super-epoch types $\mathcal{S}_{h},\mathcal{S}_{w}$ and $\left\{ \left(h,w\right)\right\}$ as follows (defined below \eqref{eq:9} and illustrated in Fig. \ref{fig:2D-Doubling-Trick}):

\textbf{Type 1 ($\mathcal{S}_{h}$)}: Let $W_{h}$ be the largest $w$ such that $\left(h,w\right)\in\mathcal{S}_{h}$.
Let $T_{h}$ be the length of $\mathcal{S}_{h}$ and let $D_{h}\triangleq\sum_{t\in\mathcal{T}_{h}}\min\left\{ d_{t},\tau_{h}-t+1\right\} \leq2^{W_{h}}$,
using Lemma \ref{lem:doublingStuff}. By applying the regret
bound on $\mathcal{S}_{h}$:
\begin{equation}
R_{\mathcal{S}_{h}}\underset{\left(a\right)}{\leq}R\left(T_{h},D_{h},\mathcal{P}^{*}\left(2^{h},2^{W_{h}}\right)\right)\underset{\left(b\right)}{\leq}R\left(2^{h},2^{W_{h}},\mathcal{P}^{*}\left(2^{h},2^{W_{h}}\right)\right)\\\leq k_{1}2^{W_{h}d}+k_{2}2^{hc}+k_{3}2^{ha}2^{W_{h}b}\leq3k_{2}2^{hc}\label{eq:24}
\end{equation}
where (a) follows since Algorithm \ref{alg:DoublingTrickWrapper} uses $\mathcal{P}^{*}\left(2^{h},2^{W_{h}}\right)$, and (b) from condition 2 of the Theorem. 

\textbf{Type 2 ($\mathcal{S}_{w}$)}: Let $H_{w}$ be the largest $h$ such that $\left(h,w\right)\in\mathcal{S}_{w}$. Let $T_{w}$
be the length of $\mathcal{S}_{w}$ and let $D_{w}\triangleq\sum_{t\in\mathcal{T}_{w}}\min\left\{ d_{t},\tau_{w}-t+1\right\} \leq2^{w-1}$,
using Lemma \ref{lem:doublingStuff}. By applying the regret
bound on $\mathcal{S}_{w}$:
\begin{equation}
R_{\mathcal{S}_{w}}\leq R\left(T_{w},D_{w},\mathcal{P}^{*}\left(2^{H_{w}},2^{w}\right)\right)\leq R\left(2^{H_{w}},2^{w},\mathcal{P}^{*}\left(2^{H_{w}},2^{w}\right)\right)\\\leq k_{1}2^{wd}+k_{2}2^{H_{w}c}+k_{3}2^{H_{w}a}2^{wb}\leq3k_{1}2^{wd}.\label{eq:25}
\end{equation}

\textbf{Type 3 ($\left\{ \left(h,w\right)\right\}$)}: For this case, we must have $2k_{3}2^{ha}2^{wb}\geq k_{1}2^{wd}+k_{2}2^{ch}$, so
\begin{equation}
R_{h,w}\leq R\left(T_{h,w},D_{h,w},\mathcal{P}^{*}\left(2^{h},2^{w}\right)\right)\leq R\left(2^{h},2^{w},\mathcal{P}^{*}\left(2^{h},2^{w}\right)\right)\\\leq k_{1}2^{wd}+k_{2}2^{hc}+k_{3}2^{ha}2^{wb}\leq3k_{3}2^{ha}2^{wb}.\label{eq:26}
\end{equation}
Then the total regret is bounded by
\begin{multline}
\mathbb{E}\left\{ R\left(T\right)\right\} =\sum_{h=1}^{H}R_{\mathcal{S}_{h}}+\sum_{w=1}^{W}R_{\mathcal{S}_{w}}+\sum_{\left(h,w\right)\notin\mathcal{S}_{h}\cup\mathcal{S}_{w}}R_{h,w}
\underset{\left(a\right)}{\leq}\\
3k_{1}\sum_{w=1}^{W}2^{dw}+3k_{2}\sum_{h=1}^{H}2^{ch}+3k_{3}\sum_{h=1}^{H}2^{ah}\sum_{w=1}^{W}2^{bw}\leq
6k_{1}\frac{2^{dW}-1}{2^{d}-1}+6k_{2}\frac{2^{cH}-1}{2^{c}-1}+12k_{3}\frac{2^{aH}-1}{2^{a}-1}\frac{2^{bW}-1}{2^{b}-1}\underset{\left(b\right)}{\leq}\\
6k_{1}\frac{\left(2D\right)^{d}-1}{2^{d}-1}+6k_{2}\frac{\left(T+2\right)^{c}-1}{2^{c}-1}+12k_{3}\frac{\left(T+2\right)^{a}-1}{2^{a}-1}\frac{\left(2D\right)^{b}-1}{2^{b}-1}=
O\left(k_{1}D^{d}+k_{2}T^{c}+k_{3}T^{a}D^{b}\right)\label{eq:31}
\end{multline}
where (a) uses the bounds in \eqref{eq:24},\eqref{eq:25},\eqref{eq:26}
even for epochs that do not occur (that trivially have zero regret), or that end after the last round.
Inequality (b) uses part 2 of Lemma \ref{lem:doublingStuff} to upper
bound $H$ and $W$. We define the expressions before the last equality of \eqref{eq:31} according to their continuous extensions at  $a,b,c,d=0$, which yield logarithmic terms in $T$ or $D$. Note that $\log D \leq 2\log T$.
\end{proof}

\section{The FKM Algorithm for Adversarial Bandit Convex Optimization with Delayed
Feedback \label{sec:Bandit-Convex-Optimization}}

In bandit convex optimization, the action $\boldsymbol{a}_{t}$ is
chosen from a convex and compact set $\mathcal{K\subset\mathbb{R}}^{n}$
with diameter $\left|\mathcal{K}\right|\triangleq\underset{\boldsymbol{x,y}\in\mathcal{K}}{\max}\left\Vert \boldsymbol{x}-\boldsymbol{y}\right\Vert $.
Without the loss of generality, we assume that $\mathcal{K}$ contains
the unit ball centered at the zero vector. The cost functions $l_{t}\left(\boldsymbol{a}_{t}\right)\in\left[0,1\right]$
are convex for all $t$ and Lipschitz continuous with parameter $L$.
The player has no access to the gradient of $l_{t}$, and only receives
the value of $l_{t}\left(\boldsymbol{a}_{t}\right)$ at the beginning of round $t+d_{t}$.
In the FKM algorithm, the player uses the cost value to estimate
the gradient. The idea is to use a perturbation $\boldsymbol{u}_{t}$,
drawn uniformly at random on the $n$-dimensional unit sphere $\mathbb{S}_{1}$.
Then, instead of playing the unperturbed action $\boldsymbol{x}_{t}$,
the player plays $\boldsymbol{a}_{t}=\boldsymbol{x}_{t}+\delta\boldsymbol{u}_{t}$
for a sampling radius $\delta>0$. Let $\mathcal{K}_{\delta}=\left\{ \boldsymbol{x}\,|\,\frac{1}{1-\delta}\boldsymbol{x}\in\mathcal{K}\right\} $.
To ensure that $\boldsymbol{x}_{t}+\delta\boldsymbol{u}_{t}\in\mathcal{K}$,
we maintain $\boldsymbol{x}_{t}\in\mathcal{K}_{\delta}$ by projecting
into $\mathcal{K}_{\delta}$ after each gradient step. Since the cost functions are Lipschitz continuous, this projection creates a bias that decreases with $\delta$.

Define the
following filtration 
\begin{equation}
\mathcal{F}_{t}=\sigma\left(\left\{ \boldsymbol{x}_{\tau}\,|\,\tau\leq t\right\} \right)\label{eq:32}
\end{equation}
which is generated from all the past unperturbed actions. With a slight abuse of notation, we use $\mathcal{F}_{s_{-}}$ to denote the filtration induced from all $\boldsymbol{x}_{\tau}$ for $\tau \leq t$ and all $\boldsymbol{x}_{q_{-}}$ for rounds $q\in\mathcal{S}_{t}$ that the algorithm used their feedback before using the feedback from round $s$, but including $\boldsymbol{x}_{s_{-}}$. 

The purpose of the action perturbation is to allow for an estimator
for the gradient at \textbf{$\boldsymbol{a}_{t}$} with a bias that decreases with $\delta$. This bias adds up to the bias that results from projecting into  $\mathcal{K}_{\delta}$. On the other hand, the variance of the estimator increases with $\delta$, introducing a bias-variance trade-off.
\begin{lem}[{\citet[Lemma 2.1]{flaxman2005online}}]
\label{SPSA}Let $\delta>0$ and define $\hat{l}\left(\boldsymbol{x}\right)\triangleq\mathbb{E}^{\boldsymbol{u}\in\mathbb{S}_{1}}\left\{ l\left(\boldsymbol{x}+\delta\boldsymbol{u}\right)\right\}$
where $\mathbb{S}_{1}$ is the unit sphere. Let $\boldsymbol{g}=\frac{n}{\delta}l\left(\boldsymbol{x}+\delta\boldsymbol{u}\right)\boldsymbol{u}$.
Then $\mathbb{E}^{\boldsymbol{u}\in\mathbb{S}_{1}}\left\{ \boldsymbol{g}\right\} =\nabla\hat{l}\left(\boldsymbol{x}\right)$.
\end{lem}

The next Lemma is the main result of this section, used to prove both
Theorem \ref{SingleAgentConvex} and Lemma \ref{lem:NoErgodicFKM}. 
\begin{lem}[Weighted-Regret Bound for FKM]
\label{WeightedRegretConvex}Let $\left\{ \eta_{t}\right\} $ be a
non-increasing step-size sequence. Let $\delta$ be the sampling radius.
For every $t$, let $l_{t}:\mathcal{K}\rightarrow\left[0,1\right]$
be a convex cost function that is Lipschitz continuous with parameter
$L$. Let $\boldsymbol{a}^{*}=\underset{\boldsymbol{a}\in\mathcal{K}}{\arg\min}\sum_{t=1}^{T}\eta_{t}l_{t}\left(\boldsymbol{a}\right)$.
Let $\left\{ d_{t}\right\} $ be a delay sequence such that the cost
from round $t$ is received at round $t+d_{t}$. Define the set $\mathcal{M}=\left\{ t\,|\,t+d_{t}>T,\,t\in\left[1,T\right]\right\}$
of all samples that are not received before round $T$. Then using FKM (Algorithm \ref{alg:FKM}) guarantees:
\begin{enumerate}
\item For an oblivious adversary:
\begin{equation}
\sum_{t=1}^{T}\eta_{t}\left(\mathbb{E}\left\{ l_{t}\left(\boldsymbol{a}_{t}\right)\right\} -l_{t}\left(\boldsymbol{a}^{*}\right)\right)\leq\\\sum_{t\in\mathcal{M}}\eta_{t}+\frac{\left|\mathcal{K}\right|^{2}}{2}+\left(3+\left|\mathcal{K}\right|\right)L\delta\sum_{t\notin\mathcal{M}}\eta_{t}+\frac{1}{2}\frac{n^{2}}{\delta^{2}}\sum_{t\notin\mathcal{M}}\eta_{t}^{2}+2L\frac{n}{\delta}\sum_{t\notin\mathcal{M}}\eta_{t}^{2}d_{t}.\label{eq:34}
\end{equation}
\item For an adaptive adversary: 
\begin{multline}
\sum_{t=1}^{T}\eta_{t}\mathbb{E}\left\{ l_{t}\left(\boldsymbol{a}_{t}\right)-l_{t}\left(\boldsymbol{a}^{*}\right)\right\} \leq\\\sum_{t\in\mathcal{M}}\eta_{t}+\frac{\left|\mathcal{K}\right|^{2}}{2}+\left(3+\left|\mathcal{K}\right|\right)L\delta\sum_{t\notin\mathcal{M}}\eta_{t}+\left|\mathcal{K}\right|\sqrt{2\sum_{t\notin\mathcal{M}}\eta_{t}^{2}\left(\frac{n^{2}}{\delta^{2}}+L^{2}\right)}+\frac{1}{2}\frac{n^{2}}{\delta^{2}}\sum_{t\notin\mathcal{M}}\eta_{t}^{2}\left(1+4d_{t}\right).\label{eq:35}
\end{multline}
\end{enumerate}
\end{lem}
\begin{proof}
See Appendix.
\end{proof}

\begin{algorithm}[t]
\caption{\label{alg:FKM}FKM with delays}

\textbf{Initialization: }Let $\left\{ \eta_{t}\right\} $ be a positive
non-increasing sequence. Let $\delta<1$ . Set $\boldsymbol{x}_{1}=0$. 

\textbf{For $t=1,...,T$ do}
\begin{enumerate}
\item Draw $\boldsymbol{u}_{t}\in\mathbb{S}_{1}$ uniformly at random, where $\mathbb{S}_{1}$ is the $n$-dimensional unit sphere.
\item Play $\boldsymbol{a}_{t}=\boldsymbol{x}_{t}+\delta\boldsymbol{u}_{t}$.
\item Obtain a set of delayed costs $l_{s}\left(\boldsymbol{a}_{s}\right)$
for all $s\in\mathcal{S}_{t}$ and compute $\boldsymbol{g}_{s}=\frac{n}{\delta}l_{s}\left(\boldsymbol{a}_{s}\right)\boldsymbol{u}_{s}$ for each.
\item Let $s_{\min}=\underset{s\in\mathcal{S}_{t}}{\min}s$ and $s_{\max}=\underset{s\in\mathcal{S}_{t}}{\max}s$.
Set $\boldsymbol{x}_{s_{\min}^{-}}=\boldsymbol{x}_{t}$. For every
$s\in\mathcal{S}_{t}$, update
\begin{equation}
\boldsymbol{x}_{s_{+}}=\prod_{\mathcal{K}_{\delta}}\left(\boldsymbol{x}_{s_{-}}-\eta_{s}\boldsymbol{g}_{s}\right)\label{eq:33}
\end{equation}
~~~~~~~where $\mathcal{K}_{\delta}=\left\{ \boldsymbol{x}\,|\,\frac{1}{1-\delta}\boldsymbol{x}\in\mathcal{K}\right\}$,
and then set \textbf{$\boldsymbol{x}_{t+1}=\boldsymbol{x}_{s_{\max}^{+}}$}.
\end{enumerate}
\textbf{End}
\end{algorithm}

The following theorem establishes the expected regret bound for FKM
with delays. It is proved by optimizing over a constant step-size
$\eta$ and sampling radius $\delta$ in Lemma \ref{WeightedRegretConvex}.
\begin{thm}
\label{SingleAgentConvex}Let $\eta>0$ and $0<\delta<1$. For every
$t$, let $l_{t}:\mathcal{K}\rightarrow\left[0,1\right]$ be a convex
cost function that is Lipschitz continuous with parameter $L$. Let
$\boldsymbol{a}^{*}=\underset{\boldsymbol{a}\in\mathcal{K}}{\arg\min}\sum_{t=1}^{T}l_{t}\left(\boldsymbol{a}\right)$.
Let $\left\{ d_{t}\right\} $ be a delay sequence such that the cost
from round $t$ is received at round $t+d_{t}$. Define the set ${\mathcal{M}=\left\{ t\,|\,t+d_{t}>T,\,t\in\left[1,T\right]\right\}}$ of all samples that are not received before round $T$. Then the expected regret of FKM (Algorithm \ref{alg:FKM}) against an oblivious adversary satisfies 
\begin{multline}
\mathbb{E}\left\{ R\left(T\right)\right\} =\sum_{t=1}^{T}\mathbb{E}\left\{ l_{t}\left(\boldsymbol{a}_{t}\right)-l_{t}\left(\boldsymbol{a}^{*}\right)\right\} \leq\\\left|\mathcal{M}\right|+\left(\left(3+\left|\mathcal{K}\right|\right)\delta L+\frac{1}{2}\eta\frac{n^{2}}{\delta^{2}}\right)\left(T-\left|\mathcal{M}\right|\right)+\frac{\left|\mathcal{K}\right|^{2}}{2\eta}+2Ln\frac{\eta}{\delta}\sum_{t\notin\mathcal{M}}d_{t}.\label{eq:36}
\end{multline}
Furthermore, for
\begin{equation}
\eta=\left|\mathcal{K}\right|\min\left\{ \frac{1}{n}T^{-\frac{3}{4}},\frac{1}{\sqrt{n}}T^{-\frac{1}{3}}\left(\sum_{t\notin\mathcal{M}}d_{t}\right)^{-\frac{1}{3}}\right\} \text{ and }\delta=\max\left\{ T^{-\frac{1}{4}},T^{-\frac{2}{3}}\left(\sum_{t\notin\mathcal{M}}d_{t}\right)^{\frac{1}{3}}\right\} \label{eq:37}
\end{equation}
we obtain
\begin{equation}
\mathbb{E}\left\{ R\left(T\right)\right\} =O\left(nT^{\frac{3}{4}}+\sqrt{n}\left(\sum_{t\notin\mathcal{M}}d_{t}\right)^{\frac{1}{3}}T^{\frac{1}{3}}+\left|\mathcal{M}\right|\right).\label{eq:38}
\end{equation}
\end{thm}
\begin{proof}
First note that if $\left|\mathcal{M}\right|=\Theta\left(T\right)$ then $\mathbb{E}\left\{ R\left(T\right)\right\} =O\left(T\right)$ and otherwise $T-\left|\mathcal{M}\right|=\Theta\left(T\right)$. To obtain \eqref{eq:36}, substitute $\eta_{t}=\eta$ in Lemma
\ref{WeightedRegretConvex} for an oblivious adversary and divide
both sides by $\eta$. We have $\Theta\left(\delta T\right)=\Theta\left(\frac{\eta}{\delta^{2}}T\right)=\Theta\left(\frac{1}{\eta}\right)$
for $\eta=\frac{\left|\mathcal{K}\right|}{n}T^{-\frac{3}{4}}$ and
$\delta=T^{-\frac{1}{4}}$, hence this choice of parameters minimizes the order of magnitude of the $T$ dependence of the following expression:
\begin{equation}
\underset{\delta,\eta}{\min}\left(\left(3+\left|\mathcal{K}\right|\right)\delta L\left(T-\left|\mathcal{M}\right|\right)+\frac{1}{2}\eta\frac{n^{2}}{\delta^{2}}\left(T-\left|\mathcal{M}\right|\right)+\frac{\left|\mathcal{K}\right|^{2}}{2\eta}\right)=O\left(nT^{\frac{3}{4}}\right).\label{eq:39}
\end{equation}
Since $\Theta\left(\delta T\right)=\Theta\left(\frac{\eta}{\delta}\sum_{t\notin\mathcal{M}}d_{t}\right)=\Theta\left(\frac{1}{\eta}\right)$
for $\eta=\frac{1}{\sqrt{n}}\left(\frac{1}{T\sum_{t\notin\mathcal{M}}d_{t}}\right)^{\frac{1}{3}}$
and $\delta=T^{-\frac{2}{3}}\left(\sum_{t\notin\mathcal{M}}d_{t}\right)^{\frac{1}{3}}$,
then this choice of parameters minimizes the order of magnitude of the $T$ dependence of the following expression:
\begin{equation}
\underset{\delta,\eta}{\min}\left(\left(3+\left|\mathcal{K}\right|\right)\delta L\left(T-\left|\mathcal{M}\right|\right)+\frac{\left|\mathcal{K}\right|^{2}}{2\eta}+2L\eta\frac{n}{\delta}\sum_{t\notin\mathcal{M}}d_{t}\right)=O\left(\sqrt{n}\left(\sum_{t\notin\mathcal{M}}d_{t}\right)^{\frac{1}{3}}T^{\frac{1}{3}}\right).\label{eq:40}
\end{equation}
Therefore \eqref{eq:36} cannot have a  better $T$ dependence than in \eqref{eq:38}
for any $\eta,\delta$. For the choice in \eqref{eq:37} we have
\begin{equation}
n^{2}\frac{\eta}{\delta^{2}}T=\frac{\left|\mathcal{K}\right|\min\left\{ nT^{\frac{1}{4}},n^{\frac{3}{2}}T^{\frac{2}{3}}\left(\sum_{t\notin\mathcal{M}}d_{t}\right)^{-\frac{1}{3}}\right\} }{\max\left\{ T^{-\frac{1}{2}},T^{-\frac{4}{3}}\left(\sum_{t\notin\mathcal{M}}d_{t}\right)^{\frac{2}{3}}\right\} }\underset{\left(a\right)}{\leq}O\left(nT^{\frac{3}{4}}\right)\label{eq:41}
\end{equation}
where (a) follows since $\frac{\min\left\{ a,b\right\} }{\max\left\{ c,d\right\} }\leq\frac{a}{c}$.
For the choice in \eqref{eq:37} we also have 
\begin{equation}
\frac{\eta}{\delta}n\sum_{t\notin\mathcal{M}}d_{t}=\frac{\left|\mathcal{K}\right|\min\left\{ T^{-\frac{3}{4}}\sum_{t\notin\mathcal{M}}d_{t},\sqrt{n}T^{-\frac{1}{3}}\left(\sum_{t\notin\mathcal{M}}d_{t}\right)^{\frac{2}{3}}\right\} }{\max\left\{ T^{-\frac{1}{4}},T^{-\frac{2}{3}}\left(\sum_{t\notin\mathcal{M}}d_{t}\right)^{\frac{1}{3}}\right\} }\underset{\left(a\right)}{\leq}O\left(\sqrt{n}\left(\sum_{t\notin\mathcal{M}}d_{t}\right)^{\frac{1}{3}}T^{\frac{1}{3}}\right)\label{eq:42}
\end{equation}
where (a) follows since $\frac{\min\left\{ a,b\right\} }{\max\left\{ c,d\right\} }\leq\frac{b}{d}$.
Therefore, for the choice in \eqref{eq:37} 
\begin{multline}
\sum_{t=1}^{T}\left(\mathbb{E}\left\{ l_{t}\left(\boldsymbol{a}_{t}\right)\right\} -l_{t}\left(\boldsymbol{a}^{*}\right)\right)\leq\left|\mathcal{M}\right|+\left(\left(3+\left|\mathcal{K}\right|\right)L+\frac{\left|\mathcal{K}\right|}{2}\right)\max\left\{ nT^{\frac{3}{4}},\sqrt{n}T^{\frac{1}{3}}\left(\sum_{t\notin\mathcal{M}}d_{t}\right)^{\frac{1}{3}}\right\} \\+O\left(nT^{\frac{3}{4}}\right)+O\left(\sqrt{n}\left(\sum_{t\notin\mathcal{M}}d_{t}\right)^{\frac{1}{3}}T^{\frac{1}{3}}\right)=O\left(nT^{\frac{3}{4}}+\sqrt{n}\left(\sum_{t\notin\mathcal{M}}d_{t}\right)^{\frac{1}{3}}T^{\frac{1}{3}}+\left|\mathcal{M}\right|\right).\label{eq:43}
\end{multline}
\end{proof}
The next Corollary shows that the bound of Theorem \ref{SingleAgentConvex}
is slightly tighter than $O\left(nT^{\frac{3}{4}}+\sqrt{n}D^{\frac{1}{3}}T^{\frac{1}{3}}\right)$,
where one sums $D=\sum_{t=1}^{T}\min\left\{ d_{t},T-t+1\right\} $
instead of $\sum_{t\notin\mathcal{M}}d_{t}$ in our regret bound \eqref{eq:38},
depending on the pattern of the missing samples. 
\begin{cor}
\label{tighterFKM}Choose the fixed $\eta$ and $\delta$ according
to \eqref{eq:37}. For every $t$, let $l_{t}:\mathcal{K}\rightarrow\left[0,1\right]$ be a convex cost function that is Lipschitz continuous with parameter
$L$. Let $\boldsymbol{a}^{*}=\underset{\boldsymbol{a}\in\mathcal{K}}{\arg\min}\sum_{t=1}^{T}l_{t}\left(\boldsymbol{a}\right)$.
Let $\left\{ d_{t}\right\} $ be a delay sequence such that the cost from round $t$ is received at round $t+d_{t}$. Let $D=\sum_{t=1}^{T}\min\left\{ d_{t},T-t+1\right\} $.
Then the expected regret of FKM (Algorithm \ref{alg:FKM}) against an oblivious adversary satisfies 
\begin{equation}
\mathbb{E}\left\{ R\left(T\right)\right\} =\sum_{t=1}^{T}\left(\mathbb{E}\left\{ l_{t}\left(\boldsymbol{a}_{t}\right)\right\} -l_{t}\left(\boldsymbol{a}^{*}\right)\right)=O\left(nT^{\frac{3}{4}}+\sqrt{n}D^{\frac{1}{3}}T^{\frac{1}{3}}\right).\label{eq:44}
\end{equation}
\end{cor}
\begin{proof}
The $m\triangleq\left|\mathcal{M}\right|$ missing samples contribute
at least $\frac{m\left(m+1\right)}{2}\geq\frac{m^{2}}{2}$ to $D=\sum_{t=1}^{T}\min\left\{ d_{t},T-t+1\right\} $.
This follows since the best case is when the feedback of round $T$
is delayed by one and arrives after $T$, the feedback of round $T-1$ now has to be delayed by at least 2 to arrive after $T$ and so on, $m$ times. Therefore 
\begin{equation}
T^{\frac{1}{3}}D^{\frac{1}{3}}\geq\left(T\sum_{t\notin\mathcal{M}}d_{t}+T\frac{m^{2}}{2}\right)^{\frac{1}{3}}\underset{\left(a\right)}{\geq}\frac{1}{2}\left(2T\sum_{t\notin\mathcal{M}}d_{t}\right)^{\frac{1}{3}}+\frac{1}{2}\left(Tm^{2}\right)^{\frac{1}{3}}\geq\frac{T^{\frac{1}{3}}}{2^{\frac{2}{3}}}\left(\sum_{t\notin\mathcal{M}}d_{t}\right)^{\frac{1}{3}}+\frac{m}{2}\label{eq:46}
\end{equation}
where (a) follows from the concavity of $f\left(x\right)=x^{\frac{1}{3}}$.
We conclude that the regret in \eqref{eq:44} is greater than that in \eqref{eq:38}.

\end{proof}

\subsection{Agnostic FKM \label{subsec:DoublingTrickDelays}}

If the horizon $T$ and sum of delays $D$ are unknown, then we can
apply Algorithm \ref{alg:DoublingTrickWrapper} to wrap FKM. The next
Theorem is an immediate application of Theorem \ref{2DDoublingTrickTheorem}
on the FKM regret bound for this agnostic case. The resulting bound
retains the same order of magnitude as the bound of Corollary \ref{tighterFKM}
even though $D$ and $T$ are unknown. The only difference with the
bound of Theorem \ref{SingleAgentConvex} arises because the doubling trick
discards samples that cross super-epochs. Hence, the bound below uses
$D=\sum_{t=1}^{T}\min\left\{ d_{t},T-t+1\right\} $ instead of $\sum_{t\notin\mathcal{M}}d_{t}$
and $\left|\mathcal{M}\right|$.
\begin{thm}
\label{AdaptiveFKMBound}For every $t$, let $l_{t}:\mathcal{K}\rightarrow\left[0,1\right]$
be a convex cost function that is Lipschitz continuous with parameter
$L$. Let $\left\{ d_{t}\right\} $ be a delay sequence such that
the cost from round $t$ is received at round $t+d_{t}$. Let $D=\sum_{t=1}^{T}\min\left\{ d_{t},T-t+1\right\} $.
If the player uses Algorithm \ref{alg:DoublingTrickWrapper} to wrap
FKM (Algorithm \ref{alg:FKM}) such that in the $(h,w)$ epoch  $\eta_{h,w}=\left|\mathcal{K}\right|\min\left\{ \frac{1}{n}2^{-\frac{3h}{4}},\frac{1}{\sqrt{n}}2^{-\frac{h+w}{3}}\right\} $
and $\delta_{h,w}=\delta_{0}\max\left\{ 2^{-\frac{h}{4}},2^{\frac{w-2h}{3}}\right\} $
for $\delta_{0}<1$, then the expected regret of FKM (Algorithm \ref{alg:FKM}) against an oblivious adversary satisfies 
\begin{equation}
\mathbb{E}\left\{ R\left(T\right)\right\} =O\left(nT^{\frac{3}{4}}+\sqrt{n}D^{\frac{1}{3}}T^{\frac{1}{3}}\right).\label{eq:47}
\end{equation}
\end{thm}
\begin{proof}
Consider the regret bound in \eqref{eq:36} after bounding $\left|\mathcal{M}\right|\leq\sqrt{2D}$ (as in Corollary \ref{tighterFKM}), $T-\left|\mathcal{M}\right|\leq T$ and $\sum_{t\notin\mathcal{M}}d_{t}\leq D$. For the choice in \eqref{eq:37}, this bound takes the form $O(nT^{\frac{3}{4}}+\sqrt{n}T^{\frac{1}{3}}D^{\frac{1}{3}}+D^{\frac{1}{2}})$ (where $T^{\frac{1}{3}}D^{\frac{1}{3}}\geq D^{\frac{1}{2}})$ so it matches Assumption 1 in Theorem \ref{2DDoublingTrickTheorem} with $a=b=\frac{1}{3}, c=\frac{3}{4}$ and $d=\frac{1}{2}$. This bound also satisfies Assumption 2 since it is increasing with $T$ and $D$ for any $\eta,\delta$.
\end{proof}

\subsection{No Weighted-Regret Property for FKM}

In this subsection, we provide conditions for FKM to have no weighted-regret with respect to the delay sequence and its step-size sequence as the weight sequence of Section \ref{sec:Games}. As discussed in Section
\ref{sec:Games}, $\sum_{t=1}^{\infty}\eta_{t}=\infty$ is necessary for the no weighted-regret property to be non-trivial. All other conditions of Lemma \ref{lem:NoErgodicFKM} are as tight as the bound of Lemma \ref{WeightedRegretConvex}. 

\begin{lem}
\label{lem:NoErgodicFKM}FKM with a non-increasing and positive step-size sequence $\left\{ \eta_{t}\right\} $
and sampling radius $\delta_{T}$ has no weighted-regret with respect to the sequence of delays
$\left\{ d_{t}\right\}$ and $\{\eta_t\}$ as the weight sequence, if the following three conditions hold:
\begin{enumerate}
\item $\sum_{t=1}^{\infty}\eta_{t}=\infty$.
\item $\underset{t\rightarrow\infty}{\lim}\eta_{t}d_{t}<\infty$ and $\sum_{t=1}^{\infty}\eta_{t}^{2}d_{t}<\infty$.
\item $\underset{T\rightarrow\infty}{\lim}\delta_{T}=0$ and $\underset{T\rightarrow\infty}{\lim}\delta_{T}^{2}\sum_{t=1}^{T}\eta_{t}=\infty$.
\end{enumerate}
\end{lem}
\begin{proof}
Let $\mathcal{M}=\left\{ t\,|\,t+d_{t}>T,\,t\in\left[1,T\right]\right\} \text{.}$
Define $t^{*}\left(T\right)=\underset{t\in\mathcal{M}}{\min}t$, and
note that $t^{*}\left(T\right)\rightarrow\infty$ as $T\rightarrow\infty$
since $t+d_{t}\geq t$, and $f\left(t\right)=t$ is increasing. Since
$\eta_{t}$ is non-increasing then 
\begin{equation}
\sum_{t\in\mathcal{M}}\eta_{t}\leq\left|\mathcal{M}\right|\eta_{t^{*}\left(T\right)}\leq\left(T-t^{*}\left(T\right)+1\right)\eta_{t^{*}\left(T\right)}\leq d_{t^{*}\left(T\right)}\eta_{t^{*}\left(T\right)}.\label{eq:49}
\end{equation}
Let $A=\left(3+\left|\mathcal{K}\right|\right)L$. Therefore
\begin{multline}
\underset{T\rightarrow\infty}{\lim}\mathbb{E}\left\{ \frac{\sum_{t=1}^{T}\eta_{t}\left(l_{t}\left(\boldsymbol{a}_{t}\right)-l_{t}\left(\boldsymbol{a}^{*}\right)\right)}{\sum_{t=1}^{T}\eta_{t}}\right\} \\\underset{\left(a\right)}{\leq}\underset{T\rightarrow\infty}{\lim}\frac{d_{t^{*}\left(T\right)}\eta_{t^{*}\left(T\right)}+\frac{\left|\mathcal{K}\right|^{2}}{2}+A\delta_{T}\sum_{t=1}^{T}\eta_{t}+\frac{2n\left|\mathcal{K}\right|}{\delta_{T}}\sqrt{\sum_{t=1}^{T}\eta_{t}^{2}}+\frac{1}{2}\frac{n^{2}}{\delta_{T}^{2}}\sum_{t=1}^{T}\eta_{t}^{2}\left(1+4d_{t}\right)}{\sum_{t=1}^{T}\eta_{t}}\underset{\left(b\right)}{=}0\label{eq:50}
\end{multline}
where (a) is Lemma \ref{WeightedRegretConvex} and \eqref{eq:49}, and (b) follows since
$\underset{t\rightarrow\infty}{\lim}d_{t}\eta_{t}<\infty$, $\sum_{t=1}^{\infty}\eta_{t}^{2}d_{t}<\infty$,
$\sum_{t=1}^{\infty}\eta_{t}=\infty$, $\underset{T\rightarrow\infty}{\lim}\delta_{T}=0$
and $\underset{T\rightarrow\infty}{\lim}\delta_{T}^{2}\sum_{t=1}^{T}\eta_{t}=\infty$. 
\end{proof}
Note that one can choose $\eta_{t}=\frac{1}{t\log t\log\log t\log\log\log t}$ and $\delta_{T}=O\left(\left(\log\log\log\left(\log T\right)\right)^{-\frac{1}{3}}\right)$
to guarantee no weighted-regret for all sequences such that $d_{t}=O\left(t\log t\log\log t\right)$.
This boundary can only be slightly improved by adding $\log\left(\log\left(\ldots\log\left(T\right)\right)\right)$
iteratively in this manner as long as $\sum_{t=1}^{T}\frac{1}{d_{t}}=\infty$. Outside this boundary, we cannot guarantee that FKM has no weighted-regret even if $d_t$ is known. Hence, no knowledge of the individual terms of the sequence $d_{t}$ is required to tune $\eta_{t}$ and $\delta_{T}$ such that FKM has no weighted-regret for all sequences inside this boundary, which can be arbitrarily extended to include all the sequences for which Lemma \ref{lem:NoErgodicFKM} holds. However, if a tighter bound on
the rate of growth of $d_{t}$ is available then one can improve the
convergence rate to the set of CCE by picking a more slowly decaying
$\eta_{t}$ than $\eta_{t}=\frac{1}{t\log t\log\log t\log\log\log t}$. This still would not require knowledge of the individual terms in $d_t$.
 In general, for a given $T$, FKM gives an $\varepsilon$-CCE with
$\varepsilon=O\left(\max\left\{ \frac{1}{\delta_{T}^{2}\sum_{t=1}^{T}\eta_{t}},\delta_{T}\right\} \right)$,
as given in \eqref{eq:50}. 

The following Proposition shows that FKM can have no weighted-regret even when it has linear regret in $T$. As a result, FKM can be used to approximate a CCE in a non-cooperative game with convex cost functions or a NE in a two-player convex-concave zero-sum game despite not having no-regret guarantees. 
\begin{prop}
\label{StillConvergesFKM}There exist a delay sequence $\left\{ d_{t}\right\}$  and Lipschitz continuous convex functions $\left\{ l_{1},...,l_{T}\right\}$  on $\left[0,1\right]$ such that for a large enough $T$
\begin{equation}
    \mathbb{E}\left\{ R\left(T\right)\right\} =\sum_{t=1}^{T}\mathbb{E}\left\{ l_{t}\left(\boldsymbol{a}_{t}\right)-l_{t}\left(\boldsymbol{a}^{*}\right)\right\} \geq\frac{T}{4\left(\left|\mathcal{K}\right|+1\right)^{2}}
\end{equation}
but still the step-sizes $\left\{ \eta_{t}\right\}$ and sampling radius $\delta_{T}$ for Algorithm \ref{alg:FKM} (FKM) can be chosen such that it has no weighted-regret with respect to $\left\{ d_{t}\right\}$ and $\{\eta_{t}\}$ as the weight sequence.
\end{prop}
\begin{proof}
Let $d_{t}=t$ and choose $\delta_{T}=0.1\left(\log\log (T+1)\right)^{-\frac{1}{3}}, \eta_{t}=\frac{1}{t\log\left(t+1\right)}$ for all $t\geq 1$, for which $\sum_{t=1}^{T}\eta_{t}=O\left(\log\log T\right) , \sum_{t=1}^{\infty}d_{t}\eta_{t}^{2}<\infty$ and also $\delta_{T}\rightarrow0$ and $\delta_{T}^{2}\sum_{t=1}^{T}\eta_{t}\rightarrow\infty$ as $T\rightarrow\infty$. Hence, by Lemma \ref{lem:NoErgodicFKM} we obtain that FKM has no weighted-regret with respect to $d_{t}$ and $\eta_{t}$. However, the feedback for the last $\frac{T}{2}$ rounds is never received. Therefore, the unperturbed action $\boldsymbol{x}_{t}$ is $\boldsymbol{x}_{\frac{T}{2}}$ for all $t\geq\frac{T}{2}$. Consider the sequence of costs $l_{t}\left(\boldsymbol{a}\right)=\boldsymbol{0}$ for all $t\leq\frac{T}{2}$ and $l_{t}\left(\boldsymbol{a}\right)=\frac{\left\Vert \boldsymbol{a}-\frac{1}{\sqrt{n}}\boldsymbol{1}\right\Vert ^{2}}{\left(\left|\mathcal{K}\right|+1\right)^{2}}$ for all $t>\frac{T}{2}$ where $\boldsymbol{1}\in\mathbb{R}^{n}$ is a vector of ones. Starting from $\boldsymbol{x}_{1}=\boldsymbol{0}$ and computing $\boldsymbol{g}_{t}=\boldsymbol{0}$ for all $t\leq\frac{T}{2}$ we have $\boldsymbol{x}_{\frac{T}{2}}=\boldsymbol{0}$. Then, from the Lipschitz continuity of $l_{t}$ we obtain for all $t>\frac{T}{2}$, for large enough $T$,
\begin{equation}
    \mathbb{E}\left\{ l_{t}\left(\boldsymbol{a}_{t}\right)\right\} =\mathbb{E}\left\{ l_{t}\left(\boldsymbol{x}_{\frac{T}{2}}+\delta_{T}\boldsymbol{u}_{t}\right)\right\} \geq\mathbb{E}\left\{ l_{t}\left(\boldsymbol{0}\right)\right\} -\delta_{T}L=\frac{1}{\left(\left|\mathcal{K}\right|+1\right)^{2}}-\delta_{T}L\geq\frac{1}{2\left(\left|\mathcal{K}\right|+1\right)^{2}}
\end{equation}
which means that this sequence yields an expected regret of at least $\frac{T}{4\left(\left|\mathcal{K}\right|+1\right)^{2}}$. 
\end{proof}

\section{The EXP3 Algorithm for Adversarial Multi-Armed Bandits with Delayed
Feedback \label{sec:EXP3-in-Adversarial}}

Consider a player that at each round $t$ picks one out of $K$
arms. Let $a_{t}$ be the arm the player chooses at round $t$. The cost at round $t$ of playing arm $i$ is $l_{t}^{\left(i\right)}\in\left[0,1\right]$,
and let \textbf{$\boldsymbol{l}_{t}=\left(l_{t}^{\left(1\right)},...,l_{t}^{\left(K\right)}\right)$}
be the cost vector. At round $t$, the EXP3 algorithm, detailed in Algorithm
\ref{alg:EXP3}, chooses an arm at random using a distribution that depends on the history of the game. The variant when $\gamma_{t}\neq0$, as we use against an adaptive adversary, is known as EXP3-IX (see \cite{neu2015explore}). We denote the vector of probabilities
of the player for choosing arms at round $t$ by $\boldsymbol{p}_{t}\in\Delta^{K}$,
where $\Delta^{K}$ denotes the $K$-simplex. This is also known as
the mixed action of the player. We also define the following filtration
\begin{equation}
\mathcal{F}_{t}=\sigma\left(\left\{ {a}_{s}\,|\,s+d_{s}\leq t\right\} \cup\left\{ \boldsymbol{l}_{s}\,|\,s\leq t\right\} \right)\label{eq:53}
\end{equation}
which is generated from all the actions for which the feedback was received up to round $t$ and all cost functions up to round $t$. Note that the mixed action
$\boldsymbol{p}_{t}$ is a $\mathcal{F}_{t}$-measurable random variable
since $\mathcal{F}_{t}$ includes everything that could affect the algorithm up to round $t$. With
a slight abuse of notation, we use $\mathcal{F}_{s_{-}}$ to denote the filtration induced from all actions for which the feedback has been received and used up to step $s_{-}$ and all the cost functions up to round $t$.

The next Lemma is the main result of this section, used to prove both
Theorem \ref{SingleAgent} and Lemma \ref{lem:NoErgodicEXP3}. 
\begin{lem}[Weighted-Regret Bound for EXP3]
\label{WeightedRegret}Let $\left\{ \eta_{t}\right\} $ be a non-increasing
step-size sequence such that $\eta_{t}\leq\frac{1}{2}e^{-2}$
for all $t$. Let $\left\{ l_{t}^{\left(i\right)}\right\} $
be a cost sequence such that $l_{t}^{\left(i\right)}\in\left[0,1\right]$
for every $t,i$. Let $\left\{ d_{t}\right\} $ be a delay sequence
such that the cost from round $t$ is received at round $t+d_{t}$.
Define the set $\mathcal{\mathcal{M^{*}}}$ of all samples that are
not received before round $T$ or that were delayed by $d_{t}\geq\frac{1}{e^{2}\eta_{t}}-1$.
Then using EXP3 (Algorithm \ref{alg:EXP3}) guarantees:
\begin{enumerate}
\item With an oblivious adversary and $\gamma_{t}=0$ for all $t$:
\begin{equation}
\mathbb{E}\left\{ \sum_{t=1}^{T}\eta_{t}l_{t}^{\left(a_{t}\right)}-\underset{i}{\min}\sum_{t=1}^{T}\eta_{t}l_{t}^{\left(i\right)}\right\} \leq\log  K+\frac{e^{2}}{2}K\sum_{t=1}^{T}\eta_{t}^{2}+4\sum_{t\notin\mathcal{M^{*}}}\eta_{t}^{2}d_{t}+\sum_{t\in\mathcal{\mathcal{M^{*}}}}\eta_{t}.\label{eq:56}
\end{equation}
\item With an adaptive adversary and $\gamma_{t}=\eta_{t}$ for all $t$:
\begin{equation}
\mathbb{E}\left\{ \sum_{t=1}^{T}\eta_{t}l_{t}^{\left(a_{t}\right)}-\underset{i}{\min}\sum_{t=1}^{T}\eta_{t}l_{t}^{\left(i\right)}\right\} \leq2+2\log  K+(1+\frac{e^{2}}{2})K\sum_{t=1}^{T}\eta_{t}^{2}+4e^{2}K\sum_{t\notin\mathcal{M^{*}}}\eta_{t}^{2}d_{t}+\sum_{t\in\mathcal{\mathcal{M^{*}}}}\eta_{t}.\label{eq:57}
\end{equation}
\end{enumerate}
\end{lem}
\begin{proof}
See Appendix.
\end{proof}

\begin{algorithm}[t]
\caption{\label{alg:EXP3}EXP3 with delays}

\textbf{Initialization: }Let $\left\{ \eta_{t}\right\} $ and $\left\{ \gamma_{t}\right\} $
be non-negative non-increasing sequences such that $\eta_{1}\leq\frac{e^{-2}}{2}$
and set $\tilde{L}_{1}^{\left(i\right)}=0$ and $p_{1}^{\left(i\right)}=\frac{1}{K}$
for $i=1,...,K$. 

\textbf{For $t=1,...,T$ do}
\begin{enumerate}
\item Choose an arm $a_{t}$ at random according to the distribution $\boldsymbol{p}_{t}$.
\item Collect in $\mathcal{S}_{t}$ all the rounds $s$ for which $l_{s}^{\left(a_{s}\right)}$ arrived at round $t$ after a delay of $d_{s}\leq\frac{1}{e^{2}\eta_{s}}-1$.
\item Set $\tilde{L}_{t}^{\left(i\right)}=\tilde{L}_{t-1}^{\left(i\right)}$ for all $i$. Update the weights of arm $a_{s}$ for all $s\in\mathcal{S}_{t}$ using
\begin{equation}
\tilde{L}_{t}^{\left(a_{s}\right)}=\tilde{L}_{t}^{\left(a_{s}\right)}+\eta_{s}\frac{l_{s}^{\left(a_{s}\right)}}{p_{s}^{\left(a_{s}\right)}+\gamma_{s}}.\label{eq:54}
\end{equation}
\item Update the mixed action for $i=1,\ldots,K$ using
\begin{equation}
p_{t+1}^{\left(i\right)}=\frac{e^{-\tilde{L}_{t}^{\left(i\right)}}}{\sum_{j=1}^{K}e^{-\tilde{L}_{t}^{\left(j\right)}}}.\label{eq:55}
\end{equation}
\end{enumerate}
\textbf{End}\\
\end{algorithm}

The following theorem establishes the expected regret bound for EXP3
with delays. It is proved by optimizing over a constant step-size
$\eta$ in Lemma \ref{WeightedRegret}.
\begin{thm}
\label{SingleAgent}Let $\left\{ l_{t}^{\left(i\right)}\right\} $ be a cost sequence
such that $l_{t}^{\left(i\right)}\in\left[0,1\right]$ for every $t,i$. Let $\left\{ d_{t}\right\} $ be a delay sequence such that the cost
from round $t$ is received at round $t+d_{t}$. Define the set $\mathcal{M}=\left\{ t\,|\,t+d_{t}>T,\,t\in\left[1,T\right]\right\}$ of all samples that are not received before round $T$. Let us choose the
fixed step-size $\eta=\frac{e^{-2}}{2}\sqrt{\frac{\log  K}{KT+\sum_{t\notin\mathcal{M}}d_{t}}}$.
Then the expected regret of EXP3 (Algorithm \ref{alg:EXP3}) against
an oblivious adversary satisfies
\begin{equation}
\mathbb{E}\left\{ R\left(T\right)\right\} =\mathbb{E}\left\{ \sum_{t=1}^{T}\left\langle \boldsymbol{l}_{t},\boldsymbol{p}_{t}\right\rangle -\underset{i}{\min}\sum_{t=1}^{T}l_{t}^{\left(i\right)}\right\} =O\left(\sqrt{\log  K\left(KT+\sum_{t\notin\mathcal{M}}d_{t}\right)}+\left|\mathcal{M}\right|\right).\label{eq:58}
\end{equation}
\end{thm}
\begin{proof}
We choose $\eta_{t}=\eta$ in \eqref{eq:56} of Lemma \ref{WeightedRegret}
and define $\mathcal{D}=\left\{ t\,|\,d_{t}\geq\frac{1}{e^{2}\eta}-1\text{ and }t+d_{t}\leq T\right\} $
so $\mathcal{M}^{*}=\mathcal{M}\cup\mathcal{\mathcal{D}}$ in the
statement of Lemma \ref{WeightedRegret}. Now divide both sides by
$\eta$:
\begin{equation}
\mathbb{E}\left\{ \sum_{t=1}^{T}l_{t}^{\left(a_{t}\right)}-\underset{i}{\min}\sum_{t=1}^{T}l_{t}^{\left(i\right)}\right\} \leq\frac{\log  K}{\eta}+\frac{e^{2}}{2}\eta KT+4\eta\sum_{t\notin\mathcal{\mathcal{M}^{*}}}d_{t}+\left|\mathcal{M}\right|+\left|\mathcal{D}\right|.\label{eq:59a}
\end{equation}
Then, choosing $\eta=\frac{e^{-2}}{2}\sqrt{\frac{\log  K}{KT+\sum_{t\notin\mathcal{M}}d_{t}}}$
yields \eqref{eq:58}. Note that for this choice $\left|\mathcal{D}\right|\leq\frac{\sum_{t\notin\mathcal{M}}d_{t}}{e^{-2}\eta^{-1}-1}\leq \sqrt{\left(\sum_{t\notin\mathcal{M}}d_{t}\right)\log  K}$,
since $\sum_{t\notin\mathcal{M}}d_{t}\geq\left(\frac{1}{e^{2}\eta}-1\right)\left|\mathcal{D}\right|$
(discarded samples in $\mathcal{D}$ are not missing).
\end{proof}
Similar to the bandit convex optimization case, the bound of Theorem
\ref{SingleAgent} is tighter than $O\left(\sqrt{\log  K\left(KT+D\right)}\right)$
for $D=\sum_{t=1}^{T}\min\left\{ d_{t},T-t+1\right\} $, as the next
Corollary shows. 
\begin{cor}
\label{tighterEXP3}Let $\eta=\frac{e^{-2}}{2}\sqrt{\frac{\log  K}{KT+\sum_{t\notin\mathcal{M}}d_{t}}}$.
Let $\left\{ l_{t}^{\left(i\right)}\right\} $ be a cost sequence
such that $l_{t}^{\left(i\right)}\in\left[0,1\right]$ for every $t,i$.
Let $\left\{ d_{t}\right\} $ be a delay sequence such that the cost
from round $t$ is received at round $t+d_{t}$. Let $D=\sum_{t=1}^{T}\min\left\{ d_{t},T-t+1\right\} $.
Then the expected regret of EXP3 (Algorithm \ref{alg:EXP3}) against
an oblivious adversary satisfies 
\begin{equation}
\mathbb{E}\left\{ R\left(T\right)\right\} =\mathbb{E}\left\{ \sum_{t=1}^{T}\left\langle \boldsymbol{l}_{t},\boldsymbol{p}_{t}\right\rangle -\underset{i}{\min}\sum_{t=1}^{T}l_{t}^{\left(i\right)}\right\} =O\left(\sqrt{\log  K\left(KT+D\right)}\right).\label{eq:60a}
\end{equation}
\end{cor}
\begin{proof}
The $m=\left|\mathcal{M}\right|$ missing samples contribute at least
$\frac{m\left(m+1\right)}{2}$ to $D$ (as in Corollary \ref{tighterFKM}),
so
\begin{multline}
\sqrt{\log  K\left(KT+D\right)}\geq\sqrt{\log  K\left(KT+\sum_{t\notin\mathcal{M}}d_{t}+\frac{m\left(m+1\right)}{2}\right)}\\
\underset{\left(a\right)}{\geq}\frac{1}{2}\sqrt{2\log  K\left(KT+\sum_{t\notin\mathcal{M}}d_{t}\right)}+\frac{1}{2}\sqrt{\log  Km\left(m+1\right)}\geq O\left(\sqrt{\log  K\left(KT+\sum_{t\notin\mathcal{M}}d_{t}\right)}\right)+\frac{\left|\mathcal{M}\right|}{4}\label{eq:61a}
\end{multline}
where (a) follows from the concavity of $f\left(x\right)=\sqrt{x}$.
\end{proof}

\subsection{Agnostic EXP3}
The step-size $\eta=\frac{e^{-2}}{2}\sqrt{\frac{\log  K}{KT+\sum_{t\notin\mathcal{M}}d_{t}}}$
used in Algorithm \ref{alg:EXP3} requires knowing the horizon $T$
and the sum of delays $D$. When these parameters are unknown, we can
apply Algorithm \ref{alg:DoublingTrickWrapper} to wrap EXP3. The next
Theorem is an immediate application Theorem \ref{2DDoublingTrickTheorem}
on the EXP3 regret bound for this agnostic case. The resulting bound
retains the same order of magnitude as the bound of Corollary \ref{tighterEXP3}
even though $D$ and $T$ are unknown. The only difference with the
bound of Theorem \ref{SingleAgent} arises because the doubling trick
discards samples that cross super-epochs. Hence, the bound below uses
$D=\sum_{t=1}^{T}\min\left\{ d_{t},T-t+1\right\} $ instead of $\sum_{t\notin\mathcal{M}}d_{t}$
and $\left|\mathcal{M}\right|$.
\begin{thm}
\label{thm:EXP3-Doubling} Let $\left\{ l_{t}^{\left(i\right)}\right\} $
be a cost sequence such that $l_{t}^{\left(i\right)}\in\left[0,1\right]$
for every $t,i$. Let $\left\{ d_{t}\right\} $ be a delay sequence
such that the cost from round $t$ is received at round $t+d_{t}$.
Let $D=\sum_{t=1}^{T}\min\left\{ d_{t},T-t+1\right\} $. If the player uses Algorithm \ref{alg:DoublingTrickWrapper} to wrap EXP3 (Algorithm \ref{alg:EXP3}) with step
size $\eta_{h,w}=\frac{e^{-2}}{2}\sqrt{\frac{\log  K}{K2^{h}+2^{w}}}$
for epoch $\left(h,w\right)$, then the regret against an oblivious
adversary satisfies 
\begin{equation}
\mathbb{E}\left\{ R\left(T\right)\right\} =\mathbb{E}\left\{ \sum_{t=1}^{T}\left\langle \boldsymbol{l}_{t},\boldsymbol{p}_{t}\right\rangle -\underset{i}{\min}\sum_{t=1}^{T}l_{t}^{\left(i\right)}\right\} =O\left(\sqrt{\log  K\left(KT+D\right)}\right).\label{eq:62a}
\end{equation}
\end{thm}
\begin{proof}
Consider the regret bound in \eqref{eq:59a} after bounding $\left|\mathcal{M}\right|\leq\sqrt{2D}$, (as in the proof of Corollary \ref{tighterEXP3}), $T-\left|\mathcal{M}\right|\leq T$, $\sum_{t\notin\mathcal{M^{*}}}d_{t}\leq D$ and $\left|\mathcal{D}\right|\leq\sqrt{D\log  K}$ (as in the proof of Theorem \ref{SingleAgent}). For $\eta=\frac{e^{-2}}{2}\sqrt{\frac{\log  K}{KT+D}}$, this bound yields that of Corollary \ref{tighterEXP3} which is of the form $\sqrt{TK\log  K}+\sqrt{D\log  K}$, so it matches Assumption 1 in Theorem \ref{2DDoublingTrickTheorem} with $a=b=0$ and $c=d=\frac{1}{2}$. This bound also satisfies Assumption 2 since it is increasing with $T$ and $D$ for any $\eta$.
\end{proof}

\subsection{No Weighted-Regret Property for EXP3}

In this subsection, we provide conditions for EXP3 to have no weighted-regret with respect to the delay sequence and its step-size sequence as the weight sequence of Section \ref{sec:Games}. As discussed in Section
\ref{sec:Games}, $\sum_{t=1}^{\infty}\eta_{t}=\infty$, is necessary for the no weighted-regret property to be non-trivial. All other conditions of Lemma \ref{lem:NoErgodicEXP3} are as tight as the bound of Lemma \ref{WeightedRegret}. 

\begin{lem}
\label{lem:NoErgodicEXP3}EXP3 with a non-increasing and positive step-size sequence $\left\{ \eta_{t}\right\}$ such that $\eta_{t}\leq\frac{1}{2}e^{-2}$ for all $t$ has no weighted-regret with respect to the sequence of delays
$\left\{ d_{t}\right\}$ and $\{\eta_t\}$ as the weight sequence, if the following two conditions hold:
\begin{enumerate}
\item $\sum_{t=1}^{\infty}\eta_{t}=\infty$.
\item $\underset{t\rightarrow\infty}{\lim}\eta_{t}d_{t}<\infty$ and $\sum_{t=1}^{\infty}\eta_{t}^{2}d_{t}<\infty$. 
\end{enumerate}
\end{lem}
\begin{proof}
Define the set of missing samples $\mathcal{M}=\left\{ t\,|\,t+d_{t}>T\right\}$
and the set of discarded samples $\mathcal{\mathcal{D}}_{T}=\left\{ t\,|d_{t}\eta_{t}>e^{-2}-\eta_{t}\right\} $.
Define $t^{*}\left(T\right)=\underset{t\in\mathcal{M}}{\min}t$, and
note that $t^{*}\left(T\right)\rightarrow\infty$ as $T\rightarrow\infty$
since $t+d_{t}\geq t$, and $f\left(t\right)=t$ is increasing. Since
$\eta_{t}$ is non-increasing then 
\begin{equation}
\sum_{t\in\mathcal{M}}\eta_{t}\leq\left|\mathcal{M}\right|\eta_{t^{*}\left(T\right)}\leq\left(T-t^{*}\left(T\right)+1\right)\eta_{t^{*}\left(T\right)}\leq d_{t^{*}\left(T\right)}\eta_{t^{*}\left(T\right)}.\label{eq:63b}
\end{equation}
Given $\sum_{t=1}^{\infty}\eta_{t}^{2}d_{t}<\infty$ and $\sum_{t=1}^{\infty}\eta_{t}=\infty$
we must have $\underset{t\rightarrow\infty}{\lim}\eta_{t}d_{t}=0$
if this limit exists, so $\underset{T\rightarrow\infty}{\lim}\left|\mathcal{D}_{T}\right|<\infty$.
Therefore for the optimal arm $i^{*}$
\begin{equation}
\underset{T\rightarrow\infty}{\lim}\frac{\mathbb{E}\left\{ \sum_{t=1}^{T}\eta_{t}\left(l_{t}^{\left(a_{t}\right)}-l_{t}^{\left(i^{*}\right)}\right)\right\} }{\sum_{t=1}^{T}\eta_{t}}\underset{\left(a\right)}{\leq}\underset{T\rightarrow\infty}{\lim}\frac{\eta_{1}\left|\mathcal{D}_{T}\right|+d_{t^{*}\left(T\right)}\eta_{t^{*}\left(T\right)}+4\log  K+5K\sum_{t=1}^{T}\eta_{t}^{2}\left(1+6d_{t}\right)}{\sum_{t=1}^{T}\eta_{t}}\underset{\left(b\right)}{=}0\label{eq:55a}
\end{equation}
where (a) is Lemma \ref{WeightedRegret} and \eqref{eq:63b}, and (b) uses $d_{t}\eta_{t}\rightarrow0$
as $t\rightarrow\infty$, $\sum_{t=1}^{\infty}\eta_{t}=\infty$ and
$\sum_{t=1}^{\infty}d_{t}\eta_{t}^{2}<\infty$. 
\end{proof}

Note that one can choose $\eta_{t}=\frac{1}{t\log t\log\log t\log\log\log t}$
to guarantee no weighted-regret for all sequences such that $d_{t}=\left(t\log t\log\log t\right)$.
This boundary can only be slightly improved by adding $\log\left(\log\left(\ldots\log\left(T\right)\right)\right)$
iteratively in this manner as long as $\sum_{t=1}^{\infty}\frac{1}{d_{t}}=\infty$.
Outside this boundary, we cannot guarantee that EXP3 has no weighted-regret even if $d_t$ is known. Hence, no knowledge of the individual terms of the sequence $d_{t}$ is required to tune $\eta_{t}$ such that EXP3 has no weighted-regret for all sequences inside this boundary, which can be arbitrarily extended to include all the sequences for which Lemma \ref{lem:NoErgodicEXP3} holds. However, if a tighter bound on
the rate of growth of $d_{t}$ is available then one can improve the
convergence rate to the set of CCE by picking a more slowly decaying
$\eta_{t}$ than $\eta_{t}=\frac{1}{t\log t\log\log t\log\log\log t}$. This still would not require knowledge of the individual terms in $d_t$.
 In general, for a given $T$, EXP3 gives an $\varepsilon$-CCE with
$\varepsilon=O\left(\frac{1}{\sum_{t=1}^{T}\eta_{t}}\right)$, as
given in \eqref{eq:55a}. 

The following Proposition shows that EXP3 can have no weighted-regret even when it has linear regret in $T$. As a result, EXP3 can be used to approximate a CCE in a discrete non-cooperative game or a NE in a discrete two-player zero-sum game despite not having no-regret guarantees.
\begin{prop}
\label{StillConvergesEXP3}There exist a delay sequence $\left\{ d_{t}\right\}$ and a cost sequence $\left\{ l_{t}^{\left(1\right)},...,l_{t}^{\left(K\right)}\right\} _{t}$ with $0\leq l_{t}^{\left(i\right)}\leq1$ for all $t$ and $i$, such that 
\begin{equation}
    \mathbb{E}\left\{ R\left(T\right)\right\} =\mathbb{E}\left\{ \sum_{t=1}^{T}\left\langle \boldsymbol{l}_{t},\boldsymbol{p}_{t}\right\rangle -\underset{i}{\min}\sum_{t=1}^{T}l_{t}^{\left(i\right)}\right\} \geq\left(1-\frac{1}{K}\right)\frac{T}{2}
\end{equation}
but still the step-sizes $\left\{ \eta_{t}\right\}$  for Algorithm \ref{alg:EXP3} (EXP3) can be chosen such that it has no weighted-regret with respect to $\left\{ d_{t}\right\}$ and $\{\eta_{t}\}$ as the weight sequence.
\end{prop}
\begin{proof}
Let $d_{t}=t$ and $\eta_{t}=\frac{1}{t\log\left(t+1\right)}$
for all $t$, for which $d_{t}\eta_{t}^{2}=\frac{1}{t\log^{2}\left(t+1\right)}$
so $\sum_{t=1}^{\infty}\eta_{t}=\infty$, $\sum_{t=1}^{\infty}d_{t}\eta_{t}^{2}<\infty$
and $\underset{t\rightarrow\infty}{\lim}\eta_{t}d_{t}=0$. Hence,
by Lemma \ref{lem:NoErgodicEXP3} we obtain that EXP3 has no weighted-regret with respect to $d_{t}$ and $\eta_{t}$. However, the feedback for the last
$\frac{T}{2}$ rounds is never received. Therefore, the mixed action $\boldsymbol{p}_{t}$
does not change for all $t\geq\frac{T}{2}$. Then the
cost sequence such that $l_{t}^{\left(i\right)}=0$ for all $i$
and all $t\leq\frac{T}{2}$ and $l_{t}^{\left(1\right)}=0$, $l_{t}^{\left(j\right)}=1$
for all $j>1$ and all $t>\frac{T}{2}$ yields an expected
regret of exactly $\left(1-\frac{1}{K}\right)\frac{T}{2}$.
\end{proof}

\section{Conclusions\label{sec:Conclusions}}

We studied the weighted-regret of online learning with adversarial
bandit feedback and an arbitrary delay sequence $\left\{ d_{t}\right\}$. Our results have implications both for the single-agent and multi-agent cases.

For the single-agent case, our weighted-regret bounds yield standard regret bounds as a special case. We showed an expected regret bound of $O\left( nT^{\frac{3}{4}}+\sqrt{n}T^{\frac{1}{3}}D^{\frac{1}{3}} \right)$ for FKM and $O\left(\sqrt{\log  K\left(KT+D\right)}\right)$ for EXP3, where $D=\sum_{t=1}^{T}\min\left\{ d_{t},T-t+1\right\} $. These bounds
hold even if $D,T$ are unknown thanks to a novel doubling
trick. Our doubling trick can be applied to any online learning algorithm with delays in a plug-and-play manner. Under mild conditions, the novel doubling trick provably retains the order of magnitude dependence on $D$, $T$ (and other parameters) of the regret bound for when $D,T$ are known. 

Our single-agent results in this paper focus on FKM and EXP3 since they are the most widely used algorithms for bandit convex optimization and multi-armed bandits, respectively. Therefore it is crucial to understand how they perform under delays, which are prevalent in practical systems. However, this leaves open the question of what are the best algorithms for delayed bandit feedback.

For multi-armed bandits the lower bound is $O\left(\sqrt{KT+D\log  {K}}\right)$, which is achieved by the algorithm of \cite{zimmert2019optimal}. EXP3, which has lower computational complexity, achieves this lower bound up to the $\log {K}$ that factors $KT$, which is negligible if the average delay is larger than $O\left(\frac{K}{\log  K}\right)$. 

For bandit convex optimization, much less is understood. A breakthrough was made in \citet{bubeck2017kernel}, that introduced a bandit convex optimization algorithm 
that achieves an expected regret of $O\left(n^{9.5}\sqrt{T}\log^{7.5}T\right)$. Recently,  it was shown in \cite{lattimore2020improved} that an algorithm exists that achieves an expected regret of $O(n^{2.5}\sqrt T\log T)$, which improves the bound from \cite{bubeck2016multi}. However, the result in \cite{lattimore2020improved} is non-constructive so an algorithm that achieves the improved bound is still unknown. 
Compared to FKM, the algorithm in \cite{bubeck2017kernel}  suffers from a few drawbacks. First, the $n$ dependence is a high-degree polynomial, which is much worse than the linear dependence of FKM. Second, the algorithm proposed in \citet{bubeck2017kernel} has a $T$ dependent complexity per round. Since this new algorithm may need a very large $T$ to have lower regret than FKM, a $T$ dependant complexity is a serious practical concern.
Finally, it is an open question how robust the algorithm in \citet{bubeck2017kernel} is to delays.  The main difficulty seems to be that their algorithm requires increasing the step-size multiplicatively by a factor larger than one once a certain condition holds. An increasing step-size sequence conflicts with the $T, D$-dependent tuning that optimizes the expected regret with delays or the decreasing step-size sequence that guarantees no weighted-regret for unbounded delay sequences. In contrast, FKM gives a weighted-regret bound as a function of $\eta_{t}$ that can be easily tuned.

For the multi-agent case, we proved that if the algorithms have no weighted-regret, then the expected weighted ergodic distribution of play converges to the set of coarse correlated equilibria (CCE) for a general non-cooperative game. For a two-player zero-sum game, the weighted ergodic average of play converges in $L^{1}$ to the set of Nash equilibria. Then, we showed that FKM and EXP3 have no weighted-regret with their step-size sequence as the weight sequence even under significant unbounded delay sequences (e.g., $d_{t}=O\left(t\log t\right)$) for which their regret is $\Theta(T)$. Hence, by simulating a game and endowing the players with FKM or EXP3, we can use the weighted ergodic distribution or average to approximate an equilibrium. By tuning the weights according to the conditions we provide, this approximation method can still converge even if the algorithms have linear regret. Since delays are prevalent when simulating model-free multi-agent interactions (i.e., games), this extends the set of tools that can approximate equilibria in practice. Approximating equilibria of model-free games in a simulated environment can help to predict their outcomes in practice, or design distributed algorithms in case the equilibria have good global performance.

Our results highlight the role of no weighted-regret in online learning under delayed feedback. This motivates to further study the analogy between no weighted-regret under delays to no-regret in the standard no-delay case. In particular, it might be possible to prove results on the internal weighted-regret, based on the techniques introduced in \cite{blum2007external} for multi-armed bandits. An internal weighted-regret bound will allow approximating the correlated equilibrium in delayed feedback environments, generalizing our result on external weighted-regret and the CCE.

\bibliography{DelayedFeedback}

\newpage
\section{Proof of Theorem \ref{CorrelatedEquilibrium}}

We start by showing that $\mathcal{\mathbb{E}}\left\{\rho_{T}\right\}=\mathcal{\mathbb{E}}\left\{\frac{\sum_{t=1}^{T}\eta_{t}\delta_{\boldsymbol{a}_{t}}}{\sum_{t=1}^{T}\eta_{t}}\right\}$ converges to an $\varepsilon$-CCE of the game as
$T\rightarrow\infty$, for every $\varepsilon>0$. For each $n$, define the cost function $l_{n,t}\left(\boldsymbol{a}_{n}\right)=1-u_{n}\left(\boldsymbol{a}_{n},\boldsymbol{a}_{-n,t}\right)$.
Let $\varepsilon>0$. Since each player $n$ has no weighted-regret with respect to $d_{t}^{n}$ and $\eta_{t}$, then there exists a $T_{0}>0$
such that for all $T>T_{0}$, we have for every $n$ and every action $\boldsymbol{a}_{n}\in\mathcal{A}$:
\begin{align*}
\mathcal{\mathbb{E}}\left\{ \mathbb{E}^{\boldsymbol{a}^{*}\sim\rho_{T}}\left\{ u_{n}\left(\boldsymbol{a}_{n},\boldsymbol{a}_{-n}^{*}\right)-u_{n}\left(\boldsymbol{a}^{*}\right)\right\} \right\} & \underset{\left(a\right)}{=}\mathbb{E}\left\{ \frac{\sum_{t=1}^{T}\eta_{t}\left(u_{n}\left(\boldsymbol{a}_{n},\boldsymbol{a}_{-n,t}\right)-u_{n}\left(\boldsymbol{a}_{t}\right)\right)}{\sum_{t=1}^{T}\eta_{t}}\right\} \\
&=\mathbb{E}\left\{ \frac{\sum_{t=1}^{T}\eta_{t}\left(l_{n,t}\left(\boldsymbol{a}_{n,t}\right)-l_{n,t}\left(\boldsymbol{a}_{n}\right)\right)}{\sum_{t=1}^{T}\eta_{t}}\right\} \underset{\left(b\right)}{\leq}\varepsilon\numberthis\label{eq:66b}
\end{align*}
where (a) uses the definition of $\rho_{T}$ and (b) follows since
player $n$ has no weighted-regret. Now pick $\boldsymbol{a}_{n}=\underset{\boldsymbol{a}_{n}'\in\mathcal{A}}{\arg\max}\mathbb{E}^{\boldsymbol{a}^{*}\sim\rho_{T}}\left\{ u_{n}\left(\boldsymbol{a}_{n}',\boldsymbol{a}_{-n}^{*}\right)\right\} $ in \eqref{eq:66b}. Since $\mathbb{E}^{\boldsymbol{a}^{*}\sim\rho_{T}}\left\{ u_{n}\left(\boldsymbol{a}_{n},\boldsymbol{a}_{-n}^{*}\right)-u_{n}\left(\boldsymbol{a}^{*}\right)\right\} $
is linear in $\rho_{T}$, then
\begin{equation}
\mathcal{\mathbb{E}}\left\{ \mathbb{E}^{\boldsymbol{a}^{*}\sim\rho_{T}}\left\{ u_{n}\left(\boldsymbol{a}_{n},\boldsymbol{a}_{-n}^{*}\right)-u_{n}\left(\boldsymbol{a}^{*}\right)\right\} \right\} =\mathbb{E}^{\boldsymbol{a}^{*}\sim\mathbb{E}\left\{ \rho_{T}\right\} }\left\{ u_{n}\left(\boldsymbol{a}_{n},\boldsymbol{a}_{-n}^{*}\right)-u_{n}\left(\boldsymbol{a}^{*}\right)\right\}
\label{eq:57-1}
\end{equation}
so we conclude that by definition $\mathbb{E}\left\{ \rho_{T}\right\} $
is an $\varepsilon$-CCE. 

Let $\Delta>0$. From Lemma \ref{lem:epsilonNash}, we know that there exists an $\varepsilon_{\Delta}>0$ such that for all $\varepsilon\leq\varepsilon_{\Delta}$
we have $\underset{\rho^{*}\in\mathcal{C}_{0}}{\min}\left\Vert \rho{}_{\varepsilon}-\rho^{*}\right\Vert \leq\Delta$
for all $\rho_{\varepsilon}\in\mathcal{C}_{\varepsilon}$. From \eqref{eq:66b} we know that there exists a large enough
$T_{1}$ such that for all $T>T_{1}$ we have $\mathbb{E}\left\{ \rho_{T}\right\} \in\mathcal{C}_{\varepsilon_{\Delta}}$, which implies that $\underset{\rho^{*}\in\mathcal{C}_{0}}{\min}\left\Vert \mathbb{E}\left\{ \rho_{T}\right\} -\rho^{*}\right\Vert \leq\Delta$. Therefore, $\mathbb{E}\left\{ \rho_{T}\right\} $ converges
to the set of CCE $\mathcal{C}_{0}$ as $T\rightarrow\infty$.

\section{Proof of Theorem \ref{ZeroSumTheorem}}

Recall that the weighted ergodic average of $\boldsymbol{a}_{t}$ is $\boldsymbol{\bar{a}}_{T}\triangleq\frac{\sum_{t=1}^{T}\eta_{t}\boldsymbol{a}_{t}}{\sum_{t=1}^{T}\eta_{t}}$. Let $\varepsilon>0$. Define the ergodic
average of the value of the game by
\begin{equation}
\overline{u}_{T}=\frac{\sum_{t=1}^{T}\eta_{t}u\left(\boldsymbol{y}_{t},\boldsymbol{z}_{t}\right)}{\sum_{t=1}^{T}\eta_{t}}.\label{eq:68}
\end{equation}
Define the row cost function $l_{r,t}\left(\boldsymbol{y}\right)=u\left(\boldsymbol{y},\boldsymbol{z}_{t}\right)$.
Since the row player has no weighted-regret with respect to $d_{t}^{r}$ and $\eta_{t}$ then there exists a $T_{1}>0$ such that for all $T>T_{1}$ and every $\boldsymbol{y}\in\mathcal{A}$ (even in hindsight):
\begin{align*}
\mathbb{E}\left\{ \overline{u}_{T}-u\left(\boldsymbol{y},\boldsymbol{\bar{z}}_{T}\right)\right\} &\underset{\left(a\right)}{\leq}\mathbb{E}\left\{ \frac{\sum_{t=1}^{T}\eta_{t}\left(u\left(\boldsymbol{y}_{t},\boldsymbol{z}_{t}\right)-u\left(\boldsymbol{y},\boldsymbol{z}_{t}\right)\right)}{\sum_{t=1}^{T}\eta_{t}}\right\} \\
&=\mathbb{E}\left\{ \frac{\sum_{t=1}^{T}\eta_{t}\left(l_{r,t}\left(\boldsymbol{y}_{t}\right)-l_{r,t}\left(\boldsymbol{y}\right)\right)}{\sum_{t=1}^{T}\eta_{t}}\right\} \underset{\left(b\right)}{\leq}\frac{\varepsilon}{2}\numberthis\label{eq:69}
\end{align*}
where (a) uses the concavity of $u\left(\boldsymbol{y},\boldsymbol{\bar{z}}_{T}\right)$
in $\boldsymbol{\bar{z}}_{T}$ and (b) uses
the no weighted-regret of the algorithm.

Define the column cost function
$l_{c,t}\left(\boldsymbol{z}\right)=1-u\left(\boldsymbol{y}_{t},\boldsymbol{z}\right)$.
Since the column player has no weighted-regret with respect to $d_{t}^{c}$ and $\eta_{t}$,
then there exists a $T_{2}>0$ such that for all $T>T_{2}$ and for
every  $\boldsymbol{z}\in\mathcal{A}$ (even in hindsight):
\begin{align*}
\mathbb{E}\left\{ u\left(\boldsymbol{\bar{y}}_{T},\boldsymbol{z}\right)-\overline{u}_{T}\right\} & \underset{\left(a\right)}{\leq}\mathbb{E}\left\{ \frac{\sum_{t=1}^{T}\eta_{t}\left(u\left(\boldsymbol{y}_{t},\boldsymbol{z}\right)-u\left(\boldsymbol{y}_{t},\boldsymbol{z}_{t}\right)\right)}{\sum_{t=1}^{T}\eta_{t}}\right\} \\ &
=\mathbb{E}\left\{ \frac{\sum_{t=1}^{T}\eta_{t}\left(l_{c,t}\left(\boldsymbol{z}_{t}\right)-l_{c,t}\left(\boldsymbol{z}\right)\right)}{\sum_{t=1}^{T}\eta_{t}}\right\} \underset{\left(b\right)}{\leq}\frac{\varepsilon}{2}\numberthis\label{eq:70a}
\end{align*}
where (a) uses the convexity of $u\left(\boldsymbol{\bar{y}}_{T},\boldsymbol{z}\right)$
in $\boldsymbol{\bar{y}}_{T}$ and (b) uses the
no weighted-regret of the algorithm.

Now define the best-response to $\boldsymbol{\boldsymbol{\bar{z}}}_{T}$
as $\boldsymbol{y}_{T}^{b}=\underset{\boldsymbol{y}'}{\arg\min}\,u\left(\boldsymbol{y}',\boldsymbol{\boldsymbol{\bar{z}}}_{T}\right)$
and the best-response to $\boldsymbol{\boldsymbol{\bar{y}}}_{T}$
as $\boldsymbol{z}_{T}^{b}=\underset{\boldsymbol{z}'}{\arg\max\,}u\left(\boldsymbol{\bar{y}}_{T},\boldsymbol{z}'\right).$
By choosing $\boldsymbol{y}=\boldsymbol{y}_{T}^{b}$, $\boldsymbol{z}=\boldsymbol{\bar{z}}_{T}$
in \eqref{eq:69} and \eqref{eq:70a} and adding them together we
conclude that for all $T>\max\left\{ T_{1},T_{2}\right\} $
\begin{equation}
\mathbb{E}\left\{ \left|u\left(\boldsymbol{\bar{y}}_{T},\boldsymbol{\bar{z}}_{T}\right)-\underset{\boldsymbol{y}'}{\min}u\left(\boldsymbol{y}',\boldsymbol{\boldsymbol{\bar{z}}}_{T}\right)\right|\right\} \underset{\left(a\right)}{=}\mathbb{E}\left\{ \overline{u}_{T}-u\left(\boldsymbol{y}_{T}^{b},\boldsymbol{\bar{z}}_{T}\right)\right\} +\mathbb{E}\left\{ u\left(\boldsymbol{\bar{y}}_{T},\boldsymbol{\bar{z}}_{T}\right)-\overline{u}_{T}\right\} \leq\varepsilon\label{eq:71}
\end{equation}
where (a) follows since $u\left(\boldsymbol{\bar{y}}_{T},\boldsymbol{\bar{z}}_{T}\right)\geq u\left(\boldsymbol{y}_{T}^{b},\boldsymbol{\bar{z}}_{T}\right)$.
By choosing instead $\boldsymbol{y}=\boldsymbol{\bar{y}}_{T}$, $\boldsymbol{z}=\boldsymbol{z}_{T}^{b}$
in \eqref{eq:69} and \eqref{eq:70a} and adding them together we conclude that for all $T>\max\left\{ T_{1},T_{2}\right\} $
\begin{equation}
\mathbb{E}\left\{ \left|u\left(\boldsymbol{\bar{y}}_{T},\boldsymbol{\bar{z}}_{T}\right)-\underset{\boldsymbol{z}'}{\max}u\left(\boldsymbol{\boldsymbol{\bar{y}}}_{T},\boldsymbol{z}'\right)\right|\right\} \underset{\left(a\right)}{=}\mathbb{E}\left\{ \overline{u}_{T}-u\left(\boldsymbol{\bar{y}}_{T},\boldsymbol{\bar{z}}_{T}\right)\right\} +\mathbb{E}\left\{ u\left(\boldsymbol{\bar{y}}_{T},\boldsymbol{z}_{T}^{b}\right)-\overline{u}_{T}\right\} \leq\varepsilon\label{eq:72}
\end{equation}
where (a) follows since $u\left(\boldsymbol{\bar{y}}_{T},\boldsymbol{\bar{z}}_{T}\right)\leq u\left(\boldsymbol{\bar{y}}_{T},\boldsymbol{z}_{T}^{b}\right)$. 

Let $\Delta>0$. From Lemma \ref{lem:epsilonNash}, we know that there
exists an $\varepsilon_{\Delta}>0$ such that for all $\varepsilon\leq\varepsilon_{\Delta}$
we have $\underset{\boldsymbol{x}^{*}\in\mathcal{N}_{0}}{\min}\left\Vert \boldsymbol{x}_{\varepsilon}-\boldsymbol{x}^{*}\right\Vert \leq\Delta$
for all $\boldsymbol{x}_{\varepsilon}\in\mathcal{N}_{\varepsilon}$.
Let $\delta>0$ and let $\varepsilon=\frac{\varepsilon_{\Delta}\delta}{2}>0$.
Then from \eqref{eq:71},\eqref{eq:72} we know that there exists
a large enough $T_{3}$ such that for all $T>T_{3}$, using Markov inequality:

\begin{equation}
\mathbb{P}\left(\left|u\left(\boldsymbol{\bar{y}}_{T},\boldsymbol{\bar{z}}_{T}\right)-\underset{\boldsymbol{y}'}{\min}u\left(\boldsymbol{y}',\boldsymbol{\boldsymbol{\bar{z}}}_{T}\right)\right|\geq\varepsilon_{\Delta}\right)\leq\frac{\mathbb{E}\left\{ \left|u\left(\boldsymbol{\bar{y}}_{T},\boldsymbol{\bar{z}}_{T}\right)-\underset{\boldsymbol{y}'}{\min}u\left(\boldsymbol{y}',\boldsymbol{\boldsymbol{\bar{z}}}_{T}\right)\right|\right\} }{\varepsilon_{\Delta}}=\frac{\delta}{2}\label{eq:64-1}
\end{equation}
and 
\begin{equation}
\mathbb{P}\left(\left|u\left(\boldsymbol{\bar{y}}_{T},\boldsymbol{\bar{z}}_{T}\right)-\underset{\boldsymbol{z}'}{\max}u\left(\boldsymbol{\boldsymbol{\bar{y}}}_{T},\boldsymbol{z}'\right)\right|\geq\varepsilon_{\Delta}\right)\leq\frac{\mathbb{E}\left\{ \left|u\left(\boldsymbol{\bar{y}}_{T},\boldsymbol{\bar{z}}_{T}\right)-\underset{\boldsymbol{z}'}{\max}u\left(\boldsymbol{\boldsymbol{\bar{y}}}_{T},\boldsymbol{z}'\right)\right|\right\} }{\varepsilon_{\Delta}}=\frac{\delta}{2}.\label{eq:65-1}
\end{equation}
Hence by the union bound over \eqref{eq:64-1} and \eqref{eq:65-1}:
\begin{equation}
\mathbb{P}\left(\underset{\boldsymbol{x}^{*}\in\mathcal{N}_{0}}{\min}\left\Vert \left(\boldsymbol{\bar{y}}_{T},\boldsymbol{\bar{z}}_{T}\right)-\boldsymbol{x}^{*}\right\Vert \leq\Delta\right)\geq\mathbb{P}\left(\left(\boldsymbol{\bar{y}}_{T},\boldsymbol{\bar{z}}_{T}\right)\in\mathcal{N}_{\varepsilon_{\Delta}}\right)\geq1-\delta\label{eq:64a}
\end{equation}
so $\left(\boldsymbol{\bar{y}}_{T},\boldsymbol{\bar{z}}_{T}\right)$
converges in probability to the set of NE. Since $\left(\boldsymbol{\bar{y}}_{T},\boldsymbol{\bar{z}}_{T}\right)-\boldsymbol{x}^{*}$
is bounded, it implies that $\mathbb{E}\left\{ \underset{\boldsymbol{x}^{*}\in\mathcal{N}_{0}}{\min}\left\Vert \left(\boldsymbol{\bar{y}}_{T},\boldsymbol{\bar{z}}_{T}\right)-\boldsymbol{x}^{*}\right\Vert \right\} \rightarrow0$
as $T\rightarrow\infty$. Since $u$ is continuous, 
$u\left(\boldsymbol{\bar{y}}_{T},\boldsymbol{\bar{z}}_{T}\right)\rightarrow v$ in probability as $T\rightarrow\infty$
 where $v$ is the value of the game. Since $u$ is bounded, $u\left(\boldsymbol{\bar{y}}_{T},\boldsymbol{\bar{z}}_{T}\right)\rightarrow v$ in $L^{1}$ as $T\rightarrow\infty$.

\section{The Set of Approximate Equilibria Approaches the Set of Equilibria}

The following lemma shows that for a given game, the sets of $\varepsilon$-NE
and $\varepsilon$-CCE converge to the sets of NE and CCE, respectively,
when $\varepsilon\rightarrow0$. It allows us to convert convergence
to $\mathcal{N}_{\varepsilon}$ and $\mathcal{C}_{\varepsilon}$ for
each $\varepsilon>0$ to convergence to $\mathcal{N}_{0}$ and $\mathcal{C}_{0}$.
\begin{lem}
\label{lem:epsilonNash}Let $d_{\mathcal{N}}\left(\varepsilon\right)=\underset{\boldsymbol{x}_{\varepsilon}\in\mathcal{N}_{\varepsilon}}{\max}\underset{\boldsymbol{x}^{*}\in\mathcal{N}_{0}}{\min}\left\Vert \boldsymbol{x}_{\varepsilon}-\boldsymbol{x}^{*}\right\Vert $
and $d_{\mathcal{C}}\left(\varepsilon\right)=\underset{\rho_{\varepsilon}\in\mathcal{C}_{\varepsilon}}{\max}\underset{\rho^{*}\in\mathcal{C}_{0}}{\min}\left\Vert \rho_{\varepsilon}-\rho^{*}\right\Vert .$
Then $d_{\mathcal{N}}\left(\varepsilon\right)\rightarrow0$ and $d_{\mathcal{C}}\left(\varepsilon\right)\rightarrow0$
as $\varepsilon\rightarrow0$.
\end{lem}
\begin{proof}
Let $A_{\varepsilon}\left(\boldsymbol{y}\right)=\left\{ \boldsymbol{x}\in\mathcal{A}^{N}\,|\,u_{i}\left(\boldsymbol{x}\right)\geq u_{i}\left(\boldsymbol{y}_{i},\boldsymbol{x}_{-i}\right)-\varepsilon\,,\forall i\right\} $.
This is a compact set since $A_{\varepsilon}\left(\boldsymbol{y}\right)=\bigcap_{i}f_{i,\boldsymbol{y}}^{-1}\left(\left[-\varepsilon,u_{i,\max}\right]\right)$
for the continuous $f_{i,\boldsymbol{y}}\left(\boldsymbol{x}\right)=u_{i}\left(\boldsymbol{x}\right)-u_{i}\left(\boldsymbol{y}_{i},\boldsymbol{x}_{-i}\right)$
and $u_{i,\max}=\underset{\boldsymbol{x}\in\mathcal{A}^{N}}{\max}f_{i,\boldsymbol{y}}\left(\boldsymbol{x}\right)$.
Now note that $\mathcal{N_{\varepsilon}}=\bigcap_{\boldsymbol{y}\in\mathcal{A}^{N}}A_{\varepsilon}\left(\boldsymbol{y}\right)$,
since $\mathcal{N_{\varepsilon}}$ only includes action profiles $\boldsymbol{x}$ where no deviation $\boldsymbol{y}$ gives any player more than $\varepsilon$ gain.
 Hence $\mathcal{N}_{\varepsilon}$ is compact for all
$\varepsilon\geq0$. 
Since $\mathcal{N}_{\frac{1}{n}}$ and $\mathcal{N}_{0}$ are compact,
we can define the sequence $\left\{ \tilde{\boldsymbol{x}}_{n}\right\} $
such that
\begin{equation}
\tilde{\boldsymbol{x}}_{n}\in\underset{\boldsymbol{x}_{\varepsilon}\in\mathcal{N}_{\frac{1}{n}}}{\arg\max}\underset{\boldsymbol{x}^{*}\in\mathcal{N}_{0}}{\min}\left\Vert \boldsymbol{x}_{\varepsilon}-\boldsymbol{x}^{*}\right\Vert. \label{eq:67}
\end{equation}
Since $\mathcal{A}^{N}$ is compact then the infinite sequence $\tilde{\boldsymbol{x}}_{n}$
has a subsequence $\tilde{\boldsymbol{x}}_{n_{k}}$ that converges
to a point $\tilde{\boldsymbol{x}}\in\mathcal{A}^{N}$ (i.e., Bolzano--Weierstrass
Theorem, see \citet{simmons1963introduction}). Since $\mathcal{N}_{\varepsilon_{1}}\subseteq\mathcal{N}_{\varepsilon_{2}}$
if $\varepsilon_{2}\geq\varepsilon_{1}$,  we must have $\tilde{\boldsymbol{x}}\in\bigcap_{n=1}^{\infty}\mathcal{N}_{\frac{1}{n}}$
so $\underset{\boldsymbol{y}_{i}}{\max}\: u_{i}\left(\boldsymbol{y}_{i},\boldsymbol{\tilde{\boldsymbol{x}}}_{-i}\right)-u_{i}\left(\tilde{\boldsymbol{x}}\right)\leq\frac{1}{n}$
for all $n$, implying that $\underset{\boldsymbol{y}_{i}}{\max}\: u_{i}\left(\boldsymbol{y}_{i},\boldsymbol{\tilde{\boldsymbol{x}}}_{-i}\right)\leq u_{i}\left(\tilde{\boldsymbol{x}}\right)$
and $\tilde{\boldsymbol{x}}\in\mathcal{N}_{0}$. The fact that  $\mathcal{N}_{\varepsilon_{1}}\subseteq\mathcal{N}_{\varepsilon_{2}}$
if $\varepsilon_{2}\geq\varepsilon_{1}$ also implies that $d_{\mathcal{N}}\left(\varepsilon\right)$
is non-increasing. Additionally, $d_{\mathcal{N}}\left(\varepsilon\right)$
is bounded since $\mathcal{A}^{N}$ is bounded. Hence, $\underset{n\rightarrow\infty}{\lim}d_{\mathcal{N}}\left(\frac{1}{n}\right)$
exists. However, since $\tilde{\boldsymbol{x}}_{n_{k}}\in\mathcal{N}_{\frac{1}{n_{k}}}$
and $\tilde{\boldsymbol{x}}\in\mathcal{N}_{0}$ then 
\begin{equation}
d_{\mathcal{N}}\left(\frac{1}{n_{k}}\right)=\underset{\boldsymbol{x}_{\varepsilon}\in\mathcal{N}_{\frac{1}{n_{k}}}}{\max}\underset{\boldsymbol{x}^{*}\in\mathcal{N}_{0}}{\min}\left\Vert \boldsymbol{x}_{\varepsilon}-\boldsymbol{x}^{*}\right\Vert =\underset{\boldsymbol{x}^{*}\in\mathcal{N}_{0}}{\min}\left\Vert \tilde{\boldsymbol{x}}_{n_{k}}-\boldsymbol{x}^{*}\right\Vert \leq\left\Vert \tilde{\boldsymbol{x}}_{n_{k}}-\tilde{\boldsymbol{x}}\right\Vert \label{eq:70-1}
\end{equation}
so $\underset{k\rightarrow\infty}{\lim}d_{\mathcal{N}}\left(\frac{1}{n_{k}}\right)=0$
since $\tilde{\boldsymbol{x}}_{n_{k}}\rightarrow\tilde{\boldsymbol{x}}$
. Hence, we must have $\underset{n\rightarrow\infty}{\lim}d_{\mathcal{N}}\left(\frac{1}{n}\right)=0$
and $\underset{\varepsilon\rightarrow0}{\lim}d_{\mathcal{N}}\left(\varepsilon\right)=0$. 

Let $\mathcal{P}\left(\mathcal{A}^{N}\right)$ be the set of all Borel probability
measures over $\mathcal{A}^{N}$, equipped with the weak-{*} topology
(see \citet{simmons1963introduction}). As discussed in Section \ref{sec:Games}, a CCE is a CE when all of the departure functions are constant and therefore continuous. Hence, by simply replacing the zero constant of the half-space with $-\varepsilon$ in the last line of the proof of \citet[Theorem 9, Page 194]{stoltz2007learning}, we conclude from their Theorem 9 that $\mathcal{C}_{\varepsilon}$ is compact for all $\varepsilon\geq0$.
Define $\tilde{\rho}_{n}\in\underset{\rho{}_{\varepsilon}\in\mathcal{C}_{\frac{1}{n}}}{\arg\max}\underset{\rho^{*}\in\mathcal{C}_{0}}{\min}\left\Vert \rho_{\varepsilon}-\rho^{*}\right\Vert $. Then,  Prokhorov's Theorem, given as Proposition 8 in \cite[Page 194]{stoltz2007learning}, states that there exists a subsequence $\tilde{\rho}_{n_{k}}$ that converges to a point $\tilde{\rho}\in\mathcal{\mathcal{P}\left(\mathcal{A}^{N}\right)}$.
Since $\mathcal{C}_{\varepsilon_{1}}\subseteq\mathcal{C}_{\varepsilon_{2}}$
if $\varepsilon_{2}\geq\varepsilon_{1}$ then $\tilde{\rho}\in\bigcap_{n=1}^{\infty}\mathcal{C}_{\frac{1}{n}}$
and therefore $\mathbb{E}^{\boldsymbol{x}^{*}\sim\tilde{\rho}}\left\{ \underset{\boldsymbol{y}_{i}}{\max}u_{i}\left(\boldsymbol{y}_{i},\boldsymbol{x}_{-i}^{*}\right)-u_{i}\left(\boldsymbol{x}^{*}\right)\right\} \leq 0$
so $\tilde{\rho}\in\mathcal{C}_{0}$. Following the same argument
as in \eqref{eq:70-1} we conclude that $\underset{\varepsilon\rightarrow0}{\lim}d_{\mathcal{C}}\left(\varepsilon\right)=0$.
\end{proof}

\section{Upper Bound on the Effect of the Delays on the Regret}
\begin{lem}
\label{Counting Delays Lemma}Let $\left\{ \eta_{t}\right\} $ be
a non-increasing step-size sequence. Let $\left\{ d_{t}\right\} $ be a delay 
such that the cost from round $t$ is received at round $t+d_{t}$. Let $\mathcal{S}_{t}$ be the
set of feedback samples received and used at round $t$. Define the
set $\mathcal{\mathcal{M}^{*}}$ of all samples that are not received
and used before round $T$. Then
\begin{equation}
\sum_{t=1}^{T}\sum_{s\in\mathcal{S}_{t}}\eta_{s}\left(\sum_{r=s}^{t-1}\sum_{q\in\mathcal{S}_{r}}\eta_{q}+\sum_{q\in\mathcal{S}_{t},q<s}\eta_{q}\right)\leq2\sum_{t\notin\mathcal{\mathcal{\mathcal{M}^{*}}}}\eta_{t}^{2}d_{t}.\label{eq:73}
\end{equation}
\end{lem}

\begin{proof}
Up to the weights $\eta_q$, the quantity inside the parentheses in \eqref{eq:73}:
\begin{equation}
    Q_{s,t}\triangleq\sum_{r=s}^{t-1}\sum_{q\in\mathcal{S}_{r}}\eta_{q}+\sum_{q\in\mathcal{S}_{t},q<s}\eta_{q}\label{eq:68b}
\end{equation}
counts of the number of feedback samples received and used between round $s$ and round $t$ such that $s\in\mathcal{S}_{t}$, before the feedback from round $s$ was used. To prove \eqref{eq:73}, we want to upper bound $\sum_{t=1}^{T}\sum_{s\in\mathcal{S}_{t}}\eta_{s}Q_{s,t}$ for all possible delay sequences $\left\{ d_{t}\right\}$. A sample can be received and used between round $s$ and round $t$ if it belongs to one of two types:

\begin{enumerate}
    \item The first type is a feedback sample from $q\geq s$ that is received and used before the feedback from round s is used. There are a maximum of $d_{s}$ feedback samples of this type. We denote the contribution of these samples to $Q_{s,t}$ by $Q_{s,t}^{1}$. Each feedback sample can contribute $\eta_{q}\leq\eta_{s}$ with $q\geq s$ (since $\eta_{t}$ is non-increasing) to $Q_{s,t}^{1}$ for $s\in\mathcal{S}_{t}$. We over count them by giving each $Q_{s,t}^{1}$ term all of its $d_{s}$ possible samples of this type. Summing over all $t$:
    \begin{equation}
       \sum_{t=1}^{T}\sum_{s\in\mathcal{S}_{t}}\eta_{s}Q_{s,t}^{1}\leq\sum_{t=1}^{T}\sum_{s\in\mathcal{S}_{t}}\eta_{s}^{2}d_{s}=\sum_{t\notin\mathcal{\mathcal{\mathcal{\mathcal{M^{*}}}}}}\eta_{s}^{2}d_{s}.\label{eq:67a}
    \end{equation}
    
    \item The second type is a feedback sample from $q<s$ that is received and used before $s$ is used. We denote the contribution of these samples to $Q_{s,t}$ by $Q_{s,t}^{2}$. The samples from round $q$ can contribute to a maximum of $d_{q}$ different $Q_{s,t}^{2}$ terms, all with $s\geq q$. This follows simply because the feedback sample of $q$ is not received before $q+d_{q}$. Let $\Gamma_{q}$ be the set of rounds $s$ such that the samples from round $q$ contribute to $Q_{s,t}^{2}$. Then
    \begin{equation}    \sum_{t=1}^{T}\sum_{s\in\mathcal{S}_{t}}\eta_{s}Q_{s,t}^{2}\underset{\left(a\right)}{=}\sum_{q\notin\mathcal{\mathcal{\mathcal{\mathcal{M^{*}}}}}}\sum_{s\in\Gamma_{q}}\eta_{s}\eta_{q}\underset{\left(b\right)}{\leq}\sum_{q\notin\mathcal{\mathcal{\mathcal{\mathcal{M^{*}}}}}}\eta_{q}^{2}\left|\Gamma_{q}\right|\leq\sum_{q\notin\mathcal{\mathcal{\mathcal{\mathcal{M^{*}}}}}}\eta_{q}^{2}d_{q}\label{eq:68a}
    \end{equation}
    where (a) follows since only rounds $q$ that their feedback is received and used sometime before round $T$ are counted in $Q_{s,t}^{2}$ for some $s,t$. Inequality (b) uses $\eta_{s}\leq\eta_{q}$ since $\eta_{t}$ is non-increasing and $s\geq q$ for all $s\in\Gamma_{q}$.
\end{enumerate}

Summing the contribution to \eqref{eq:68b} from these two possible types of samples, we have $Q_{s,t}=Q_{s,t}^{1}+Q_{s,t}^{2}$ so by summing \eqref{eq:67a} and \eqref{eq:68a} we obtain \eqref{eq:73}.
\end{proof}

\section{Proof of Lemma \ref{WeightedRegretConvex}: Weighted-Regret of FKM with Delays}

Recall that $\mathcal{S}_{t}$ is the
set of feedback samples received and used at round $t$, and that $\mathcal{M}=\left\{ t\,|\,t+d_{t}>T\right\}$ is the set of samples that are not received before round $T$. Let $\boldsymbol{a}^{*}=\underset{\boldsymbol{a}\in\mathcal{K}}{\arg\min}\sum_{t=1}^{T}\eta_{t}l_{t}\left(\boldsymbol{a}\right)$ and note that $\boldsymbol{a}^{*}$ is random for an adaptive adversary.
We have
\begin{align*}
\sum_{t=1}^{T}\eta_{t}\mathbb{E}\left\{ l_{t}\left(\boldsymbol{a}_{t}\right)-l_{t}\left(\boldsymbol{a}^{*}\right)\right\} & =\mathbb{E}\left\{ \sum_{t=1}^{T}\sum_{s\in\mathcal{S}_{t}}\eta_{s}\left(l_{s}\left(\boldsymbol{a}_{s}\right)-l_{s}\left(\boldsymbol{a}^{*}\right)\right)\right\} +\mathbb{E}\left\{ \sum_{t\in\mathcal{M}}\eta_{t}\left(l_{t}\left(\boldsymbol{a}_{t}\right)-l_{t}\left(\boldsymbol{a}^{*}\right)\right)\right\} \\
& \underset{\left(a\right)}{\leq}\mathbb{E}\left\{ \sum_{t=1}^{T}\sum_{s\in\mathcal{S}_{t}}\eta_{s}\left(l_{s}\left(\boldsymbol{a}_{s}\right)-l_{s}\left(\boldsymbol{a}^{*}\right)\right)\right\} +\sum_{t\in\mathcal{M}}\eta_{t}\\
& \underset{\left(b\right)}{\leq}\mathbb{E}\left\{ \sum_{t=1}^{T}\sum_{s\in\mathcal{S}_{t}}\eta_{s}\left(l_{s}\left(\boldsymbol{x}_{s}\right)-l_{s}\left(\boldsymbol{a}^{*}\right)\right)\right\} +\sum_{t\in\mathcal{M}}\eta_{t}+\delta L\sum_{t\notin\mathcal{M}}\eta_{t}\numberthis \label{eq:70}
\end{align*}
where (a) uses $l_{t}\left(\boldsymbol{a}\right)\in\left[0,1\right]$
and (b) uses $\left|l_{s}\left(\boldsymbol{x}_{s}\right)-l_{s}\left(\boldsymbol{a}_{s}\right)\right|\leq L\left\Vert \boldsymbol{x}_{s}-\boldsymbol{a}_{s}\right\Vert \leq\delta L.$

Define $s_{-},s_{+}$ as the step a moment before and a moment after the algorithm uses the feedback from round $s$, which updates the action from $\boldsymbol{a}_{s_{-}}$ to $\boldsymbol{a}_{s_{+}}$. Both $s_{-}$
and $s_{+}$ are algorithm update steps that take place in round $t$ of the game if $s\in\mathcal{S}_{t}$.
Define the projection $\boldsymbol{a}_{\delta}^{*}=\prod_{\mathcal{K}_{\delta}}\left(\boldsymbol{a}^{*}\right)$, where $\mathcal{K}_{\delta}=\left\{ \boldsymbol{x}\,|\,\frac{1}{1-\delta}\boldsymbol{x}\in\mathcal{K}\right\}$ as defined in Algorithm \ref{alg:FKM}.
Recall that $\hat{l}\left(\boldsymbol{x}\right)=\mathbb{E}^{\boldsymbol{u}\in\mathbb{S}_{1}}\left\{ l\left(\boldsymbol{x}+\delta\boldsymbol{u}\right)\right\}$. We bound the first term in \eqref{eq:70} as follows
\begin{align*}
\sum_{t=1}^{T}\sum_{s\in\mathcal{S}_{t}}\eta_{s}\left(l_{s}\left(\boldsymbol{x}_{s}\right)-l_{s}\left(\boldsymbol{a}^{*}\right)\right) & \underset{\left(a\right)}{\leq}L\left|\mathcal{K}\right|\delta\sum_{t\notin\mathcal{M}}\eta_{t}+\sum_{t=1}^{T}\sum_{s\in\mathcal{S}_{t}}\eta_{s}\left(l_{s}\left(\boldsymbol{x}_{s}\right)-l_{s}\left(\boldsymbol{a}_{\delta}^{*}\right)\right)\\ &  \underset{\left(b\right)}{\leq}\left(2+\left|\mathcal{K}\right|\right)L\delta\sum_{t\notin\mathcal{M}}\eta_{t}+\sum_{t=1}^{T}\sum_{s\in\mathcal{S}_{t}}\eta_{s}\left(\hat{l}_{s}\left(\boldsymbol{x}_{s}\right)-\hat{l}_{s}\left(\boldsymbol{a}_{\delta}^{*}\right)\right)\numberthis\label{eq:80}
\end{align*}
where (a) follows from the Lipschitz continuity of $l_{s}$ since $\left\Vert \boldsymbol{a}^{*}-\prod_{\mathcal{K}_{\delta}}\left(\boldsymbol{a}^{*}\right)\right\Vert \leq\left\Vert \boldsymbol{a}^{*}-\left(1-\delta\right)\boldsymbol{a}^{*}\right\Vert \leq\delta\left|\mathcal{K}\right|$.
Inequality (b) uses the Lipschitz continuity again, this time on $l_{s}\left(\boldsymbol{x}_{s}\right)$ and $l_{s}\left(\boldsymbol{a}_{\delta}^{*}\right)$:
\begin{equation}
\hat{l}_{s}\left(\boldsymbol{x}\right)-l_{s}\left(\boldsymbol{x}\right)=\mathbb{E}^{\boldsymbol{u}\in\mathbb{S}_1}\left\{ l_{s}\left(\boldsymbol{x}+\delta\boldsymbol{u}\right)-l_{s}\left(\boldsymbol{x}\right)\right\} \leq\delta L\mathbb{E}^{\boldsymbol{u}\in\mathbb{S}_{1}}\left\{ \left\Vert \boldsymbol{u}\right\Vert \right\} =\delta L.\label{eq:81}
\end{equation}
Now recall that $g_{t}=\frac{n}{\delta}l_{t}\left(\boldsymbol{x}_{t}+\delta\boldsymbol{u}_{t}\right)\boldsymbol{u}_{t}$
where $\boldsymbol{u}_{t}$ is on the unit sphere $\mathbb{S}_{1}$,
and define
\begin{equation}
h_{t}\left(\boldsymbol{x}\right)\triangleq\hat{l}_{t}\left(\boldsymbol{x}\right)+\left(\boldsymbol{g}_{t}-\nabla\hat{l}_{t}\left(\boldsymbol{x}_{t}\right)\right)^{T}\boldsymbol{x}\label{eq:82}
\end{equation}
for which $\nabla h_{t}\left(\boldsymbol{x}_{t}\right)=\boldsymbol{g}_{t}$
and $\mathbb{E}\left\{ h_{t}\left(\boldsymbol{x}\right)\right\} =\mathbb{E}\left\{ \hat{l}_{t}\left(\boldsymbol{x}\right)\right\} $ for all $\boldsymbol{x}$, and also $\mathbb{E}\left\{ h_{t}\left(\boldsymbol{x}_{t}\right)\right\} =\mathbb{E}\left\{ \hat{l}_{t}\left(\boldsymbol{x}_{t}\right)\right\} $ since $\boldsymbol{u}_{t}$ is independent of $\boldsymbol{x}_{t}$ and $l_{t}$
(see Lemma \ref{SPSA}). Next note that 
\begin{multline}
\left\Vert \boldsymbol{x}_{s_{+}}-\boldsymbol{a}_{\delta}^{*}\right\Vert ^{2}=\left\Vert \prod_{\mathcal{K}_{\delta}}\left(\boldsymbol{x}_{s_{-}}-\eta_{s}\nabla h_{s}\left(\boldsymbol{x}_{s}\right)\right)-\boldsymbol{a}_{\delta}^{*}\right\Vert ^{2}\underset{\left(a\right)}{\leq}\left\Vert \left(\boldsymbol{x}_{s_{-}}-\boldsymbol{a}_{\delta}^{*}\right)-\eta_{s}\nabla h_{s}\left(\boldsymbol{x}_{s}\right)\right\Vert ^{2}\\
=\left\Vert \boldsymbol{x}_{s_{-}}-\boldsymbol{a}_{\delta}^{*}\right\Vert ^{2}-2\eta_{s}\left\langle \boldsymbol{x}_{s_{-}}-\boldsymbol{a}_{\delta}^{*},\nabla h_{s}\left(\boldsymbol{x}_{s}\right)\right\rangle +\eta_{s}^{2}\left\Vert \nabla h_{s}\left(\boldsymbol{x}_{s}\right)\right\Vert ^{2}=\left\Vert \boldsymbol{x}_{s_{-}}-\boldsymbol{a}_{\delta}^{*}\right\Vert ^{2}+\eta_{s}^{2}\left\Vert \nabla h_{s}\left(\boldsymbol{x}_{s}\right)\right\Vert ^{2}\\
-2\eta_{s}\left\langle \boldsymbol{x}_{s}-\boldsymbol{a}_{\delta}^{*},\nabla h_{s}\left(\boldsymbol{x}_{s}\right)\right\rangle -2\eta_{s}\left\langle \sum_{r=s}^{t-1}\sum_{q\in\mathcal{S}_{r}}\left(\boldsymbol{x}_{q_{+}}-\boldsymbol{x}_{q_{-}}\right)+\sum_{q\in\mathcal{S}_{t},q<s}\left(\boldsymbol{x}_{q_{+}}-\boldsymbol{x}_{q_{-}}\right),\nabla h_{s}\left(\boldsymbol{x}_{s}\right)\right\rangle \\
\underset{\left(b\right)}{\leq}\left\Vert \boldsymbol{x}_{s_{-}}-\boldsymbol{a}_{\delta}^{*}\right\Vert ^{2}-2\eta_{s}\left(h_{s}\left(\boldsymbol{x}_{s}\right)-h_{s}\left(\boldsymbol{a}_{\delta}^{*}\right)\right)+\eta_{s}^{2}\left\Vert \boldsymbol{g}_{s}\right\Vert ^{2}\\
-2\eta_{s}\left\langle \sum_{r=s}^{t-1}\sum_{q\in\mathcal{S}_{r}}\left(\boldsymbol{x}_{q_{+}}-\boldsymbol{x}_{q_{-}}\right)+\sum_{q\in\mathcal{S}_{t},q<s}\left(\boldsymbol{x}_{q_{+}}-\boldsymbol{x}_{q_{-}}\right),\nabla h_{s}\left(\boldsymbol{x}_{s}\right)\right\rangle \label{eq:72a}
\end{multline}
where (a) follows since $\prod_{\mathcal{K}_{\delta}}$ is the projection
of $\boldsymbol{x}_{s_{-}}-\eta_{s}\nabla h_{s}\left(\boldsymbol{x}_{s}\right)$
onto the convex $\mathcal{K}_{\delta}$. Inequality
(b) uses the convexity and differentiability of $h_{s}$ on $\mathcal{K}_{\delta}$, so
$h_{s}\left(\boldsymbol{a}_{\delta}^{*}\right)\geq h_{s}\left(\boldsymbol{x}_{s}\right)+\left\langle \boldsymbol{a}_{\delta}^{*}-\boldsymbol{x}_{s},\nabla h_{s}\left(\boldsymbol{x}_{s}\right)\right\rangle $. 

\subsection{Adaptive Adversary}

First note that for any $\boldsymbol{x}\in\mathcal{K}$
\begin{align*}
\left|\sum_{t\notin\mathcal{M}}\eta_{t}\left(\hat{l}_{t}\left(\boldsymbol{x}\right)-h_{t}\left(\boldsymbol{x}\right)\right)\right| & \underset{\left(a\right)}{=}\left|\sum_{t\notin\mathcal{M}}\eta_{t}\left\langle \boldsymbol{g}_{t}-\nabla\hat{l}_{t}\left(\boldsymbol{x}_{t}\right),\boldsymbol{x}\right\rangle \right| \\ & \leq\left\Vert \boldsymbol{x}\right\Vert \left\Vert \sum_{t\notin\mathcal{M}}\eta_{t}\left(\boldsymbol{g}_{t}-\nabla\hat{l}_{t}\left(\boldsymbol{x}_{t}\right)\right)\right\Vert \leq\left|\mathcal{K}\right|\left\Vert \sum_{t\notin\mathcal{M}}\eta_{t}\left(\boldsymbol{g}_{t}-\nabla\hat{l}_{t}\left(\boldsymbol{x}_{t}\right)\right)\right\Vert\numberthis\label{eq:78a}
\end{align*}
where (a) uses \eqref{eq:82}. The expectation of the last term can be bounded as follows 
\begin{multline}
\mathbb{E}^{2}\left\{ \left\Vert \sum_{t\notin\mathcal{M}}\eta_{t}\left(\boldsymbol{g}_{t}-\nabla\hat{l}_{t}\left(\boldsymbol{x}_{t}\right)\right)\right\Vert \right\} \leq\mathbb{E}\left\{ \left\Vert \sum_{t\notin\mathcal{M}}\eta_{t}\left(\boldsymbol{g}_{t}-\nabla\hat{l}_{t}\left(\boldsymbol{x}_{t}\right)\right)\right\Vert ^{2}\right\} =\\\sum_{t\notin\mathcal{M}}\eta_{t}^{2}\mathbb{E}\left\{ \left\Vert \boldsymbol{g}_{t}-\nabla\hat{l}_{t}\left(\boldsymbol{x}_{t}\right)\right\Vert ^{2}\right\} +\sum_{t_{1}\notin\mathcal{M}}\sum_{t_{1}\neq t_{2}}\eta_{t_{1}}\eta_{t_{2}}\mathbb{E}\left\{ \left\langle \boldsymbol{g}_{t_{1}}-\nabla\hat{l}_{t_{1}}\left(\boldsymbol{x}_{t_{1}}\right),\boldsymbol{g}_{t_{2}}-\nabla\hat{l}_{t_{2}}\left(\boldsymbol{x}_{t_{2}}\right)\right\rangle \right\} \underset{\left(a\right)}{\leq}\\2\sum_{t\notin\mathcal{M}}\eta_{t}^{2}\mathbb{E}\left\{ \left\Vert \boldsymbol{g}_{t}\right\Vert ^{2}+\left\Vert \nabla\hat{l}_{t}\left(\boldsymbol{x}_{t}\right)\right\Vert ^{2}\right\} \underset{\left(b\right)}\leq2\sum_{t\notin\mathcal{M}}\eta_{t}^{2}\left(\frac{n^{2}}{\delta^{2}}+L^{2}\right)\label{eq:79a}
\end{multline}
where (a) uses that $\left\langle \boldsymbol{g}_{t_{1}}-\nabla\hat{l}_{t_{1}}\left(\boldsymbol{x}_{t_{1}}\right),\mathbb{E}\left\{ \boldsymbol{g}_{t_{2}}-\nabla\hat{l}_{t_{2}}\left(\boldsymbol{x}_{t_{2}}\right)\,|\,\sigma\left(\mathcal{F}_{t_{2}},\boldsymbol{u}_{t_{1}}\right)\right\} \right\rangle =0$
for all $t_{1}<t_{2}$, which follows from Lemma \ref{SPSA}. Inequality (b) uses that $\hat{l}_{t}$ is differentiable and $L$-Lipschitz continuous.

Then
\begin{multline}
\mathbb{E}\left\{ \sum_{t=1}^{T}\sum_{s\in\mathcal{S}_{t}}\eta_{s}\left(\hat{l}_{s}\left(\boldsymbol{x}_{s}\right)-\hat{l}_{s}\left(\boldsymbol{a}_{\delta}^{*}\right)\right)\right\} \underset{\left(a\right)}{\leq}\\
\mathbb{E}\left\{ \sum_{t=1}^{T}\sum_{s\in\mathcal{S}_{t}}\eta_{s}\left(h_{s}\left(\boldsymbol{x}_{s}\right)-h_{s}\left(\boldsymbol{a}_{\delta}^{*}\right)\right)\right\} +\left|\mathcal{K}\right|\sqrt{2\sum_{t=1}^{T}\sum_{s\in\mathcal{S}_{t}}\eta_{s}^{2}\left(\frac{n^{2}}{\delta^{2}}+L^{2}\right)}
\underset{\left(b\right)}{\leq}\\\frac{1}{2}\sum_{t=1}^{T}\sum_{s\in\mathcal{S}_{t}}\mathbb{E}\left\{ \left\Vert \boldsymbol{x}_{s_{-}}-\boldsymbol{a}_{\delta}^{*}\right\Vert ^{2}-\left\Vert \boldsymbol{x}_{s_{+}}-\boldsymbol{a}_{\delta}^{*}\right\Vert ^{2}\right\} +\frac{1}{2}\sum_{t=1}^{T}\sum_{s\in\mathcal{S}_{t}}\eta_{s}^{2}\mathbb{E}\left\{ \left\Vert \boldsymbol{g}_{s}\right\Vert ^{2}\right\} +\left|\mathcal{K}\right|\sqrt{2\sum_{t\notin\mathcal{M}}\eta_{t}^{2}\left(\frac{n^{2}}{\delta^{2}}+L^{2}\right)}\\
-\sum_{t=1}^{T}\sum_{s\in\mathcal{S}_{t}}\eta_{s}\mathbb{E}\left\{ \left\langle \sum_{r=s}^{t-1}\sum_{q\in\mathcal{S}_{r}}\left(\boldsymbol{x}_{q_{+}}-\boldsymbol{x}_{q_{-}}\right)+\sum_{q\in\mathcal{S}_{t},q<s}\left(\boldsymbol{x}_{q_{+}}-\boldsymbol{x}_{q_{-}}\right),\mathbb{E}\left\{ \nabla h_{s}\left(\boldsymbol{x}_{s}\right)\,|\,\mathcal{F}_{s_{-}}\right\} \right\rangle \right\} 
\underset{\left(c\right)}{\leq}\\\frac{\left|\mathcal{K}\right|^{2}}{2}+\frac{n^{2}}{2\delta^{2}}\sum_{t\notin\mathcal{M}}\eta_{t}^{2}+\left|\mathcal{K}\right|\sqrt{2\sum_{t\notin\mathcal{M}}\eta_{t}^{2}\left(\frac{n^{2}}{\delta^{2}}+L^{2}\right)}\\
+\frac{n}{\delta}\sum_{t=1}^{T}\sum_{s\in\mathcal{S}_{t}}\eta_{s}\left(\sum_{r=s}^{t-1}\sum_{q\in\mathcal{S}_{r}}\mathbb{E}\left\{ \left\Vert \boldsymbol{x}_{q_{+}}-\boldsymbol{x}_{q_{-}}\right\Vert \right\} +\sum_{q\in\mathcal{S}_{t},q<s}\mathbb{E}\left\{ \left\Vert \boldsymbol{x}_{q_{+}}-\boldsymbol{x}_{q_{-}}\right\Vert \right\} \right)\label{eq:80a}
\end{multline}
where (a) uses \eqref{eq:78a} and \eqref{eq:79a} on $\boldsymbol{x}=\boldsymbol{a}_{\delta}^{*}$ and $\mathbb{E}\left\{ h_{s}\left(\boldsymbol{x}_{s}\right)\right\} =\mathbb{E}\left\{ \hat{l}_{s}\left(\boldsymbol{x}_{s}\right)\right\}$, (b) uses \eqref{eq:72a}
and (c) uses the telescopic sum with $\left\Vert \boldsymbol{x}_{0}-\boldsymbol{a}_{\delta}^{*}\right\Vert ^{2}-\left\Vert \boldsymbol{x}_{T}-\boldsymbol{a}_{\delta}^{*}\right\Vert ^{2}\leq\left|\mathcal{K}\right|^{2}$,
that $\boldsymbol{g}_{s}=\frac{n}{\delta}l_{s}\left(\boldsymbol{x}_{s}+\delta\boldsymbol{u}_{s}\right)\boldsymbol{u}_{s}$,
and also Cauchy-Schwarz and then applying Lemma \ref{SPSA} to obtain 
$\left\Vert \mathbb{E}\left\{ \nabla h_{s}\left(\boldsymbol{x}_{s}\right)\,|\,\mathcal{F}_{s_{-}}\right\} \right\Vert \leq\mathbb{E}\left\{ \left\Vert \boldsymbol{g}_{s}\right\Vert \,|\,\mathcal{F}_{s_{-}}\right\} \leq\frac{n}{\delta}$.
Finally we bound the last term by bounding 
\begin{equation}
\left\Vert \boldsymbol{x}_{q_{+}}-\boldsymbol{x}_{q_{-}}\right\Vert _{2}=\left\Vert \prod_{\mathcal{K}_{\delta}}\left(\boldsymbol{x}_{q_{-}}-\eta_{q}\boldsymbol{g}_{q}\right)-\boldsymbol{x}_{q_{-}}\right\Vert _{2}\underset{\left(a\right)}{\leq}\eta_{q}\left\Vert \boldsymbol{g}_{q}\right\Vert _{2}\leq n\frac{\eta_{q}}{\delta}\label{eq:82a}
\end{equation}
where (a) follows since $\prod_{\mathcal{K}_{\delta}}$ is the projection
of $\boldsymbol{x}_{q_{-}}-\eta_{q}\boldsymbol{g}_{q}$ onto the convex
$\mathcal{K}_{\delta}$. We obtain
\begin{multline}
\frac{n}{\delta}\sum_{t=1}^{T}\sum_{s\in\mathcal{S}_{t}}\eta_{s}\left(\sum_{r=s}^{t-1}\sum_{q\in\mathcal{S}_{r}}\mathbb{E}\left\{ \left\Vert \boldsymbol{x}_{q_{+}}-\boldsymbol{x}_{q_{-}}\right\Vert \right\} +\sum_{q\in\mathcal{S}_{t},q<s}\mathbb{E}\left\{ \left\Vert \boldsymbol{x}_{q_{+}}-\boldsymbol{x}_{q_{-}}\right\Vert \right\} \right)\\
\leq\frac{n^{2}}{\delta^{2}}\sum_{t=1}^{T}\sum_{s\in\mathcal{S}_{t}}\eta_{s}\left(\sum_{r=s}^{t-1}\sum_{q\in\mathcal{S}_{r}}\eta_{q}+\sum_{q\in\mathcal{S}_{t},q<s}\eta_{q}\right)\underset{\left(a\right)}{\leq}2\frac{n^{2}}{\delta^{2}}\sum_{t\notin\mathcal{M}}\eta_{t}^{2}d_{t}\label{eq:86b}
\end{multline}
where (a) uses Lemma \ref{Counting Delays Lemma}. We conclude by
applying \eqref{eq:86b} on \eqref{eq:80a} and adding \eqref{eq:80}
and \eqref{eq:70}.

\subsection{Oblivious Adversary }
For an oblivious adversary, $l_{q}$ is not random and for $q>s$ does not depend on $\boldsymbol{a}_{s}$, so by Lemma \ref{SPSA}
\begin{equation}
\left\Vert \mathbb{E}\left\{ \nabla h_{s}\left(\boldsymbol{x}_{s}\right)\,|\,\mathcal{F}_{s_{-}}\right\} \right\Vert \underset{\left(a\right)}{=}\left\Vert \mathbb{E}\left\{ \frac{n}{\delta}l_{s}\left(\boldsymbol{x}_{s}+\delta\boldsymbol{u}_{s}\right)\boldsymbol{u}_{s}\,|\,\mathcal{F}_{s_{-}}\right\} \right\Vert =\left\Vert \nabla\hat{l}_{s}\left(\boldsymbol{x}_{s}\right)\right\Vert \underset{\left(b\right)}{\leq}L\label{eq:79b}
\end{equation}
where (a) uses \eqref{eq:82} and (b) follows since $\hat{l}_{s}$ is differentiable and $L$-Lipschitz continuous. Then
\begin{multline}
\sum_{t=1}^{T}\sum_{s\in\mathcal{S}_{t}}\eta_{s}\mathbb{E}\left\{ \hat{l}_{s}\left(\boldsymbol{x}_{s}\right)-\hat{l}_{s}\left(\boldsymbol{a}_{\delta}^{*}\right)\right\} \underset{\left(a\right)}{=}\sum_{t=1}^{T}\sum_{s\in\mathcal{S}_{t}}\eta_{s}\mathbb{E}\left\{ h_{s}\left(\boldsymbol{x}_{s}\right)-h_{s}\left(\boldsymbol{a}_{\delta}^{*}\right)\right\} \\\underset{\left(b\right)}{\leq}\frac{1}{2}\sum_{t=1}^{T}\sum_{s\in\mathcal{S}_{t}}\mathbb{E}\left\{ \left\Vert \boldsymbol{x}_{s_{-}}-\boldsymbol{a}_{\delta}^{*}\right\Vert ^{2}-\left\Vert \boldsymbol{x}_{s_{+}}-\boldsymbol{a}_{\delta}^{*}\right\Vert ^{2}\right\} +\frac{1}{2}\sum_{t=1}^{T}\sum_{s\in\mathcal{S}_{t}}\eta_{s}^{2}\mathbb{E}\left\{ \left\Vert \boldsymbol{g}_{s}\right\Vert ^{2}\right\} -\\\sum_{t=1}^{T}\sum_{s\in\mathcal{S}_{t}}\eta_{s}\mathbb{E}\left\{ \left\langle \sum_{r=s}^{t-1}\sum_{q\in\mathcal{S}_{r}}\left(\boldsymbol{x}_{q_{+}}-\boldsymbol{x}_{q_{-}}\right)+\sum_{q\in\mathcal{S}_{t},q<s}\left(\boldsymbol{x}_{q_{+}}-\boldsymbol{x}_{q_{-}}\right),\mathbb{E}\left\{ \nabla h_{s}\left(\boldsymbol{x}_{s}\right)\,|\,\mathcal{F}_{s_{-}}\right\} \right\rangle \right\} \\\underset{\left(c\right)}{\leq}\frac{\left|\mathcal{K}\right|^{2}}{2}+\frac{n^{2}}{2\delta^{2}}\sum_{t\notin\mathcal{M}}\eta_{t}^{2}+L\sum_{t=1}^{T}\sum_{s\in\mathcal{S}_{t}}\eta_{s}\left(\sum_{r=s}^{t-1}\sum_{q\in\mathcal{S}_{r}}\mathbb{E}\left\{ \left\Vert \boldsymbol{x}_{q_{+}}-\boldsymbol{x}_{q_{-}}\right\Vert \right\} +\sum_{q\in\mathcal{S}_{t},q<s}\mathbb{E}\left\{ \left\Vert \boldsymbol{x}_{q_{+}}-\boldsymbol{x}_{q_{-}}\right\Vert \right\} \right)\label{eq:88a}
\end{multline}
where (a) uses $\mathbb{E}\left\{ \hat{l}_{s}\left(\boldsymbol{x}_{s}\right)\right\} =\mathbb{E}\left\{ h_{s}\left(\boldsymbol{x}_{s}\right)\right\}$  and $\mathbb{E}\left\{ \hat{l}_{s}\left(\boldsymbol{a}_{\delta}^{*}\right)\right\} =\mathbb{E}\left\{ h_{s}\left(\boldsymbol{a}_{\delta}^{*}\right)\right\}$ as explained below \eqref{eq:82} (since with an oblivious adversary, $\boldsymbol{a}_{\delta}^{*}$ is not random),  (b) uses \eqref{eq:72a} and (c) uses the telescopic sum with
$\left\Vert \boldsymbol{x}_{0}-\boldsymbol{a}_{\delta}^{*}\right\Vert ^{2}-\left\Vert \boldsymbol{x}_{T}-\boldsymbol{a}_{\delta}^{*}\right\Vert ^{2}\leq\left|\mathcal{K}\right|^{2}$,
$\left\Vert \boldsymbol{g}_{s}\right\Vert \leq\frac{n}{\delta}$,
and Cauchy-Schwarz with \eqref{eq:79b}. Finally, we use \eqref{eq:82a}
to bound the last term in \eqref{eq:88a}: 
\begin{multline}
L\sum_{t=1}^{T}\sum_{s\in\mathcal{S}_{t}}\eta_{s}\left(\sum_{r=s}^{t-1}\sum_{q\in\mathcal{S}_{r}}\mathbb{E}\left\{ \left\Vert \boldsymbol{x}_{q_{+}}-\boldsymbol{x}_{q_{-}}\right\Vert \right\} +\sum_{q\in\mathcal{S}_{t},q<s}\mathbb{E}\left\{ \left\Vert \boldsymbol{x}_{q_{+}}-\boldsymbol{x}_{q_{-}}\right\Vert \right\} \right)\\
\leq\frac{Ln}{\delta}\sum_{t=1}^{T}\sum_{s\in\mathcal{S}_{t}}\eta_{s}\left(\sum_{r=s}^{t-1}\sum_{q\in\mathcal{S}_{r}}\eta_{q}+\sum_{q\in\mathcal{S}_{t},q<s}\eta_{q}\right)\underset{\left(a\right)}{\leq}2L\frac{n}{\delta}\sum_{t\notin\mathcal{M}}\eta_{t}^{2}d_{t}\label{eq:86-1}
\end{multline}
where (a) uses Lemma \ref{Counting Delays Lemma}. We conclude by
applying \eqref{eq:86-1} on \eqref{eq:88a} and adding \eqref{eq:80}
and \eqref{eq:70}.

\newpage
\section{Proof of Lemma \ref{lem:NoErgodicEXP3}: Weighted-Regret of EXP3
with Delays }
Recall that $\mathcal{M^{*}}$ is the set of missing or discarded samples. Let $s_{-},s_{+}$
as the step a moment before and a moment after the algorithm uses the feedback
from round $s$, which updates the mixed action from $\boldsymbol{p}_{s_{-}}$
to $\boldsymbol{p}_{s_{+}}$. Both $s_{-}$
and $s_{+}$ are algorithm update steps that take place in round $t$ of the game if $s\in\mathcal{S}_{t}$. Let $s_{T}$ be the last
feedback to be updated. 

We begin by the standard EXP3 analysis for arbitrary $\tilde{\boldsymbol{l}}_{s}$
\citep{lattimore2020bandit}, but with careful consideration to both
the interleaved arrivals and the weight $\eta_{s}$. Recall that $\tilde{L}_{t}^{\left(i\right)}=\sum_{t\notin\mathcal{M^{*}}}\eta_{t}\frac{l_{t}^{\left(i\right)}1_{\left\{ a_{t}=i\right\} }}{p_{t}^{\left(i\right)}+\gamma_{t}}$, as defined in Algorithm \ref{alg:EXP3}. Define $\Phi\left(t\right)=-\log \left(\sum_{i=1}^{K}e^{-\tilde{L}_{t}^{\left(i\right)}}\right)$
and $\tilde{\boldsymbol{l}}_{t}=\left(0,...,\frac{l_{t}^{\left(a_{t}\right)}}{p_{t}^{\left(a_{t}\right)}+\gamma_{t}},...,0\right)$.
Then
\begin{align*}
\Phi\left(s_{+}\right)-\Phi\left(s_{-}\right)=-\log \left(\frac{\sum_{i=1}^{K}e^{-\tilde{L}_{s_{-}}^{\left(i\right)}}e^{-\eta_{s}\tilde{l}_{s}^{\left(i\right)}}}{\sum_{j=1}^{K}e^{-\tilde{L}_{s_{-}}^{\left(j\right)}}}\right)&=-\log \left(\sum_{i=1}^{K}p_{s_{-}}^{\left(i\right)}e^{-\eta_{s}\tilde{l}_{s}^{\left(i\right)}}\right)\\
&\underset{\left(a\right)}{\geq}-\log \left(\sum_{i=1}^{K}p_{s_{-}}^{\left(i\right)}\left(1-\eta_{s}\tilde{l}_{s}^{\left(i\right)}+\frac{1}{2}\eta_{s}^{2}\left(\tilde{l}_{s}^{\left(i\right)}\right)^{2}\right)\right)\\&=-\log \left(1-\sum_{i=1}^{K}p_{s_{-}}^{\left(i\right)}\left(\eta_{s}\tilde{l}_{s}^{\left(i\right)}-\frac{1}{2}\eta_{s}^{2}\left(\tilde{l}_{s}^{\left(i\right)}\right)^{2}\right)\right)\\
&\underset{\left(b\right)}{\geq}\eta_{s}\sum_{i=1}^{K}p_{s_{-}}^{\left(i\right)}\tilde{l}_{s}^{\left(i\right)}-\frac{\eta_{s}^{2}}{2}\sum_{i=1}^{K}p_{s_{-}}^{\left(i\right)}\left(\tilde{l}_{s}^{\left(i\right)}\right)^{2}\numberthis\label{eq:85a}
\end{align*}
where (a) is $e^{-x}\leq1-x+\frac{1}{2}x^{2}$ and (b) is $\log \left(1-x\right)\leq-x$.
Hence, iterating \eqref{eq:85a} over $s$ yields
\begin{equation}
\Phi\left(s_{T}^{+}\right)-\Phi\left(1\right)=\sum_{t=1}^{T}\sum_{s\in\mathcal{S}_{t}}\left(\Phi\left(s_{+}\right)-\Phi\left(s_{-}\right)\right)\geq\sum_{t=1}^{T}\sum_{s\in\mathcal{S}_{t}}\eta_{s}\sum_{i=1}^{K}p_{s_{-}}^{\left(i\right)}\tilde{l}_{s}^{\left(i\right)}-\frac{1}{2}\sum_{t=1}^{T}\sum_{s\in\mathcal{S}_{t}}\eta_{s}^{2}\sum_{i=1}^{K}p_{s_{-}}^{\left(i\right)}\left(\tilde{l}_{s}^{\left(i\right)}\right)^{2}.\label{eq:86a}
\end{equation}
Next we upper bound $\Phi\left(s_{T}^{+}\right)-\Phi\left(1\right)$.
We have for $i^{*}\triangleq{\arg\underset{i}\min}\sum_{t=1}^{T}\eta_{t}l_{t}^{\left(i\right)}$ that
\begin{equation}
\Phi\left(s_{T}^{+}\right)-\Phi\left(1\right)=-\log \left(\sum_{j=1}^{K}e^{-\tilde{L}_{s_{T}^{+}}^{\left(j\right)}}\right)+\log  K\underset{\left(a\right)}{\leq}\tilde{L}_{s_{T}^{+}}^{\left(i^{*}\right)}+\log  K=\sum_{t=1}^{T}\sum_{s\in\mathcal{S}_{t}}\eta_{s}\tilde{l}_{s}^{\left(i^{*}\right)}+\log  K\label{eq:87}
\end{equation}
where (a) omits positive terms from $\sum_{j=1}^{K}e^{-\tilde{L}_{s_{T}^{+}}^{\left(j\right)}}$.
We conclude from \eqref{eq:86a} and \eqref{eq:87} that 
\begin{equation}
\sum_{t=1}^{T}\sum_{s\in\mathcal{S}_{t}}\eta_{s}\sum_{i=1}^{K}p_{s_{-}}^{\left(i\right)}\tilde{l}_{s}^{\left(i\right)}-\sum_{t=1}^{T}\sum_{s\in\mathcal{S}_{t}}\eta_{s}\tilde{l}_{s}^{\left(i^{*}\right)}\leq\log  K+\frac{1}{2}\sum_{t=1}^{T}\sum_{s\in\mathcal{S}_{t}}\eta_{s}^{2}\sum_{i=1}^{K}p_{s_{-}}^{\left(i\right)}\left(\tilde{l}_{s}^{\left(i\right)}\right)^{2}.\label{eq:88}
\end{equation}
Now observe that
\begin{align*}
\sum_{i=1}^{K}p_{s_{-}}^{\left(i\right)}\tilde{l}_{s}^{\left(i\right)}=\sum_{i=1}^{K}p_{s_{-}}^{\left(i\right)}\frac{l_{s}^{\left(i\right)}1_{\left\{ a_{s}=i\right\} }}{p_{s}^{\left(i\right)}+\gamma_{s}}&=\sum_{i=1}^{K}\left(p_{s}^{\left(i\right)}+\gamma_{s}\right)\frac{l_{s}^{\left(i\right)}1_{\left\{ a_{s}=i\right\} }}{p_{s}^{\left(i\right)}+\gamma_{s}}-\sum_{i=1}^{K}\left(p_{s}^{\left(i\right)}+\gamma_{s}-p_{s_{-}}^{\left(i\right)}\right)\frac{l_{s}^{\left(i\right)}1_{\left\{ a_{s}=i\right\} }}{p_{s}^{\left(i\right)}+\gamma_{s}}\\
&=l_{s}^{\left(a_{s}\right)}-\gamma_{s}\sum_{i=1}^{K}\tilde{l}_{s}^{\left(i\right)}+\sum_{i=1}^{K}\left(p_{s_{-}}^{\left(i\right)}-p_{s}^{\left(i\right)}\right)\tilde{l}_{s}^{\left(i\right)}.\numberthis\label{eq:89}
\end{align*}
Using $\frac{p_{s_{-}}^{\left(i\right)}}{p_{s}^{\left(i\right)}}\leq e^{2}$
from Lemma \ref{Multiplicative Lemma}, we obtain
\begin{equation}
\sum_{i=1}^{K}p_{s_{-}}^{\left(i\right)}\left(\tilde{l}_{s}^{\left(i\right)}\right)^{2}=\sum_{i=1}^{K}p_{s_{-}}^{\left(i\right)}\left(\frac{l_{s}^{\left(i\right)}1_{\left\{ a_{s}=i\right\} }}{p_{s}^{\left(i\right)}+\gamma_{s}}\right)^{2}\leq e^{2}\sum_{i=1}^{K}\frac{p_{s}^{\left(i\right)}}{p_{s}^{\left(i\right)}+\gamma_{s}}\frac{l_{s}^{\left(i\right)}1_{\left\{ a_{s}=i\right\} }}{p_{s}^{\left(i\right)}+\gamma_{s}}\leq e^{2}\sum_{i=1}^{K}\tilde{l}_{s}^{\left(i\right)}.\label{eq:90}
\end{equation}
Taking the weighted sum of \eqref{eq:89} over $t\notin\mathcal{M^{*}}$
and subtracting $\sum_{t=1}^{T}\sum_{s\in\mathcal{S}_{t}}\eta_{s}\tilde{l}_{s}^{\left(i^{*}\right)}$
from both sides:
\begin{align*}
\sum_{t=1}^{T}\sum_{s\in\mathcal{S}_{t}}\eta_{s}l_{s}^{\left(a_{s}\right)}-\sum_{t=1}^{T}\sum_{s\in\mathcal{S}_{t}}\gamma_{s}\eta_{s}\sum_{i=1}^{K}\tilde{l}_{s}^{\left(i\right)}&+\sum_{t=1}^{T}\sum_{s\in\mathcal{S}_{t}}\eta_{s}\sum_{i=1}^{K}\left(p_{s_{-}}^{\left(i\right)}-p_{s}^{\left(i\right)}\right)\tilde{l}_{s}^{\left(i\right)}-\sum_{t=1}^{T}\sum_{s\in\mathcal{S}_{t}}\eta_{s}\tilde{l}_{s}^{\left(i^{*}\right)}\\
&=\sum_{t=1}^{T}\sum_{s\in\mathcal{S}_{t}}\eta_{s}\sum_{i=1}^{K}p_{s_{-}}^{\left(i\right)}\tilde{l}_{s}^{\left(i\right)}-\sum_{t=1}^{T}\sum_{s\in\mathcal{S}_{t}}\eta_{s}\tilde{l}_{s}^{\left(i^{*}\right)}\\&\underset{\left(a\right)}{\leq}\log  K+\frac{1}{2}\sum_{t=1}^{T}\sum_{s\in\mathcal{S}_{t}}\eta_{s}^{2}\sum_{i=1}^{K}p_{s_{-}}^{\left(i\right)}\left(\tilde{l}_{s}^{\left(i\right)}\right)^{2}\\
&\underset{\left(b\right)}{\leq}\log  K+\frac{1}{2}\sum_{t=1}^{T}\sum_{s\in\mathcal{S}_{t}}e^{2}\eta_{s}^{2}\sum_{i=1}^{K}\tilde{l}_{s}^{\left(i\right)}\numberthis\label{eq:91}
\end{align*}
where (a) uses \eqref{eq:88} and (b) uses \eqref{eq:90}. Rearranging
\eqref{eq:91} and subtracting $\sum_{t=1}^{T}\sum_{s\in\mathcal{S}_{t}}\eta_{s}l_{s}^{\left(i^{*}\right)}$ from both sides gives
\begin{multline}
\sum_{t=1}^{T}\sum_{s\in\mathcal{S}_{t}}\eta_{s}\left(l_{s}^{\left(a_{s}\right)}-l_{s}^{\left(i^{*}\right)}\right)
\leq\\\underbrace{\sum_{t=1}^{T}\sum_{s\in\mathcal{S}_{t}}\eta_{s}\sum_{i=1}^{K}\left(p_{s}^{\left(i\right)}-p_{s_{-}}^{\left(i\right)}\right)\tilde{l}_{s}^{\left(i\right)}}_{A}+\log  K+\sum_{t=1}^{T}\sum_{s\in\mathcal{S}_{t}}\eta_{s}\left(\tilde{l}_{s}^{\left(i^{*}\right)}-l_{s}^{\left(i^{*}\right)}\right)+\sum_{t=1}^{T}\sum_{s\in\mathcal{S}_{t}}\left(\gamma_{s}\eta_{s}+\frac{e^{2}}{2}\eta_{s}^{2}\right)\sum_{i=1}^{K}\tilde{l}_{s}^{\left(i\right)}.\label{eq:92}
\end{multline}

\subsection{Adaptive Adversary with $\gamma_{t}=\eta_{t}$}

\subsubsection{Taking the Expectation}

Define $W_{1}^{\left(i\right)}=\sum_{t=1}^{T}\sum_{s\in\mathcal{S}_{t}}\eta_{s}\left(\tilde{l}_{s}^{\left(i\right)}-l_{s}^{\left(i\right)}\right)-\log  K$.
From Lemma \ref{lem:HighProbability} with $\delta\leftarrow\frac{\delta}{K}$ and $\alpha_{s}^{\left(i\right)}=\eta_{s}$
and $\alpha_{s}^{\left(j\right)}=0$ for all $j\neq i$ we get by
using the union bound that
\begin{equation}
\mathbb{P}\left(\underset{i}{\max}W_{1}^{\left(i\right)}\geq\log \frac{1}{\delta}\right)\leq K\mathbb{P}\left(\sum_{t=1}^{T}\sum_{s\in\mathcal{S}_{t}}\eta_{s}\left(\tilde{l}_{s}^{\left(i\right)}-l_{s}^{\left(i\right)}\right)\geq\log \frac{K}{\delta}\right)\leq\delta\label{eq:93a}
\end{equation}
 so by substituting $x=\log \frac{1}{\delta}$ so $dx=-\frac{d\delta}{\delta}$
\begin{equation}
\mathbb{E}\left\{ \underset{i}{\max}W_{1}^{\left(i\right)}\right\} \leq\intop_{0}^{\infty}\mathbb{P}\left(\underset{i}{\max}W_{1}^{\left(i\right)}\geq x\right)dx=\intop_{0}^{1}\frac{1}{\delta}\mathbb{P}\left(\underset{i}{\max}W_{1}^{\left(i\right)}\geq\log \frac{1}{\delta}\right)d\delta\leq1.\label{eq:94a}
\end{equation}
Define $W_{2}=\sum_{t=1}^{T}\sum_{s\in\mathcal{S}_{t}}\left(\gamma_{s}\eta_{s}+\frac{e^{2}}{2}\eta_{s}^{2}\right)\sum_{i=1}^{K}\left(\tilde{l}_{s}^{\left(i\right)}-l_{s}^{\left(i\right)}\right)$.
From Lemma \ref{lem:HighProbability} with $\alpha_{s}^{\left(i\right)}=\gamma_{s}\eta_{s}+\frac{e^{2}}{2}\eta_{s}^{2}\leq2\eta_{s}$
for all $i$ we get $\mathbb{P}\left(W_{2}\geq\log \frac{1}{\delta}\right)\leq\delta$,
so using \eqref{eq:94a} on $W_{2}$ we obtain $\mathbb{E}\left\{ W_{2}\right\} \leq1$. 

From $\mathbb{E}\left\{ \underset{i}{\max}W_{1}^{\left(i\right)}\right\} \leq1$
and $\mathbb{E}\left\{ W_{2}\right\} \leq1$ we conclude that
\begin{multline}
\mathbb{E}\left\{ \underset{i}{\max}\sum_{t=1}^{T}\sum_{s\in\mathcal{S}_{t}}\eta_{s}\left(\tilde{l}_{s}^{\left(i\right)}-l_{s}^{\left(i\right)}\right)+\sum_{t=1}^{T}\sum_{s\in\mathcal{S}_{t}}\left(\gamma_{s}\eta_{s}+\frac{e^{2}}{2}\eta_{s}^{2}\right)\sum_{i=1}^{K}\tilde{l}_{s}^{\left(i\right)}\right\} \leq\\\log  K+2+\sum_{t=1}^{T}\sum_{s\in\mathcal{S}_{t}}\left(\gamma_{s}\eta_{s}+\frac{e^{2}}{2}\eta_{s}^{2}\right)\sum_{i=1}^{K}l_{s}^{\left(i\right)}\leq\log  K+2+K\sum_{t=1}^{T}\sum_{s\in\mathcal{S}_{t}}\left(\gamma_{s}\eta_{s}+\frac{e^{2}}{2}\eta_{s}^{2}\right).\label{eq:94}
\end{multline}

\subsubsection{The Effect of Delays}

Next we bound the $A$ term in \eqref{eq:92}, which quantifies the effect of the delays. Let $s\in \mathcal{S}_{t}$, and let $q$ be a round for which the feedback is used after or at round $s$, but before the feedback from round $s$ is used. Define
$h_{i}\left(\boldsymbol{\tilde{L}}_{q_{-}}\right)\triangleq p_{q_{-}}^{\left(i\right)}=\frac{e^{-\tilde{L}_{q_{-}}^{\left(i\right)}}}{\sum_{j=1}^{K}e^{-\tilde{L}_{q_{-}}^{\left(j\right)}}}$, so $p_{q^{+}}^{\left(i\right)}=h_{i}\left(\boldsymbol{\tilde{L}}_{q_{-}}+\eta_{q}\tilde{\boldsymbol{l}}_{q}\right)$.
Using Lemma \ref{Lipchitz Lemma} with $\boldsymbol{x}=\boldsymbol{\tilde{L}}_{q_{-}}$
and $\Delta=\eta_{q}\tilde{\boldsymbol{l}}_{q}$ (so $\boldsymbol{h}\left(\boldsymbol{x}\right)=\boldsymbol{p}_{q_{-}}$)
yields
\begin{align*}
\left\Vert \boldsymbol{p}_{q_{-}}-\boldsymbol{p}_{q_{+}}\right\Vert _{1}&\leq2\eta_{q}\sum_{i=1}^{K}p_{q_{-}}^{\left(i\right)}\tilde{l}_{q}^{\left(i\right)}=2\eta_{q}\sum_{i=1}^{K}p_{q_{-}}^{\left(i\right)}\frac{l_{q}^{\left(i\right)}1_{\left\{ a_{q}=i\right\} }}{p_{q}^{\left(i\right)}+\gamma_{q}}\\&\underset{\left(a\right)}{\leq}2e^{2}\eta_{q}\sum_{i=1}^{K}p_{q}^{\left(i\right)}\frac{l_{q}^{\left(i\right)}1_{\left\{ a_{q}=i\right\} }}{p_{q}^{\left(i\right)}+\gamma_{q}}\leq2e^{2}\eta_{q}\sum_{i=1}^{K}l_{q}^{\left(i\right)}1_{\left\{ a_{q}=i\right\} }=2e^{2}\eta_{q}l_{q}^{\left(q_{q}\right)}\leq2e^{2}\eta_{q}\numberthis\label{eq:96}
\end{align*}
where (a) uses $\frac{p_{q_{-}}^{\left(i\right)}}{p_{q}^{\left(i\right)}}\leq e^{2}$
from Lemma \ref{Multiplicative Lemma}. Hence
\begin{equation}
\left\langle \boldsymbol{\tilde{l}}_{s},\boldsymbol{p}_{q_{-}}-\boldsymbol{p}_{q_{+}}\right\rangle =\sum_{i=1}^{K}\left(p_{q_{-}}^{\left(i\right)}-p_{q_{+}}^{\left(i\right)}\right)\frac{l_{s}^{\left(i\right)}1_{\left\{ a_{s}=i\right\} }}{p_{s}^{\left(i\right)}+\gamma_{s}}\underset{\left(a\right)}{\leq}2e^{2}\eta_{q}\sum_{i=1}^{K}\frac{l_{s}^{\left(i\right)}1_{\left\{ a_{s}=i\right\} }}{p_{s}^{\left(i\right)}+\gamma_{s}}\leq2e^{2}\eta_{q}\sum_{i=1}^{K}\frac{1_{\left\{ a_{s}=i\right\} }}{p_{s}^{\left(i\right)}+\gamma_{s}}\label{eq:97}
\end{equation}
where (a)  follows since $p_{q_{-}}^{\left(i\right)}-p_{q_{+}}^{\left(i\right)}\leq\left|p_{q_{-}}^{\left(i\right)}-p_{q_{+}}^{\left(i\right)}\right|\leq\left\Vert \boldsymbol{p}_{q_{-}}-\boldsymbol{p}_{q_{+}}\right\Vert _{1}\leq2e^{2}\eta_{q}$. Using \eqref{eq:97} we can write
\begin{equation}
\mathbb{E}\left\{ \left\langle \boldsymbol{\tilde{l}}_{s},\boldsymbol{p}_{q_{-}}-\boldsymbol{p}_{q_{+}}\right\rangle \,|\,\mathcal{F}_{s}\right\} \leq2e^{2}\eta_{q}\sum_{i=1}^{K}\mathbb{E}\left\{ \frac{1_{\left\{ a_{s}=i\right\} }}{p_{s}^{\left(i\right)}+\gamma_{s}}\,|\,\mathcal{F}_{s}\right\} \underset{\left(a\right)}{=}2e^{2}\eta_{q}\sum_{i=1}^{K}\frac{p_{s}^{\left(i\right)}}{p_{s}^{\left(i\right)}+\gamma_{s}}\leq2e^{2}\eta_{q}K\label{eq:98a}
\end{equation}
where (a) uses that $p_{s}^{\left(i\right)}$ is $\mathcal{F}_{s}$-measurable and that $a_{s}=i$ with probability $p_{s}^{\left(i\right)}$. Hence the A term in \eqref{eq:92} can be bounded as
\begin{align*}
&
\mathbb{E}\left\{ \sum_{t=1}^{T}\sum_{s\in\mathcal{S}_{t}}\eta_{s}\left\langle \boldsymbol{\tilde{l}}_{s},\boldsymbol{p}_{s}-\boldsymbol{p}_{s_{-}}\right\rangle \right\} \\&=\mathbb{E}\left\{ \sum_{t=1}^{T}\sum_{s\in\mathcal{S}_{t}}\eta_{s}\left(\left\langle \boldsymbol{\tilde{l}}_{s},\boldsymbol{p}_{t}-\boldsymbol{p}_{s_{-}}\right\rangle +\sum_{r=s}^{t-1}\left\langle \boldsymbol{\tilde{l}}_{s},\boldsymbol{p}_{r}-\boldsymbol{p}_{r+1}\right\rangle \right)\right\} \\&=\mathbb{E}\left\{ \sum_{t=1}^{T}\sum_{s\in\mathcal{S}_{t}}\eta_{s}\left(\sum_{q\in\mathcal{S}_{t},q<s}\left\langle \boldsymbol{\tilde{l}}_{s},\boldsymbol{p}_{q_{-}}-\boldsymbol{p}_{q_{+}}\right\rangle +\sum_{r=s}^{t-1}\sum_{q\in\mathcal{S}_{r}}\left\langle \boldsymbol{\tilde{l}}_{s},\boldsymbol{p}_{q_{-}}-\boldsymbol{p}_{q_{+}}\right\rangle \right)\right\} \\&=\sum_{t=1}^{T}\sum_{s\in\mathcal{S}_{t}}\eta_{s}\left(\sum_{q\in\mathcal{S}_{t},q<s}\mathbb{E}\left\{ \left\langle \boldsymbol{\tilde{l}}_{s},\boldsymbol{p}_{q_{-}}-\boldsymbol{p}_{q_{+}}\right\rangle \right\} +\sum_{r=s}^{t-1}\sum_{q\in\mathcal{S}_{r}}\mathbb{E}\left\{ \left\langle \boldsymbol{\tilde{l}}_{s},\boldsymbol{p}_{q_{-}}-\boldsymbol{p}_{q_{+}}\right\rangle \right\} \right)\\&\underset{\left(a\right)}{\leq}2e^{2}K\sum_{t=1}^{T}\sum_{s\in\mathcal{S}_{t}}\eta_{s}\left(\sum_{q\in\mathcal{S}_{t},q<s}\eta_{q}+\sum_{r=s}^{t-1}\sum_{q\in\mathcal{S}_{r}}\eta_{q}\right)\underset{\left(b\right)}{\leq}4e^{2}K\sum_{t\notin\mathcal{M^{*}}}\eta_{t}^{2}d_{t}\numberthis\label{eq:99}
\end{align*}
where (a) uses \eqref{eq:98a} with the tower rule, and (b) uses Lemma \ref{Counting Delays Lemma}.

\subsubsection{Concluding the Proof}
We conclude that for $i^{*}\triangleq{\arg\underset{i}\min}\sum_{t=1}^{T}\eta_{t}l_{t}^{\left(i\right)}$:
\begin{align*}
\mathbb{E}\left\{ \sum_{t=1}^{T}\eta_{t}\left(l_{t}^{\left(a_{t}\right)}-l_{t}^{\left(i^{*}\right)}\right)\right\} &\underset{\left(a\right)}{\leq}\mathbb{E}\left\{ \sum_{t=1}^{T}\sum_{s\in\mathcal{S}_{t}}\eta_{s}\left(l_{s}^{\left(a_{s}\right)}-l_{s}^{\left(i^{*}\right)}\right)\right\} +\sum_{t\in\mathcal{\mathcal{M^{*}}}}\eta_{t}\\&\underset{\left(b\right)}{\leq}\mathbb{E}\left\{ \sum_{t=1}^{T}\sum_{s\in\mathcal{S}_{t}}\eta_{s}\sum_{i=1}^{K}\left(p_{s}^{\left(i\right)}-p_{s_{-}}^{\left(i\right)}\right)\tilde{l}_{s}^{\left(i\right)}\right\} +\log  K\\&+\mathbb{E}\left\{ \sum_{t=1}^{T}\sum_{s\in\mathcal{S}_{t}}\left(\eta_{s}\left(\tilde{l}_{s}^{\left(i^{*}\right)}-l_{s}^{\left(i^{*}\right)}\right)+\left(\gamma_{s}\eta_{s}+\frac{e^{2}}{2}\eta_{s}^{2}\right)\sum_{i=1}^{K}\tilde{l}_{s}^{\left(i\right)}\right)\right\} +\sum_{t\in\mathcal{\mathcal{M^{*}}}}\eta_{t}\\&\underset{\left(c\right)}{\leq}4e^{2}K\sum_{t\notin\mathcal{M^{*}}}\eta_{t}^{2}d_{t}+2+\left(1+\frac{e^{2}}{2}\right)K\sum_{t=1}^{T}\sum_{s\in\mathcal{S}_{t}}\eta_{s}^{2}+2\log K+\sum_{t\in\mathcal{\mathcal{M^{*}}}}\eta_{t}\numberthis\label{eq:100}
\end{align*}
where (a) uses that $0\leq l_{t}^{\left(i\right)}\leq1$ for every $i$ and $t$, (b) is \eqref{eq:92} and (c) is \eqref{eq:94} and  \eqref{eq:99} for $\gamma_{s}=\eta_{s}$.
\subsection{Oblivious Adversary with $\gamma_{t}=0$}

\subsubsection{Taking the Expectation}

With an oblivious adversary $\mathbb{E}\left\{ \tilde{l}_{s}^{\left(i\right)}\right\} =l_{s}^{\left(i\right)}\mathbb{E}\left\{ \frac{1_{\left\{ a_{s}=i\right\} }}{p_{s}^{\left(i\right)}}\right\} =l_{s}^{\left(i\right)}$ for each $s$ and $i$, since $l_{s}^{\left(i\right)}$ is not random. Then for any $i$, in particular $i^{*}\triangleq{\arg\underset{i}\min}\sum_{t=1}^{T}\eta_{t}l_{t}^{\left(i\right)}$, we have
\begin{equation}
\mathbb{E}\left\{ \sum_{t=1}^{T}\sum_{s\in\mathcal{S}_{t}}\eta_{s}\left(\tilde{l}_{s}^{\left(i^{*}\right)}-l_{s}^{\left(i^{*}\right)}\right)+\frac{e^{2}}{2}\sum_{t=1}^{T}\sum_{s\in\mathcal{S}_{t}}\eta_{s}^{2}\sum_{i=1}^{K}\tilde{l}_{s}^{\left(i\right)}\right\} =\frac{e^{2}}{2}\sum_{t=1}^{T}\sum_{s\in\mathcal{S}_{t}}\eta_{s}^{2}\sum_{i=1}^{K}l_{s}^{\left(i\right)}\leq\frac{e^{2}}{2}K\sum_{t=1}^{T}\sum_{s\in\mathcal{S}_{t}}\eta_{s}^{2}.\label{eq:101}
\end{equation}

\subsubsection{The Effect of Delays}

Let $s\in \mathcal{S}_{t}$, and let $q$ be a round for which the feedback is used after or at round $s$, but before the feedback from round $s$ is used. Using Lemma \ref{Lipchitz Lemma} with $\boldsymbol{x}=\boldsymbol{\tilde{L}}_{q_{-}}$
and $\Delta=\eta_{q}\tilde{\boldsymbol{l}}_{q}$, so $\boldsymbol{h}\left(\boldsymbol{x}\right)=\boldsymbol{p}_{q_{-}}$
yields
\begin{multline}
\mathbb{E}\left\{ \left\Vert \boldsymbol{p}_{q_{-}}-\boldsymbol{p}_{q_{+}}\right\Vert _{1}\,|\,\mathcal{F}_{q_{-}}\right\} \leq2\eta_{q}\mathbb{E}\left\{ \sum_{i=1}^{K}p_{q_{-}}^{\left(i\right)}\tilde{l}_{q}^{\left(i\right)}\,|\,\mathcal{F}_{q_{-}}\right\} \\\underset{\left(a\right)}{=}2\eta_{q}\sum_{i=1}^{K}p_{q_{-}}^{\left(i\right)}\mathbb{E}\left\{ \tilde{l}_{q}^{\left(i\right)}\,|\,\mathcal{F}_{q_{-}}\right\} \underset{\left(b\right)}{=}2\eta_{q}\sum_{i=1}^{K}p_{q_{-}}^{\left(i\right)}l_{q}^{\left(i\right)}\leq2\eta_{q}\sum_{i=1}^{K}p_{q_{-}}^{\left(i\right)}=2\eta_{q}\label{eq:102}
\end{multline}
where (a) uses that  $p_{q_{-}}^{\left(i\right)}$ is $\mathcal{F}_{q_{-}}$-measurable
and (b) uses that $p_{q}^{\left(i\right)}$ is $\mathcal{F}_{q_{-}}$-measurable (since
$q<q_{-}$) and that $\tilde{l}_{q}^{\left(i\right)}$
is $\frac{l_{q}^{\left(i\right)}}{p_{q}^{\left(i\right)}}$ with probability
$p_{q}^{\left(i\right)}$ and zero otherwise. Note that $a_{q}$ given
$\mathcal{F}_{q}$ is independent of $\mathcal{F}_{q_{-}}$ since
by definition the feedback from $a_{q}$ was not received until round
$q_{-}$. This is unique to the oblivious adversary case when
$\boldsymbol{l}_{r}$ for $r>s$ is not a random variable that depends on
$a_{q}$, which the adversary observes already at the end of round
$q$.  Then for every $q\in\mathcal{S}_{r}$ for $r<t$ or $q\in\mathcal{S}_{t}$ such that $q<s$ we have
\begin{align*}
\mathbb{E}\left\{ \left\langle \tilde{\boldsymbol{l}}_{s},\boldsymbol{p}_{q_{-}}-\boldsymbol{p}_{q_{+}}\right\rangle \right\} &=\mathbb{E}\left\{ \mathbb{E}\left\{ \sum_{i=1}^{K}\left(p_{q_{-}}^{\left(i\right)}-p_{q_{+}}^{\left(i\right)}\right)\frac{l_{s}^{\left(i\right)}1_{\left\{ a_{s}=i\right\} }}{p_{s}^{\left(i\right)}}\,|\,\mathcal{F}_{s_{-}}\right\} \right\} \\&=\mathbb{E}\left\{ \sum_{i=1}^{K}\left(p_{q_{-}}^{\left(i\right)}-p_{q_{+}}^{\left(i\right)}\right)\mathbb{E}\left\{ \frac{l_{s}^{\left(i\right)}1_{\left\{ a_{s}=i\right\} }}{p_{s}^{\left(i\right)}}\,|\,\mathcal{F}_{s_{-}}\right\} \right\} \\&=\mathbb{E}\left\{ \sum_{i=1}^{K}\left(p_{q_{-}}^{\left(i\right)}-p_{q_{+}}^{\left(i\right)}\right)l_{s}^{\left(i\right)}\right\} \\&=\mathbb{E}\left\{ \left\langle \boldsymbol{l}_{s},\boldsymbol{p}_{q_{-}}-\boldsymbol{p}_{q_{+}}\right\rangle \right\} \underset{\left(a\right)}{\leq}\mathbb{E}\left\{ \left\Vert \boldsymbol{l}_{s}\right\Vert _{\infty}\left\Vert \boldsymbol{p}_{q_{-}}-\boldsymbol{p}_{q_{+}}\right\Vert _{1}\right\} \\&\underset{\left(b\right)}{\leq}\mathbb{E}\left\{ \left\Vert \boldsymbol{p}_{q_{-}}-\boldsymbol{p}_{q_{+}}\right\Vert _{1}\right\} \underset{\left(c\right)}{\leq}2\eta_{q}\numberthis\label{eq:103a}
\end{align*}
where (a) is H\"{o}lder's inequality, (b) uses $0 \leq l_{t}^{\left(i\right)}\leq1$
and (c) uses \eqref{eq:102} and the tower rule.  Therefore
\begin{align*}
&
\mathbb{E}\left\{ \sum_{t=1}^{T}\sum_{s\in\mathcal{S}_{t}}\eta_{s}\left\langle \tilde{\boldsymbol{l}}_{s},\boldsymbol{p}_{s}-\boldsymbol{p}_{s_{-}}\right\rangle \right\} \\&=\mathbb{E}\left\{ \sum_{t=1}^{T}\sum_{s\in\mathcal{S}_{t}}\eta_{s}\left(\left\langle \tilde{\boldsymbol{l}}_{s},\boldsymbol{p}_{t}-\boldsymbol{p}_{s_{-}}\right\rangle +\sum_{r=s}^{t-1}\left\langle \tilde{\boldsymbol{l}}_{s},\boldsymbol{p}_{r}-\boldsymbol{p}_{r+1}\right\rangle \right)\right\} \\&=\sum_{t=1}^{T}\sum_{s\in\mathcal{S}_{t}}\eta_{s}\left(\sum_{q\in\mathcal{S}_{t},q<s}\mathbb{E}\left\{ \left\langle \tilde{\boldsymbol{l}}_{s},\boldsymbol{p}_{q_{-}}-\boldsymbol{p}_{q_{+}}\right\rangle \right\} +\sum_{r=s}^{t-1}\sum_{q\in\mathcal{S}_{r}}\mathbb{E}\left\{ \left\langle \tilde{\boldsymbol{l}}_{s},\boldsymbol{p}_{q_{-}}-\boldsymbol{p}_{q_{+}}\right\rangle \right\} \right)\\&\underset{\left(a\right)}{\leq}2\sum_{t=1}^{T}\sum_{s\in\mathcal{S}_{t}}\eta_{s}\left(\sum_{q\in\mathcal{S}_{t},q<s}\eta_{q}+\sum_{r=s}^{t-1}\sum_{q\in\mathcal{S}_{r}}\eta_{q}\right)\underset{\left(b\right)}{\leq}4\sum_{t\notin\mathcal{M^{*}}}\eta_{t}^{2}d_{t}\numberthis\label{eq:103}
\end{align*}
where (a) follows from \eqref{eq:103a} and (b) follows from Lemma
\ref{Counting Delays Lemma}. 

\subsubsection{Concluding the Proof}
We conclude that for $i^{*}\triangleq{\arg\underset{i}\min}\sum_{t=1}^{T}\eta_{t}l_{t}^{\left(i\right)}$:
\begin{multline}
\mathbb{E}\left\{ \sum_{t=1}^{T}\eta_{t}l_{t}^{\left(a_{t}\right)}-\sum_{t=1}^{T}\eta_{t}l_{t}^{\left(i^{*}\right)}\right\} \underset{\left(a\right)}{\leq}\\
\mathbb{E}\left\{ \sum_{t=1}^{T}\sum_{s\in\mathcal{S}_{t}}\eta_{s}\left(l_{s}^{\left(a_{s}\right)}-l_{s}^{\left(i^{*}\right)}\right)\right\} +\sum_{t\in\mathcal{\mathcal{M^{*}}}}\eta_{t}\underset{\left(b\right)}{\leq}\log K+\frac{e^{2}}{2}K\sum_{t=1}^{T}\eta_{t}^{2}+4\sum_{t\notin\mathcal{M^{*}}}\eta_{t}^{2}d_{t}+\sum_{t\in\mathcal{\mathcal{M^{*}}}}\eta_{t}\label{eq:104}
\end{multline}
where (a) uses that $0\leq l_{t}^{\left(i\right)}\leq1$ for every
$i$ and $t$, and (b) uses \eqref{eq:92}, \eqref{eq:101} and \eqref{eq:103}.

\subsection{EXP3 Auxiliary Lemmas}

The following lemma generalizes Lemma 2 from \citet{cesa2019delay}
to a sequence of delays $\left\{ d_{t}\right\}$ and a sequence of
step-sizes $\left\{ \eta_{t}\right\}$.
\begin{lem}
\label{Multiplicative Lemma}Let $\left\{ \eta_{t}\right\} $ be a
positive non-increasing step-size sequence such that $\eta_{t}\leq\frac{1}{2}e^{-2}$
for all $t$. Let $\mathcal{D}=\left\{ t\,|\,d_{t}\geq\frac{1}{e^{2}\eta_{t}}-1\right\} $.
Then for every $s,t$ such that $s\in\mathcal{S}_{t}$ (so $s\notin\mathcal{D}$)
Algorithm \ref{alg:EXP3} maintains for all $i=1,...,K$ both $\frac{p_{s_{+}}^{\left(i\right)}}{p_{s_{-}}^{\left(i\right)}}\leq\frac{1}{1-e^{2}\eta_{s}}$
and $\frac{p_{s_{-}}^{\left(i\right)}}{p_{s}^{\left(i\right)}}\leq e^{2}$.
\end{lem}
\begin{proof}
The proof follows by induction on the feedback arrival index. Let
$s$ be the first feedback to arrive. Before that, at $s_{-}$, we have $p_{s_{-}}^{\left(i\right)}=\frac{1}{K}$ and $\frac{p_{s_{-}}^{\left(i\right)}}{p_{s}^{\left(i\right)}}=1$ for all $i$. Then, 
the first update satisfies
\begin{equation}
\frac{p_{s_{-}}^{\left(i\right)}}{p_{s_{+}}^{\left(i\right)}}=\frac{\frac{1}{K}}{\frac{\frac{1}{K}e^{-\eta_{s}\tilde{l}_{s}^{\left(i\right)}}}{\sum_{j=1}^{K}\frac{1}{K}e^{-\eta_{s}\tilde{l}_{s}^{\left(j\right)}}}}\geq1-\frac{1}{K}+\frac{1}{K}e^{-\eta_{s}\frac{l_{s}^{\left(a_{s}\right)}}{\frac{1}{K}+\gamma_{s}}}\geq1+\frac{1}{K}\left(e^{-\eta_{s}Kl_{s}^{\left(a_{s}\right)}}-1\right)\geq1-\eta_{s}l_{s}^{\left(a_{s}\right)}\geq1-e^{2}\eta_{s}.\label{eq:114}
\end{equation}
Now let $s$ be any arbitrary round for which the feedback arrives
at time $t$. According to the inductive hypothesis, we have $\frac{p_{q_{+}}^{\left(i\right)}}{p_{q_{-}}^{\left(i\right)}}\leq\frac{1}{1-e^{2}\eta_{q}}$
for all $q\in\left\{ r\in\mathcal{S}_{t},r<s\right\} \cup\left\{ \bigcup_{r=s}^{t-1}\mathcal{S}_{r}\right\} $.
Define $s_{0}$ as the minimal $q<s$ such that $s\leq q+d_{q}\leq t$
and $q\notin\mathcal{D}$ (if it exists). Then for all $i=1,...,K$
\begin{align*}
\frac{p_{s_{-}}^{\left(i\right)}}{p_{s}^{\left(i\right)}}=\prod_{r=s}^{t-1}\prod_{q\in\mathcal{S}_{r}}\prod_{q\in\mathcal{S}_{t},q<s}\frac{p_{q_{+}}^{\left(i\right)}}{p_{q_{-}}^{\left(i\right)}}& \underset{\left(a\right)}{\leq}\prod_{r=s}^{t-1}\prod_{q\in\mathcal{S}_{r}}\prod_{q\in\mathcal{S}_{t},q<s}\left(1+\frac{e^{2}\eta_{q}}{1-e^{2}\eta_{q}}\right)\\
& \underset{\left(b\right)}{\leq}\left(1+\frac{1}{e^{-2}\eta_{s_{0}}^{-1}-1}\right)^{d_{s_{0}}}\left(1+\frac{1}{e^{-2}\eta_{s}^{-1}-1}\right)^{d_{s}}\underset{\left(c\right)}{\leq}e^{2}\numberthis \label{eq:115}
\end{align*}
where (a) uses the inductive hypothesis and (c) uses that by definition
$d_{s_{0}}\leq e^{-2}\eta_{s_{0}}^{-1}-1$ and $d_{s}\leq e^{-2}\eta_{s}^{-1}-1$.
If $s_{0}$ does not exist then the first factor is one (i.e., $d_{s_{0}}=0$). Inequality (b) uses that the product runs over all rounds $q\notin\mathcal{D}$ for
which the feedback is received between $s$ and $s_{-}$. Feedback
from $q\in\mathcal{D}$ is discarded and has no effect on $\frac{p_{s_{-}}^{\left(i\right)}}{p_{s}^{\left(i\right)}}$.
The received feedback includes no more than $d_{s_{0}}$ samples of
rounds before $s$. This follows since there are at most $d_{s_{0}}$
rounds between $s_{0}$ and $s$ (since $s\leq s_{0}+d_{s_{0}}$ by
definition), and each of them contributes at most one feedback that
is received between $s$ and $s_{-}$. We have $\eta_{q}\leq\eta_{s_{0}}$
for each such round $q$, since $\eta_{t}$ is non-increasing. It
also includes no more than $d_{s}$ feedback samples of rounds after
$s$, since all these feedback samples are received before $s_{-}$,
which occurs during round $t=s+d_{s}$. We have $\eta_{r}\leq\eta_{s}$
for each such round $r$. We conclude that the update at $s_{-}$,
occurring at time $t$ using the feedback for $a_{s}$, satisfies:
\begin{align*}
\frac{p_{s_{-}}^{\left(i\right)}}{p_{s_{+}}^{\left(i\right)}}=\frac{\frac{e^{-\tilde{L}_{s_{-}}^{\left(i\right)}}}{\sum_{j=1}^{K}e^{-\tilde{L}_{s_{-}}^{\left(j\right)}}}}{\frac{e^{-\tilde{L}_{s_{+}}^{\left(i\right)}}}{\sum_{j=1}^{K}e^{-\tilde{L}_{s_{+}}^{\left(j\right)}}}}&=\frac{\frac{e^{-\tilde{L}_{s_{-}}^{\left(i\right)}}}{\sum_{j=1}^{K}e^{-\tilde{L}_{s_{-}}^{\left(j\right)}}}}{\frac{e^{-\tilde{L}_{s_{-}}^{\left(i\right)}}e^{-\eta_{s}\tilde{l}_{s}^{\left(i\right)}}}{\sum_{j=1}^{K}e^{-\tilde{L}_{s_{-}}^{\left(j\right)}}e^{-\eta_{s}\tilde{l}_{s}^{\left(j\right)}}}}\geq\frac{\sum_{j=1}^{K}e^{-\tilde{L}_{s_{-}}^{\left(j\right)}}e^{-\eta_{s}\tilde{l}_{s}^{\left(j\right)}}}{\sum_{j=1}^{K}e^{-\tilde{L}_{s_{-}}^{\left(j\right)}}}\geq\frac{\sum_{j=1}^{K}e^{-\tilde{L}_{s_{-}}^{\left(j\right)}}\left(1-\eta_{s}\tilde{l}_{s}^{\left(j\right)}\right)}{\sum_{j=1}^{K}e^{-\tilde{L}_{s_{-}}^{\left(j\right)}}}\\
&=1-\eta_{s}\sum_{j=1}^{K}p_{s_{-}}^{\left(j\right)}\tilde{l}_{s}^{\left(j\right)}=1-\eta_{s}p_{s_{-}}^{\left(a_{s}\right)}\frac{l_{s}^{\left(a_{s}\right)}}{p_{s}^{\left(a_{s}\right)}+\gamma_{s}}\geq1-\eta_{s}\frac{p_{s_{-}}^{\left(a_{s}\right)}}{p_{s}^{\left(a_{s}\right)}}\underset{\left(a\right)}{\geq}1-e^{2}\eta_{s}\numberthis\label{eq:116}
\end{align*}
where (a) follows from \eqref{eq:115}. Hence $\frac{p_{s_{+}}^{\left(i\right)}}{p_{s_{-}}^{\left(i\right)}}\leq\frac{1}{1-e^{2}\eta_{s}}$
and the proof is complete.
\end{proof}
The next lemma shows standard smoothness properties of the softmax
function, and we provide it here for completeness. 
\begin{lem}
\label{Lipchitz Lemma}Let $h_{i}\left(\boldsymbol{x}\right)=\frac{e^{-x_{i}}}{\sum_{j=1}^{K}e^{-x_{j}}}$
and $h\left(\boldsymbol{x}\right)=\left(h_{1}\left(\boldsymbol{x}\right),...,h_{K}\left(\boldsymbol{x}\right)\right)$.
Then $\forall\boldsymbol{x}\in\mathbb{R}^{K}$ and $\forall\Delta\in\mathbb{R}_{+}^{K}$
\begin{equation}
\left\Vert h\left(\boldsymbol{x}\right)-h\left(\boldsymbol{x}+\Delta\right)\right\Vert _{1}\leq2\left\langle h\left(\boldsymbol{x}\right),\Delta\right\rangle .\label{eq:62}
\end{equation}
\end{lem}
\begin{proof}
For all $\boldsymbol{x}\in\mathbb{R}^{K}$ and $\Delta\in\mathbb{R}_{+}^{K}$
\begin{equation}
h_{i}\left(\boldsymbol{x}+\Delta\right)-h_{i}\left(\boldsymbol{x}\right)=\frac{e^{-x_{i}-\Delta_{i}}}{\sum_{j=1}^{K}e^{-x_{j}-\Delta_{j}}}-\frac{e^{-x_{i}}}{\sum_{j=1}^{K}e^{-x_{j}}}\underset{\left(a\right)}{\geq}\left(e^{-\Delta_{i}}-1\right)h_{i}\left(\boldsymbol{x}\right)\underset{\left(b\right)}{\geq}-\Delta_{i}h_{i}\left(\boldsymbol{x}\right)\label{eq:63}
\end{equation}
where (a) follows since $\sum_{j=1}^{K}e^{-x_{j}-\Delta_{j}}\leq\sum_{j=1}^{K}e^{-x_{j}}$
and (b) since $1-x\leq e^{-x}$ for all $x\geq0$. We also have for
all $\boldsymbol{x}\in\mathbb{R}^{K}$ and $\Delta\in\mathbb{R}_{+}^{K}$
that
\begin{multline}
h_{i}\left(\boldsymbol{x}+\Delta\right)-h_{i}\left(\boldsymbol{x}\right)=\frac{e^{-x_{i}-\Delta_{i}}}{\sum_{j=1}^{K}e^{-x_{j}-\Delta_{j}}}-\frac{e^{-x_{i}}}{\sum_{j=1}^{K}e^{-x_{j}}}\underset{\left(a\right)}{\leq}\frac{e^{-x_{i}-\Delta_{i}}}{\sum_{j=1}^{K}e^{-x_{j}-\Delta_{j}}}-\frac{e^{-x_{i}-\Delta_{i}}}{\sum_{j=1}^{K}e^{-x_{j}}}=\\
h_{i}\left(\boldsymbol{x}+\Delta\right)\left(1-\frac{\sum_{j=1}^{K}e^{-x_{j}-\Delta_{j}}}{\sum_{l=1}^{K}e^{-x_{l}}}\right)=h_{i}\left(\boldsymbol{x}+\Delta\right)\frac{\sum_{j=1}^{K}e^{-x_{j}}\left(1-e^{-\Delta_{j}}\right)}{\sum_{l=1}^{K}e^{-x_{l}}}\underset{\left(b\right)}{\leq}h_{i}\left(\boldsymbol{x}+\Delta\right)\frac{\sum_{j=1}^{K}\Delta_{j}e^{-x_{j}}}{\sum_{l=1}^{K}e^{-x_{l}}}\label{eq:64}
\end{multline}
where (a) uses $e^{-x_{j}}\geq e^{-x_{j}-\Delta_{j}}$ and (b) uses
$1-x\leq e^{-x}$ for all $x\geq0$. Combining \eqref{eq:63} and
\eqref{eq:64}:
\begin{multline}
\left\Vert h\left(\boldsymbol{x}\right)-h\left(\boldsymbol{x}+\Delta\right)\right\Vert _{1}=\sum_{i=1}^{K}\left|h_{i}\left(\boldsymbol{x}\right)-h_{i}\left(\boldsymbol{x}+\Delta\right)\right|\underset{\left(a\right)}{\leq}\sum_{i=1}^{K}h_{i}\left(\boldsymbol{x}+\Delta\right)\left(\sum_{j=1}^{K}\frac{\Delta_{j}e^{-x_{j}}}{\sum_{l=1}^{K}e^{-x_{l}}}\right)\\
+\sum_{i=1}^{K}\Delta_{i}h_{i}\left(\boldsymbol{x}\right)=\left(\sum_{j=1}^{K}\left(\Delta_{j}\frac{e^{-x_{j}}}{\sum_{l=1}^{K}e^{-x_{l}}}\right)\right)\sum_{i=1}^{K}h_{i}\left(\boldsymbol{x}+\Delta\right)+\left\langle h\left(\boldsymbol{x}\right),\Delta\right\rangle \underset{\left(b\right)}{=}2\left\langle h\left(\boldsymbol{x}\right),\Delta\right\rangle \label{eq:65}
\end{multline}
where (a) uses
$\left|h_{i}\left(\boldsymbol{x}+\Delta\right)-h_{i}\left(\boldsymbol{x}\right)\right|\leq\max\left\{ \Delta_{i}h_{i}\left(\boldsymbol{x}\right),h_{i}\left(\boldsymbol{x}+\Delta\right)\frac{\sum_{j=1}^{K}\Delta_{j}e^{-x_{j}}}{\sum_{l=1}^{K}e^{-x_{l}}}\right\}$
for all $i$, due to  \eqref{eq:63} and \eqref{eq:64}. Equality (b) uses that $\sum_{i=1}^{K}h_{i}\left(\boldsymbol{x}+\Delta\right)=1$
by definition. 
\end{proof}
The next Lemma is taken from \citet{neu2015explore}, and we provide
a (very) slightly modified proof to verify that the same result holds even
when the order of arrivals changes as a result of the delayed feedback. 
\begin{lem}
\label{lem:HighProbability}Let $\tilde{l}_{t}^{\left(i\right)}=\frac{l_{t}^{\left(i\right)}1_{\left\{ a_{t}=i\right\} }}{p_{t}^{\left(i\right)}+\gamma_{t}}$.
If $\left\{ \alpha_{t}^{\left(i\right)}\right\} $ is a non-negative
sequence such that $\alpha_{t}^{\left(i\right)}\leq2\gamma_{t}$ for
all $t$ and all $i$, then
\begin{equation}
\mathbb{P}\left(\sum_{t\notin\mathcal{M}^{*}}\sum_{i=1}^{K}\alpha_{t}^{\left(i\right)}\left(\tilde{l}_{t}^{\left(i\right)}-l_{t}^{\left(i\right)}\right)>\log \frac{1}{\delta}\right)\leq\delta.\label{eq:105}
\end{equation}
\end{lem}
\begin{proof}
Define the filtration $\mathcal{G}_{t}=\sigma\left(\left\{ a_{\tau}\,|\,\tau<t\right\} \right)$
and note that $\boldsymbol{l}_{t}$ is $\mathcal{G}_{t}$-measurable. We have
\begin{equation}
\tilde{l}_{t}^{\left(i\right)}=\frac{l_{t}^{\left(i\right)}1_{\left\{ a_{t}=i\right\} }}{p_{t}^{\left(i\right)}+\gamma_{t}}\leq\frac{l_{t}^{\left(i\right)}1_{\left\{ a_{t}=i\right\} }}{p_{t}^{\left(i\right)}+\gamma_{t}l_{t}^{\left(i\right)}1_{\left\{ a_{t}=i\right\} }}=\frac{1}{2\gamma_{t}}\frac{2\gamma_{t}\frac{l_{t}^{\left(i\right)}}{p_{t}^{\left(i\right)}}1_{\left\{ a_{t}=i\right\} }}{1+\gamma_{t}\frac{l_{t}^{\left(i\right)}}{p_{t}^{\left(i\right)}}1_{\left\{ a_{t}=i\right\} }}\underset{\left(a\right)}{\leq}\frac{1}{2\gamma_{t}}\log\left(1+2\gamma_{t}\frac{l_{t}^{\left(i\right)}}{p_{t}^{\left(i\right)}}1_{\left\{ a_{t}=i\right\} }\right)\label{eq:106}
\end{equation}
where (a) uses $\frac{x}{1+\frac{x}{2}}\leq\log\left(1+x\right)$
which holds for $x\geq0$. Then 
\begin{align*}
\mathbb{E}\left\{ e^{\sum_{i=1}^{K}\alpha_{t}^{\left(i\right)}\tilde{l}_{t}^{\left(i\right)}}\,|\,\mathcal{G}_{t}\right\} & \leq\mathbb{E}\left\{ e^{\sum_{i=1}^{K}\frac{\alpha_{t}^{\left(i\right)}}{2\gamma_{t}}\log\left(1+2\gamma_{t}\frac{l_{t}^{\left(i\right)}}{p_{t}^{\left(i\right)}}1_{\left\{ a_{t}=i\right\} }\right)}\,|\,\mathcal{G}_{t}\right\} \\
& \underset{\left(a\right)}{\leq}\mathbb{E}\left\{ e^{\sum_{i=1}^{K}\log\left(1+\alpha_{t}^{\left(i\right)}\frac{l_{t}^{\left(i\right)}}{p_{t}^{\left(i\right)}}1_{\left\{ a_{t}=i\right\} }\right)}\,|\,\mathcal{G}_{t}\right\} =\mathbb{E}\left\{ \prod_{i=1}^{K}\left(1+\alpha_{t}^{\left(i\right)}\frac{l_{t}^{\left(i\right)}}{p_{t}^{\left(i\right)}}1_{\left\{ a_{t}=i\right\} }\right)\,|\,\mathcal{G}_{t}\right\} \\
& \underset{\left(b\right)}{=}\mathbb{E}\left\{ 1+\sum_{i=1}^{K}\alpha_{t}^{\left(i\right)}\frac{l_{t}^{\left(i\right)}}{p_{t}^{\left(i\right)}}1_{\left\{ a_{t}=i\right\} }\,|\,\mathcal{G}_{t}\right\} \underset{\left(c\right)}{=}1+\sum_{i=1}^{K}\alpha_{t}^{\left(i\right)}l_{t}^{\left(i\right)}\leq e^{\sum_{i=1}^{K}\alpha_{t}^{\left(i\right)}l_{t}^{\left(i\right)}}\numberthis\label{eq:107}
\end{align*}
where (a) uses $x\log\left(1+y\right)\leq\log\left(1+xy\right)$
which holds for $y>-1$ and $0\leq x\leq1$, since $\alpha_{t}^{\left(i\right)}\leq2\gamma_{t}$. Inequality (b) uses that
$1_{\left\{ a_{t}=i\right\} }\text{\ensuremath{1_{\left\{  a_{t}=j\right\}  }}}=0$
for all $i\neq j$, and (c) uses that $p_{t}^{\left(i\right)}$ and $l_{t}^{\left(i\right)}$ are $\mathcal{G}_{t}$-measurable and that given $\mathcal{G}_{t}$, $a_{t}=i$ with probability $p_{t}^{\left(i\right)}$.

Since $l_{t}^{\left(i\right)}$ is $\mathcal{G}_{t}$-measurable then \eqref{eq:107} yields 
$\mathbb{E}\left\{ e^{\sum_{i=1}^{K}\alpha_{t}^{\left(i\right)}\left(\tilde{l}_{t}^{\left(i\right)}-l_{t}^{\left(i\right)}\right)}\,|\,\mathcal{G}_{t}\right\} \leq1$
for all $i$.
Let $S=\sum_{t=1}^{T}\left|\mathcal{S}_{t}\right|$. Let $\tau_{l}$ be $l$-th round for which the feedback is not missing or discarded, so $\tau_{1}\leq\tau_{2}\leq\ldots\leq\tau_{S}$. Then
\begin{multline}
\mathbb{E}\left\{ e^{\sum_{l=1}^{S}\sum_{i=1}^{K}\alpha_{\tau_{l}}^{\left(i\right)}\left(\tilde{l}_{\tau_{l}}^{\left(i\right)}-l_{\tau_{l}}^{\left(i\right)}\right)}\right\} =\mathbb{E}\left\{ \mathbb{E}\left\{ e^{\sum_{l=1}^{S}\sum_{i=1}^{K}\alpha_{\tau_{l}}^{\left(i\right)}\left(\tilde{l}_{\tau_{l}}^{\left(i\right)}-l_{\tau_{l}}^{\left(i\right)}\right)}\,|\,\mathcal{G}_{\tau_{S}}\right\} \right\} \underset{\left(a\right)}{=}\\\mathbb{E}\left\{ e^{\sum_{l=1}^{S-1}\sum_{i=1}^{K}\alpha_{\tau_{l}}^{\left(i\right)}\left(\tilde{l}_{\tau_{l}}^{\left(i\right)}-l_{\tau_{l}}^{\left(i\right)}\right)}\mathbb{E}\left\{ e^{\sum_{i=1}^{K}\alpha_{\tau_{S}}^{\left(i\right)}\left(\tilde{l}_{\tau_{S}}^{\left(i\right)}-l_{\tau_{S}}^{\left(i\right)}\right)}\,|\,\mathcal{G}_{\tau_{S}}\right\} \right\} \leq\mathbb{E}\left\{ e^{\sum_{l=1}^{S-1}\sum_{i=1}^{K}\alpha_{\tau_{l}}^{\left(i\right)}\left(\tilde{l}_{\tau_{l}}^{\left(i\right)}-l_{\tau_{l}}^{\left(i\right)}\right)}\right\} \label{eq:109}
\end{multline}
where (a) uses that  $a_{\tau_{1}},\boldsymbol{l}_{\tau_{1}}, \boldsymbol{p}_{\tau_{1}},\ldots,a_{\tau_{S-1}},\boldsymbol{l}_{\tau_{\tau_{S-1}}}, \boldsymbol{p}_{\tau_{S-1}}$ are $\mathcal{G}_{\tau_{S}}$-measurable.

Iterating over \eqref{eq:109} yields $\mathbb{E}\left\{ e^{\sum_{t\notin\mathcal{M}^{*}}\sum_{i=1}^{K}\alpha_{t}^{\left(i\right)}\left(\tilde{l}_{t}^{\left(i\right)}-l_{t}^{\left(i\right)}\right)}\right\} \leq1$, so by Markov's inequality
\begin{align*}
\mathbb{P}\left(\sum_{t\notin\mathcal{M}^{*}}\sum_{i=1}^{K}\alpha_{t}^{\left(i\right)}\left(\tilde{l}_{t}^{\left(i\right)}-l_{t}^{\left(i\right)}\right)>\log \frac{1}{\delta}\right)
& =\mathbb{P}\left(e^{\sum_{t\notin\mathcal{M}^{*}}\sum_{i=1}^{K}\alpha_{t}^{\left(i\right)}\left(\tilde{l}_{t}^{\left(i\right)}-l_{t}^{\left(i\right)}\right)}>\frac{1}{\delta}\right)\\& \leq\delta\mathbb{E}\left\{ e^{\sum_{t\notin\mathcal{M}^{*}}\sum_{i=1}^{K}\alpha_{t}^{\left(i\right)}\left(\tilde{l}_{t}^{\left(i\right)}-l_{t}^{\left(i\right)}\right)}\right\} \leq\delta.\numberthis\label{eq:128}
\end{align*}
\end{proof}

\end{document}